\documentclass{article}
\usepackage[dvipsnames]{xcolor}
\usepackage{times}

\usepackage{iclr2024_conference_hacked}

\usepackage{hyperref}
\usepackage{url}
\usepackage{amsthm, amsmath, amssymb}
\usepackage{geometry}
\usepackage{algorithm}
\usepackage{algorithmic}
\usepackage{cleveref}
\usepackage{nicefrac}

\newcommand{\pr}[1]{{\color{red} {\bf Peter:} #1}}

\newcommand{\ks}[1]{{\color{blue} {\bf Kostya:} #1}}

\newcommand{\LAM}{{\color{blue}L_{\rm AM}}}
\newcommand{\LQM}{{\color{red}L_{\rm QM}}}
\newcommand{\Lvar}{L_{\rm var}}

\newcommand{\LAMsq}{{\color{blue}L_{\rm AM}^2}}
\newcommand{\LQMsq}{{\color{red}L_{\rm QM}^2}}

\usepackage{mathtools}
\usepackage{xcolor}
\usepackage{color}
\usepackage{pifont}

\usepackage{xspace}
\newcommand{\algname}[1]{{\color{ForestGreen}\small\sf#1}\xspace}
\newcommand{\algnamesmall}[1]{{\color{ForestGreen}\scriptsize\sf#1}\xspace}
\newcommand{\algnametiny}[1]{{\color{ForestGreen}\tiny\sf#1}\xspace}

\newcommand{\norm}[1]{\left\| #1 \right\|}

\newcommand{\cD}{\mathcal{D}}

\newcommand{\ExpBr}[1]{\mathbb{E}\left[#1\right]}

\newcommand{\cZ}{\mathcal{Z}}

\newcommand{\cJ}{\mathcal{J}}
\newcommand{\cI}{\mathcal{I}}

\newcommand{\sumin}{\sum_{i=1}^n}

\newcommand{\sumjn}{\sum_{j=1}^n}

\newcommand{\avein}{\frac{1}{n}\sum_{i=1}^n}

\newcommand{\RR}{\mathbb{R}}
\newcommand{\cC}{{\cal C}}
\newcommand{\cO}{{\cal O}}

\newcommand{\bA}{\mathbf{A}}

\newcommand{\RD}{\mathbb{R}^d}

\newcommand{\sqnorm}[1]{\left\| #1 \right\|^{2}}
\newcommand{\rb}[1]{\left(#1\right)}

\newcommand{\del}[1]{}

\newcommand{\eqdef}{\coloneqq} 

\newcommand{\R}{\mathbb{R}}
\newcommand{\Prob}{\mathbf{Prob}}

\newcommand{\squeeze}{\textstyle}

\DeclareMathOperator*{\argmin}{argmin}

\usepackage[colorinlistoftodos,bordercolor=orange,backgroundcolor=orange!20,linecolor=orange,textsize=scriptsize]{todonotes}

\newtheoremstyle{custom}
	{0pt}
	{0pt}
	{\itshape}
	{}
	{\bfseries}
	{.}
	{.5em}
	{}

\theoremstyle{custom}

\newtheorem{assumption}{Assumption}
\newtheorem{lemma}{Lemma}

\newtheorem{theorem}{Theorem}

\newtheorem{example}{Example}

\theoremstyle{plain}

\theoremstyle{definition}

\newtheorem{definition}[theorem]{Definition}

\usepackage{caption}
\usepackage{subcaption}

\title{Error Feedback Reloaded: From Quadratic to Arithmetic Mean of Smoothness Constants}

\author{Peter Richt\'{a}rik \\
AI Initiative \\ KAUST\thanks{King Abdullah University of Science and Technology}, Saudi Arabia
\And
Elnur Gasanov \\
AI Initiative \\ KAUST, Saudi Arabia
\And
Konstantin Burlachenko\\
AI Initiative \\ KAUST, Saudi Arabia
}

\iclrfinalcopy 

\begin{document}

\maketitle

\begin{abstract}
Error Feedback (\algname{EF}) is a highly popular and immensely effective mechanism for fixing convergence issues which arise in distributed training methods (such as distributed \algname{GD} or \algname{SGD}) when these are enhanced with greedy communication compression techniques such as TopK. While \algname{EF} was proposed almost a decade ago~\citep{Seide2014}, and despite concentrated effort by the community to advance the theoretical understanding of this mechanism, there is still a lot to explore. In this work we study a modern form of error feedback called \algname{EF21}~\citep{EF21} which offers the currently best-known theoretical guarantees, under the weakest assumptions, and also works well in practice. In particular, while the theoretical communication complexity of \algname{EF21} depends on the {\em quadratic mean} of certain smoothness parameters, we improve this dependence to their {\em arithmetic mean}, which is always smaller, and can be substantially smaller, especially in heterogeneous data regimes. We take the reader on a journey of our discovery process. Starting with the idea of applying \algname{EF21} to an equivalent reformulation of the underlying problem which (unfortunately) requires (often impractical) machine {\em cloning}, we continue to the discovery of a new {\em weighted} version of \algname{EF21} which can (fortunately) be executed without any cloning, and finally circle back to an improved {\em analysis} of the original \algname{EF21} method. While this development applies to the simplest form of \algname{EF21}, our approach naturally extends to more elaborate variants involving stochastic gradients and partial participation. Further, our technique improves the best-known theory of \algname{EF21} in the {\em rare features} regime~\citep{EF21-RF}. Finally, we validate our theoretical findings with suitable experiments.
\end{abstract}

 \tableofcontents
 \clearpage

\section{Introduction}

Due to their ability to harness the computational capabilities of modern devices and their capacity to extract value from the enormous data generated by organizations, individuals, and various digital devices and sensors, Machine Learning (ML) methods~\citep{bishop2016pattern, shalev2014understanding} have become indispensable in numerous practical applications~\citep{krizhevsky2012, lin-etal-2022-truthfulqa, transformer, onay2018review, cardio_vas_DL, gavriluct2009malware, sun2017deep}. 

The necessity to handle large datasets has driven application entities to store and process their data in powerful computing centers~\citep{YANG2019278, dean2012large, verbraeken2020survey} via {\em distributed} training algorithms. Beside this industry-standard centralized approach, decentralized forms of distributed learning are becoming increasingly popular. For example, Federated Learning (FL) facilitates a collaborative learning process in which various clients, such as hospitals or owners of edge devices, collectively train a model on their devices while retaining their data locally, without uploading it to a centralized location \citep{FEDLEARN, FEDOPT, mcmahan2017communication, FL_overview, kairouz2021advances, FieldGuide2021}. 

Distributed training problems are typically formulated as optimization problems of the form
\begin{equation}\label{eq:main_problem}
\squeeze\min\limits_{x \in \RR^d} \left\{ f(x) \eqdef \frac{1}{n}\sum \limits_{i=1}^n f_i(x) \right\},
\end{equation}
where $n$ is the number of clients/workers/nodes, vector $x\in \RR^d$ represents the $d$ trainable parameters, and  $f_i(x)$ is the loss of the model parameterized by $x$ on the training data stored on client $i \in [n]\eqdef \{1,\dots,n\}$. One of the key issues in distributed training in general, and FL in particular, is the \textit{communication bottleneck}~\citep{FEDLEARN,kairouz2021advances}. The overall efficiency of a distributed algorithm for solving  \eqref{eq:main_problem} can be characterized by multiplying the number of communication rounds needed to find a solution of acceptable accuracy by the cost of each communication round:
\begin{equation}\label{eq:comm_burden_formula}
{\text{communication complexity} = \text{\# communication rounds} \times \text{cost of 1 communication round}}.
\end{equation}
This simple formula clarifies the rationale behind two orthogonal approaches to alleviating the communication bottleneck. i) The first approach aims to minimize the first factor in~\eqref{eq:comm_burden_formula}. This is done by carefully deciding on {\em what work should be done} on the clients in each communication round in order for it  to reduce the total number of communication rounds needed, and includes methods based on local training~\citep{stich2018local,lin2018don, mishchenko2022proxskip, condat2023tamuna, li2019convergence} and  momentum~\citep{nesterov_accelerated,NesterovBook,AccMethodsBook}. Methods in this class  communicate dense $d$-dimensional vectors. ii) The second approach aims to minimize the second factor in~\eqref{eq:comm_burden_formula}. Methods in this category {\em compress the information} (typically $d$-dimensional vectors) transmitted between the clients and the server~\citep{alistarh2017qsgd, DCGD, bernstein2018signsgd, safaryan2021}.

\subsection{Communication compression}
Vector compression can be achieved through the application of a compression operator. Below, we outline two primary classes of these operators: unbiased (with conically bounded  variance) and contractive.

\begin{definition}[Compressors]
A randomized mapping $\cC: \RR^d \to \RR^d$ is called i) an {\em unbiased compressor} if for some $\omega > 0$ it satisfies 
\begin{equation}\label{eq:unbiased}
\squeeze \ExpBr{\cC(x)} = x, \quad \ExpBr{\|\cC(x) - x \|^2} \leq \omega \|x\|^2, \quad 
\forall x\in\RR^d,\end{equation}
and ii) a {\em contractive compressor} if for some $\alpha \in (0, 1]$ it satisfies 
\begin{equation}\label{eq:compressor_contraction}
\squeeze
\ExpBr{\|\cC(x) - x \|^2} \leq (1 - \alpha) \|x\|^2, \quad 
\forall x\in\RR^d.
\end{equation}\vspace{-2em}
\end{definition}
It is well known that whenever a compressor $\cC$ satisfies \eqref{eq:unbiased}, then the scaled compressor $\cC/(\omega+1)$ satisfies \eqref{eq:compressor_contraction} with $\alpha=1-
(\omega+1)^{-1}$. In this sense, the class of contractive compressors includes all unbiased compressors as well. However, it is also strictly larger. For example, the Top$K$ compressor, which retains the $K$ largest elements in absolute value of the vector it is applied to and replaces the rest by zeros, and happens to be very powerful in practice~\citep{Alistarh-EF-NIPS2018}, satisfies  \eqref{eq:compressor_contraction} with $\alpha=
\frac{K}{d}$, but does not satisfy \eqref{eq:unbiased}. From now on, we write $\mathbb{C}(\alpha)$ to denote the class of compressors satisfying \eqref{eq:compressor_contraction}. 

It will be convenient to define the following functions of the contraction parameter 
\begin{equation}\label{eq:xi}
\squeeze \theta = 
\theta(\alpha) \eqdef 1 - \sqrt{1 - \alpha}; \quad \beta = \beta(\alpha) \eqdef \frac{1 - \alpha}{1 - \sqrt{1 - \alpha}}; \quad \xi= \xi(\alpha)\eqdef \sqrt{\frac{\beta(\alpha)}{\theta(\alpha)}}=\frac{1+\sqrt{1 - \alpha}}{\alpha}-1.
\end{equation}
Note that $0\leq \xi(\alpha) <\frac{2}{\alpha}-1$.
The behavior of distributed algorithms  utilizing unbiased compressors for solving \eqref{eq:main_problem}   is relatively well-understood from a theoretical standpoint~\citep{DCGD, DIANA, ADIANA, MARINA, DASHA}. By now, the community possesses a robust theoretical understanding of the advantages such methods can offer and the mechanisms behind their efficacy~\citep{sigma_k,khaled2023,DASHA}.  However, it is well known that the class of contractive compressors includes some practically very powerful operators, such as the greedy sparsifier Top$K$~\citep{Stich-EF-NIPS2018,Alistarh-EF-NIPS2018} and the low-rank approximator Rank$K$ \citep{PowerSGD, FedNL}, which are biased, and hence their behavior is not explainable by the above developments. These compressors have demonstrated surprisingly effective performance in practice~\citep{Seide2014,Alistarh-EF-NIPS2018}, even when compared to the best results we can get with unbiased compressors~\citep{PermK}, and are indispensable on difficult tasks such as the fine-tuning of foundation models in a geographically distributed manner over slow networks~\citep{cocktailsgd}. 

However, our theoretical understanding of algorithms based on contractive compressors in general, and these powerful biased compressors in particular, is very weak. Indeed, while the SOTA theory involving unbiased compressors offers significant and often several-degrees-of-magnitude improvements  over the baseline methods that do not use compression~\citep{DIANA, DIANA2, ADIANA,sigma_k,MARINA, DASHA}, the best theory we currently have for methods that can provably work with contractive compressors, i.e., the theory behind the error feedback method called \algname{EF21}  developed by \citet{EF21} (see Algorithm~\ref{alg:EF21})  and its variants~\citep{EF21BW, EF-BV,EF21M}, merely matches the communication complexity of the underlying methods that do not use any compression~\citep{PermK}. 

To the best of our knowledge, the only exception to this is the very recent work of \citet{EF21-RF} showing that in a {\em rare features} regime, the \algname{EF21} method \citep{EF21} outperforms gradient descent (which is a special case of \algname{EF21} when $\cC_i^t(x)\equiv x$ for all $i\in [n]$ and $t\geq 0$) in theory. However, \citet{EF21-RF} obtain no improvements upon the current best theoretical result for vanilla \algname{EF21}  \citep{EF21} in the general smooth nonconvex regime, outlined in Section~\ref{eq:assumptions}, we investigate in this work.
\begin{algorithm}[t]
	\begin{algorithmic}[1]
		\STATE {\bfseries Input:} initial model $x^0 \in \RR^d$; initial gradient estimates $g_1^0, g_2^0, \dots,g_n^0 \in \R^d$ stored at the server and the clients; stepsize $\gamma>0$; number of iterations $T > 0$
		\STATE {\bfseries Initialize:} $g^0 = \avein g_i^0 $ on the server
		\FOR{$t = 0, 1, 2, \dots, T - 1 $}
		\STATE Server computes $x^{t+1} = x^t - \gamma g^t$ and  broadcasts  $x^{t+1}$ to all $n$ clients
		\FOR{$i = 1, \dots, n$ {\bf on the clients in parallel}} 
		\STATE Compute $u_i^t=\cC_i^t (\nabla f_i(x^{t+1}) - g_i^t)$ and update $g_i^{t+1} = g_i^t +u_i^t$ \label{alg_line:g_update_step}
		\STATE Send the compressed message $u_i^{t}$ to the server
		\ENDFOR 
		\STATE Server updates $g_i^{t+1} = g_i^t +u_i^t$ for all $i\in [n]$, and computes $g^{t+1} = \avein g_i^{t+1}$
		\ENDFOR
		\STATE {\bfseries Output:} Point $\hat{x}^T$ chosen from the set $\{x^0, \dots, x^{T-1}\}$ uniformly at random
	\end{algorithmic}
	\caption{\algname{EF21}: Error Feedback 2021}
	\label{alg:EF21}
\end{algorithm}

\subsection{Assumptions}\label{eq:assumptions}

We adopt the same very weak assumptions as those used by \citet{EF21} in their analysis of \algname{EF21}.

\begin{assumption}\label{as:smooth}
	The function $f$ is $L$-smooth, i.e., there exists $L>0$ such that
	\begin{equation}\label{eq:smoothness_def}
		\squeeze \norm{\nabla f(x) - \nabla f(y)} \leq L \norm{x - y}, \quad \forall x, y \in \R^d.
	\end{equation}
\end{assumption}

\begin{assumption}\label{as:L_i} The functions $f_i$ are $L_i$-smooth,  i.e., for all $i\in [n]$ there exists $L_i>0$ such that
	\begin{equation} \label{eq:L_i}
		\squeeze \norm{\nabla f_i(x) - \nabla f_i(y)} \leq L_i \norm{x - y}, \quad \forall x, y \in \R^d.
	\end{equation}
\end{assumption}

Note  that if \eqref{eq:L_i} holds,  then \eqref{eq:smoothness_def} holds, and  $L\leq \LAM \eqdef {\color{blue}\frac{1}{n}\sum_{i=1}^n L_i}$. So, Assumption~\ref{as:smooth} does {\em not} further limit the class of functions already covered by Assumption~\ref{as:L_i}. Indeed, it merely provides a new parameter $L$ better characterizing the smoothness of $f$ than the estimate $\LAM$ obtainable from Assumption~\ref{as:L_i} could. 

Since our goal   in \eqref{eq:main_problem} is to minimize $f$, the below assumption is necessary for the problem to be meaningful.

\begin{assumption}\label{as:lower_bound}
There exists $f^\ast \in \RR$ such that $\inf f \geq f^\ast$.\end{assumption}

\subsection{Summary of contributions}

In our work we  improve the current SOTA theoretical communication complexity guarantees for distributed algorithms that  work with contractive compressors in general, and  empirically powerful biased compressors such as Top$K$ and Rank$K$ in particular~\citep{EF21,EF21BW}. 


In particular, under Assumptions~\ref{as:smooth}--\ref{as:lower_bound}, the best known guarantees were obtained by \citet{EF21} for the \algname{EF21} method: to find a (random) vector $\hat{x}^T$ satisfying $\ExpBr{\norm{\nabla f(\hat{x}^T)}^2} \leq \varepsilon$,   \Cref{alg:EF21} requires 
$$
\squeeze T=\cO \left({\left(L + \LQM \xi (\alpha)\right)}{\varepsilon}^{-1}\right)
$$ 
iterations, where
$\color{red} \LQM \eqdef \sqrt{\frac{1}{n} \sum_{i=1}^n L_i^2}$ is the {\color{red}Quadratic Mean} of the smoothness constants $L_1,\dots,L_n$. Our main finding is an improvement of this result to \begin{equation}\label{eq:98y98fd}\squeeze T=\cO\left({\left( L + \LAM \xi (\alpha)\right)}{\varepsilon}^{-1} \right),\end{equation}
where $\color{blue} \LAM \eqdef \frac{1}{n} \sum_{i=1}^n L_i$ is the {\color{blue}Arithmetic Mean} of the smoothness constants $L_1,\dots,L_n$. We obtain this improvement in {\em three different ways:} 
\begin{itemize}
\item [i)] by {\em client cloning} (see Sections~\ref{sec:clone1} and \ref{sec:clone2} and Theorem~\ref{thm:clone}), 
\item [ii)] by proposing a new {\em smootness-weighted} variant of \algname{EF21} which we call \algname{EF21-W} (see Section~\ref{sec:weights} and Theorem~\ref{thm:EF21-W}), and 
\item [iii)] by a new {\em smoothness-weighted} analysis of classical \algname{EF21} (see Section~\ref{sec:weighted-analysis} and Theorem~\ref{thm:ef21_new_result}).
\end{itemize}

We obtain refined linear convergence results cases under the Polyak-Łojasiewicz condition. Further, our analysis technique extends to many variants of \algname{EF21}, including \algname{EF21-SGD} which uses stochastic gradients instead of gradients (Section~\ref{sec:EF21-W-SGD}), and \algname{EF21-PP} which enables partial participation of clients (Section~\ref{sec:EF21-W-PP}).
Our analysis also improves upon the results of \citet{EF21-RF} who study \algname{EF21} in the {\em rare features} regime (Section~\ref{sec:RF}). Finally, we validate our theory with suitable computational experiments (Sections~\ref{sec:experiments-main}, \ref{app:exp-additional-details-main-part} and \ref{app:exp-additional-experiments}).

\section{EF21 Reloaded: Our Discovery Story} \label{sec:reloaded}

We now take the reader along on a ride of our discovery process.

\subsection{Step 1: Cloning the client with the worse smoothness constant}\label{sec:clone1}

The starting point of our journey is  a simple observation described in the following example.
\begin{example} Let $n=4$ and $f(x) = \frac14 (f_1(x) + f_2(x) + f_3(x) + f_4(x))$. Assume the smoothness constants $L_1, L_2, L_3$ of $f_1, f_2, f_3$ are equal to 1, and $L_4$ is equal to $100$. In this case,\algname{EF21} needs to run for
$$
\squeeze T_1 \eqdef \cO\left({\left(L + \LQM \xi(\alpha)\right)}{\varepsilon}^{-1}\right) = \cO\left({\left(L +  \sqrt{2501.5}\xi(\alpha)\right)}{\varepsilon}^{-1}\right)
$$ 
iterations. Now, envision the existence of an additional machine capable of downloading the data from the fourth ``problematic'' machine. By rescaling local loss functions, we maintain the overall loss function as: 
$$
\squeeze f(x) = \frac14 (f_1(x) + f_2(x) + f_3(x) + f_4(x)) =  \frac15 \left( \frac54 f_1(x) + \frac54 f_2(x) + \frac54 f_3(x) + \frac58 f_4(x) +\frac58 f_4(x) \right) \eqdef \tilde{f}(x).
$$ 
Rescaling of the functions modifies the smoothness constants to $\hat{L}_i = \frac54 L_i$ for $i = 1, 2, 3$, and $\hat{L}_i = \frac58 L_4$ for $i=4, 5$. \algname{EF21}, launched on this setting of five nodes, requires 
$$
\squeeze T_2 \eqdef \cO\left({\left(L + {\color{red}\tilde{L}_{\rm QM}} \xi(\alpha)\right)}{\varepsilon}^{-1} \right) \approx \cO\left({\left(L + \sqrt{1564} \xi (\alpha)\right)}{\varepsilon}^{-1}\right)
$$
iterations, where ${\color{red}\tilde{L}_{\rm QM}}$ is the quadratic mean of the new smoothness constants $\hat{L}_1,\dots,\hat{L}_5$.
\end{example}

This simple observation highlights that the addition of just one more client significantly enhances the convergence rate. Indeed, \algname{EF21} requires approximately $\frac{\xi(\alpha)}{\varepsilon} (\sqrt{2501.5} - \sqrt{1564}) \approx 10 \frac{\xi(\alpha)}{\varepsilon}$ fewer iterations. We will generalize this client cloning idea  in the next section.

\subsection{Step 2: Generalizing the cloning idea} \label{sec:clone2}
We will now take the above motivating example further, allowing each client $i$ to be cloned arbitrarily many ($N_i$) times. Let us see where this gets us. For each $i\in [n]$, let $N_i$ denote  a positive integer. We define $N\eqdef \sum_{i=1}^n N_i$ (the total number of clients after cloning), and observe that $f$ can be equivalently written as
\begin{equation}\label{eq:cloning}
\squeeze f(x)  \overset{\eqref{eq:main_problem}}{=} \frac{1}{n}\sum \limits_{i=1}^n f_i(x) =  \frac{1}{n}\sum \limits_{i=1}^n  \sum\limits_{j=1}^{N_i} \frac{1}{N_i} f_i(x) = \frac{1}{N}\sum \limits_{i=1}^n  \sum\limits_{j=1}^{N_i}\frac{N}{n N_i} f_i(x) = \frac{1}{N}\sum \limits_{i=1}^n  \sum\limits_{j=1}^{N_i} f_{ij}(x),
\end{equation}
where 
$ f_{ij}(x) \eqdef \frac{N}{n N_i} f_i(x)$ for all $ i\in [n]$ and $ j\in [N_i].$
Notice that we scaled the functions as before, and that $f_{ij}$ is $L_{ij}$-smooth, where $L_{ij} \eqdef \frac{N}{n N_i} L_i$. 

{\bf Analysis of the convergence rate.} The performance of \algname{EF21}, when applied to the problem \eqref{eq:cloning} involving $N$ clients, depends on the quadratic mean of the new smoothness constants:
\begin{equation} \label{eq:M-solve} \squeeze M(N_1,\dots,N_n) \eqdef \sqrt{\frac{1}{N} \sum\limits_{i=1}^n \sum\limits_{j=1}^{N_i} L_{ij}^2} = \sqrt{\sum\limits_{i=1}^n \frac{N}{n^2 N_i} L_i^2} = \frac{1}{n}\sqrt{\sum\limits_{i=1}^n \frac{L_i^2}{N_i/N} }  .\end{equation}
Note that if $N_i=1$ for all $i\in [n]$, then $M(1,\dots,1)=\LQM$.

{\bf Optimal choice of cloning frequencies.} Our goal is to find integer values $N_1\in \mathbb{N},\dots,N_n \in \mathbb{N} $ minimizing the function $M(N_1,\dots,M_n)$ defined in \eqref{eq:M-solve}. While we do not have a closed-form formula for the global minimizer, we are able to explicitly find a  solution that is at most $\sqrt{2}$ times worse than the optimal one in terms of the objective value. In particular, if we let $N^\star_i =  \left \lceil  \nicefrac{L_i}{\LAM} \right \rceil$ for all $i\in [n]$, then 
\[\squeeze \LAM \leq  \min\limits_{N_1 \in \mathbb{N},\dots, N_n\in \mathbb{N}} M(N_1,\dots,N_n) \leq   M(N^\star_1,\dots,N^\star_n) \leq \sqrt{2} \LAM,\]
and moreover, $n \leq N^\star \eqdef \sum_i N^\star_i \leq 2n$. That is, we need at most double the number of clients in our client cloning construction.  See Lemma~\ref{lem:sandwitch} in the Appendix for details.

By directly applying \algname{EF21} theory from \citep{EF21} to problem \eqref{eq:cloning} involving $N^\star$ clients, we obtain the advertised improvement from $\LQM$ to $\LAM$.

\begin{theorem}[\textbf{Convergence of \algname{EF21} applied to problem \eqref{eq:cloning} with $N^\star$ machines}] \label{thm:clone}  Consider Algorithm~\ref{alg:EF21} (\algname{EF21}) applied to the ``cloning reformulation''  \eqref{eq:cloning} of the distributed optimization problem \eqref{eq:main_problem}, where $N^\star_i = \left \lceil \nicefrac{L_i}{\LAM} \right \rceil$ for all $i \in [n]$.  Let Assumptions~\ref{as:smooth}--\ref{as:lower_bound} hold, assume that
$\cC_{ij}^t \in \mathbb{C}(\alpha)$ for all $i\in [n]$, $j\in [N_i]$ and $t\geq 0$,  set
 \begin{equation*}
\squeeze G^t \eqdef \frac{1}{N}\sum \limits_{i=1}^N \sum\limits_{j=1}^{N_i} \norm{g_{ij}^t - \nabla f_{ij}(x^t)}^2,
 \end{equation*}
 and let the  stepsize satisfy
$
0 < \gamma \leq \frac{1}{L + \sqrt{2}\LAM \xi (\alpha)}.
$
If for $T \geq 1$ we define $\hat{x}^T$ as an element of the set $\{x^0, x^1, \dots, x^{T-1}\}$ chosen uniformly at random, then
\begin{equation*}
\squeeze
 \ExpBr{\norm{\nabla f(\hat{x}^T)}^2} \leq \frac{2 (f(x^0) - f^\ast) }{\gamma T} + \frac{G^0}{\theta(\alpha) T}.
\end{equation*}
\end{theorem}

When we choose the largest allowed stepsize and $g_{ij}^0 = \nabla f_{ij}(x^0)$ for all $i,j$, this leads to the complexity \eqref{eq:98y98fd}; that is, by cloning client machines, we can replace $\color{red}L_{\rm QM}$ in the standard rate with $\sqrt{2} \color{blue} L_{\rm AM}$. A similar result can be obtained even if we do not ignore the integrality constraint, but we do not include it for brevity reasons. However, it is important to note that the cloning approach has several straightforward shortcomings, which we will address in the next section.\footnote{	In our work, we address an optimization problem of the form 
	$
	\min_{\substack{w_j \geq 0 \ \forall j\in[n]; \sum_{i=1}^{n} w_i = 1}} \sum_{i=1}^{n} \frac{a_i^2}{w_i},
	$
	where \(a_i\) represent certain constants. This formulation bears a resemblance to the meta problem in the importance sampling strategy discussed in \citep{so_importance_sampling}. Despite the apparent similarities in the abstract formulation, our approach and the one in the referenced work diverge significantly in both motivation and implementation. While \cite{so_importance_sampling} applies importance sampling to reduce the variance of a stochastic gradient estimator by adjusting sampling probabilities, our method involves adjusting client cloning weights without sampling. Furthermore, our gradient estimator is biased, unlike the unbiased estimator in the referenced paper, and we aim to minimize the quadratic mean of the smoothness constants, which is inherently different from the objectives in \cite{so_importance_sampling}. Although both approaches can be expressed through a similar mathematical framework, they are employed in vastly different contexts, and any parallelism may be coincidental rather than indicative of a direct connection.}

\subsection{Step 3: From client cloning to update weighting} \label{sec:weights}

It is evident that employing client cloning improves the convergence rate. Nevertheless, there are obvious drawbacks associated with this approach. Firstly, it necessitates a larger number of computational devices, rendering its implementation less appealing from a resource allocation perspective. Secondly, the utilization of \algname{EF21} with cloned machines results in a departure from the principles of Federated Learning, as it inherently compromises user privacy -- transferring data from one device to another is prohibited in FL.

However, a  simpler approach to implementing the cloning idea emerges when we assume the compressors used to be {\em deterministic}. To illustrate this, let us initially examine how we would typically implement \algname{EF21} with cloned machines:
\begin{align} 
&\squeeze x^{t+1} = x^t - \gamma \frac{1}{N}\sum\limits_{i=1}^n \sum\limits_{j=1}^{N_i} g_{ij}^t, \label{eq:AlgStep1}  \\
&\squeeze g_{ij}^{t+1} = g_{ij}^t + \cC_{ij}^t ( \nabla f_{ij}(x^{t+1}) - g_{ij}^t), \quad i \in [n],  \quad j \in [N_i]. \label{eq:AlgStep2}
\end{align}

We will now rewrite the same method in a different way. Assume we choose $g_{ij}^0 = g_i^0$ for all $j\in[N_i]$. We show by induction that $g_{ij}^t$ is the same for all $j \in [N_i]$. We have just seen that this holds for $t=0$. Assume this holds for some $t$. Then since $ \nabla f_{ij}(x^{t+1}) = \frac{N}{n N_i} \nabla f_i(x^{t+1})$ for all $j\in[N_i]$ combined with the induction hypothesis, \eqref{eq:AlgStep2} and the determinism of $\cC_{ij}^t$, we see that $g_{ij}^{t+1}$ is the same for all $j \in [N_i]$. Let us define $g_{i}^t \equiv g_{ij}^t$ for all $t$. This is a valid definition since we have shown that $g_{ij}^t$ does not depend on $j$. Because of all of the above, iterations \eqref{eq:AlgStep1}--\eqref{eq:AlgStep2} can be equivalently written in the form
\begin{align}
&\squeeze x^{t+1} = x^t - \gamma \sum\limits_{i=1}^n {\color{ForestGreen}w_i}  g_{i}^t, \label{eq:AlgStep1X} \\
&\squeeze g_{i}^{t+1} = g_{i}^t + \cC_i^t\left( \frac{1}{n {\color{ForestGreen}w_i}} \nabla f_i(x^{t+1}) - g_{i}^t \right), \quad i\in [n], \label{eq:AlgStep2X} 
\end{align}
where ${\color{ForestGreen}w_i} = \frac{L_i}{\sum_j L_j}$. This transformation effectively enables us to operate the method on the original $n$ clients, eliminating the need for $N$ clients! This refinement has led to the creation of a new algorithm that outperforms \algname{EF21} in terms of convergence rate, which we call \algname{EF21-W} (\Cref{alg:EF21-W}).
\begin{algorithm}[!t]
	\begin{algorithmic}[1]
		\STATE {\bfseries Input:} initial model parameters $x^0 \in \RR^d$; initial gradient estimates $g_1^0, g_2^0, \dots,g_n^0 \in \R^d$ stored at the server and the clients; weights ${\color{ForestGreen}w_i} = \nicefrac{L_i}{\sum_j L_j}$; stepsize $\gamma>0$; number of iterations $T > 0$
		\STATE {\bfseries Initialize:} $g^0 = \sum_{i=1}^n {\color{ForestGreen}w_i} g_i^0 $ on the server
		\FOR{$t = 0, 1, 2, \dots, T - 1 $}
		\STATE Server computes $x^{t+1} = x^t - \gamma g^t$		 and broadcasts  $x^{t+1}$ to all $n$ clients
		\FOR{$i = 1, \dots, n$ {\bf on the clients in parallel}} 
		\STATE Compute $u_i^t=\cC_i^t (\frac{1}{n {\color{ForestGreen}w_i}}\nabla f_i(x^{t+1}) - g_i^t)$ and update $g_i^{t+1} = g_i^t +u_i^t$ \label{alg_line:g_update_step}
		\STATE Send the compressed message $u_i^{t}$ to the server
		\ENDFOR 
		\STATE Server updates $g_i^{t+1} = g_i^t +u_i^t$ for all $i\in [n]$, and computes $g^{t+1} = \sum_{i=1}^n {\color{ForestGreen}w_i} g_i^{t+1}$
		\ENDFOR
		\STATE {\bfseries Output:} Point $\hat{x}^T$ chosen from the set $\{x^0, \dots, x^{T-1}\}$ uniformly at random
	\end{algorithmic}
	\caption{\algname{EF21-W}: Weighted Error Feedback 2021}
	\label{alg:EF21-W}
\end{algorithm}
While we relied on assuming that the compressors are \textit{deterministic} in order to motivate the transition from $N$ to $n$ clients, it turns out that \algname{EF21-W} converges without the need to invoke this assumption. \begin{theorem}[\textbf{Theory for \algname{EF21-W}}]\label{thm:EF21-W}
	Consider  \Cref{alg:EF21-W} (\algname{EF21-W}) applied to the distributed optimization problem \eqref{eq:main_problem}.
Let Assumptions~\ref{as:smooth}--\ref{as:lower_bound} hold, assume that $\cC_i^t \in \mathbb{C}(\alpha)$ for all $i\in [n]$ and $t\geq 0$, set $$\squeeze G^t \eqdef \sum\limits_{i=1}^n {\color{ForestGreen}w_i}  \norm{g_i^t - \frac{1}{n\color{ForestGreen}w_i} \nabla f_i(x^t)}^2,$$ where ${\color{ForestGreen}w_i} = \frac{L_i}{\sum_j L_j}$ for all $i\in [n]$, and let the stepsize satisfy
$
0< \gamma \leq \frac{1}{L +  \LAM  \xi(\alpha)}.
$
If for $T>1$ we define $\hat{x}^T$ as an element of the set  $\{x^0, x^1, \dots, x^{T-1}\}$ chosen uniformly at random, then
	\begin{equation}
\squeeze	
		\ExpBr{\norm{ \nabla f(\hat{x}^T)  }^2} \leq \frac{2(f(x^0) - f^\ast)}{\gamma T} + \frac{G^0}{\theta(\alpha) T}.
	\end{equation}
\end{theorem}
\subsection{Step 4: From weights in the algorithm to weights in the analysis}\label{sec:weighted-analysis}

In the preceding section, we introduced a novel algorithm: \algname{EF21-W}. While it bears some resemblance to the vanilla \algname{EF21} algorithm~\citep{EF21} (we recover it for uniform weights), the reliance on  particular non-uniform weights  enables it to achieve a faster convergence rate. However, this is not the end of the story as another insight reveals yet another surprise.

Let us consider the scenario when the compressors in \Cref{alg:EF21-W} are  {\em positively homogeneous}\footnote{A compressor $\cC:\R^d\to \R^d$ is positively homogeneous if $\cC(tx) = t\cC(x)$ for all $t>0$ and $x\in\RR^d$.}. Introducing the new variable $h_i^t = n {\color{ForestGreen}w_i} g_i^t$, we can reformulate the gradient update in  \Cref{alg:EF21-W} to \begin{align*}
\squeeze h_i^{t+1} & \squeeze = {n{\color{ForestGreen}w_i}} g_i^{t+1} \overset{\eqref{eq:AlgStep2X}}{=} {n{\color{ForestGreen}w_i}} \left[g_{i}^t + \cC_i^t\left( \frac{\nabla f_i(x^{t+1})}{n {\color{ForestGreen}w_i}}  - g_{i}^t \right) \right]  =  h_i^t + \cC_i^t (\nabla f_i(x^t) - h_i^t),
\end{align*}
indicating that $h_i^t$ adheres to the update rule of the vanilla \algname{EF21} method! Furthermore, the iterates $x^t$  also follow the same rule as \algname{EF21}:
\[
 \squeeze x^{t+1} \overset{\eqref{eq:AlgStep1X}}{=} x^t - \gamma \sum \limits_{i=1}^n{\color{ForestGreen}w_i} g^t =  x^t - \gamma \sum \limits_{i=1}^n {\color{ForestGreen}w_i} \frac{1}{n{\color{ForestGreen}w_i}} h_i^t = x^t - \gamma \frac{1}{n}\sum \limits_{i=1}^n h_i^t. 
\]
So, what does this mean? One interpretation suggests that for positively homogeneous contractive compressors, the vanilla \algname{EF21} algorithm is equivalent to \algname{EF21-W}, and hence inherits its faster convergence rate that depends on $\LAM$ rather than on $\LQM$. However, it turns out that we can establish the same result without having to resort to positive homogeneity altogether. For example, the "natural compression" quantizer, which rounds to one of the two nearest powers of two, is not positively homogeneous \citep{Cnat}.

\begin{theorem}[{\bf New theory for \algname{EF21}}]\label{thm:ef21_new_result}
	Consider  \Cref{alg:EF21} (\algname{EF21}) applied to the distributed optimization problem \eqref{eq:main_problem}. Let Assumptions~\ref{as:smooth}--\ref{as:lower_bound} hold, assume that $\cC_i^t \in \mathbb{C}(\alpha)$ for all $i\in [n]$ and $t\geq 0$, set $$\squeeze G^t \eqdef \frac{1}{n}\sum \limits_{i=1}^n \frac{1}{n {\color{ForestGreen}w_i}} \norm{g_i^t - \nabla f_i(x^t)}^2,$$ where ${\color{ForestGreen}w_i} = \frac{L_i}{\sum_j L_j}$ for all $i\in [n]$, and let the stepsize satisfy
$
0< \gamma \leq \frac{1}{L + \LAM \xi(\alpha) }.
$
If for $T>1$ we define $\hat{x}^T$ as an element of the set  $\{x^0, x^1, \dots, x^{T-1}\}$ chosen uniformly at random, then
	\begin{equation}
\squeeze	
		\ExpBr{\| \nabla f(\hat{x}^T)  \|^2} \leq \frac{2(f(x^0) - f^\ast)}{\gamma T} + \frac{G^0}{\theta(\alpha) T}.
	\end{equation}
\end{theorem}
This last result effectively pushes the weights from the algorithm in \algname{EF21-W} to the proof, which enabled us to show that the original \algname{EF21} method also enjoys the same improvement: from $\LQM$ to $\LAM$.

\section{Experiments}
\label{sec:experiments-main}

\subsection{Non-convex logistic regression on benchmark datasets}
\label{sec:experiments-main-real}

In our first experiment, we employed a logistic regression model with a non-convex regularizer, i.e., \begin{equation*}
\squeeze f_i(x) \eqdef \frac{1}{n_i} \sum \limits_{j=1}^{n_i} \log \left(1+\exp({-y_{ij} \cdot a_{ij}^{\top} x})\right) + \lambda  \sum\limits_{j=1}^{d} \frac{x_j^2}{x_j^2 + 1},
\end{equation*}
where $(a_{ij},  y_{ij}) \in \mathbb{R}^{d} \times \{-1,1\}$ represents the $j$-th data point out from a set of $n_i$ data points stored at client $i$, and $\lambda>0$ denotes a regularization coefficient. We utilized six datasets from \texttt{LIBSVM} \citep{chang2011libsvm}.  The dataset shuffling strategy, detailed in \Cref{app:dataset-shuffling-for-libsvm}, was employed to emulate heterogeneous data distribution. Each client was assigned the same number of data points. Figure~\ref{fig:real-ef21-vc-ncvx} provides a comparison between \algname{EF21} employing the original stepsize~\citep{EF21} and \algname{EF21-W} with the better stepsize. The initial gradient estimators were chosen as $g_i^0 = \nabla f_i(x^0)$ for all $ i \in [n]$. As evidenced empirically, the \algname{EF21-W} algorithm emerges as a practical choice when utilized in situations characterized by high variance in smoothness constants. As evident from the plots, the algorithm employing the new step size exhibits superior performance compared to its predecessor.

Next, we conducted a comparative analysis of the performance of \algname{EF21-W-PP} and \algname{EF21-W-SGD}, as elucidated in the appendix, compared to  their non-weighted counterparts. In the \algname{EF21-PP}/\algname{EF21-W-PP} algorithms, each client participated independently in each round with probability $p_i=0.5$. Moreover, in the case of \algname{EF21-SGD}/\algname{EF21-W-SGD} algorithms, a single data point was stochastically sampled from a uniform distribution at each client during each iteration of the algorithm. As observed in \Cref{fig:ext-real-ef21-vc-ncvx}, the algorithms employing the new learning rates demonstrate faster convergence. Notably, \Cref{fig:ext-real-ef21-vc-ncvx}~(c) depicts more pronounced oscillations with updated step sizes, as the new analysis permits larger step sizes, which can induce oscillations in stochastic methods.

\subsection{Non-convex linear model on synthetic datasets}
\label{sec:experiments-main-syb-ncvx}
In our second set of experiments, we trained a linear regression model with a non-convex regularizer. The function $f_i$ for the linear regression problem is defined as follows:
$ f_i(x) \eqdef \frac{1}{n_i} \norm{\bA_i x - {b_i}}^2 + \lambda  \sum_{j=1}^{d} \frac{x_j^2}{x_j^2 + 1}.
$
Here, $\bA_i \in \RR^{n_i \times d}$ and $b_i \in \RR^{n_i}$ represent the feature matrix and labels stored on client $i$ encompassing $n_i$ data points.   The data employed in four experiments, as illustrated in \Cref{fig:syn-ef21-vc-noncvx}, was generated in such a manner that the smoothness constant $L$ remained fixed, while $L_i$ varied so that the difference between two crucial to analysis terms $\LQM$ and $\LAM$ changed from a relatively large value to negligible. As evident from  \Cref{fig:syn-ef21-vc-noncvx}, the performance of \algname{EF21-W} consistently matches or surpasses that of the original \algname{EF21}, particularly in scenarios characterized by significant variations in the smoothness constants. For additional details and supplementary experiments, we refer the reader to Sections~\ref{app:exp-additional-details-main-part} and~\ref{app:exp-additional-experiments}.

\begin{center}	
	\begin{figure*}[t!]
		\centering
		\captionsetup[sub]{font=normalsize,labelfont={}}	
		\captionsetup[subfigure]{labelformat=empty}
		\newcommand{\myVar}{0.3\textwidth}
		\newcommand{\myVarN}{0.7\textwidth}
		
		\begin{subfigure}[ht]{\myVar}
			\includegraphics[width=\myVarN]{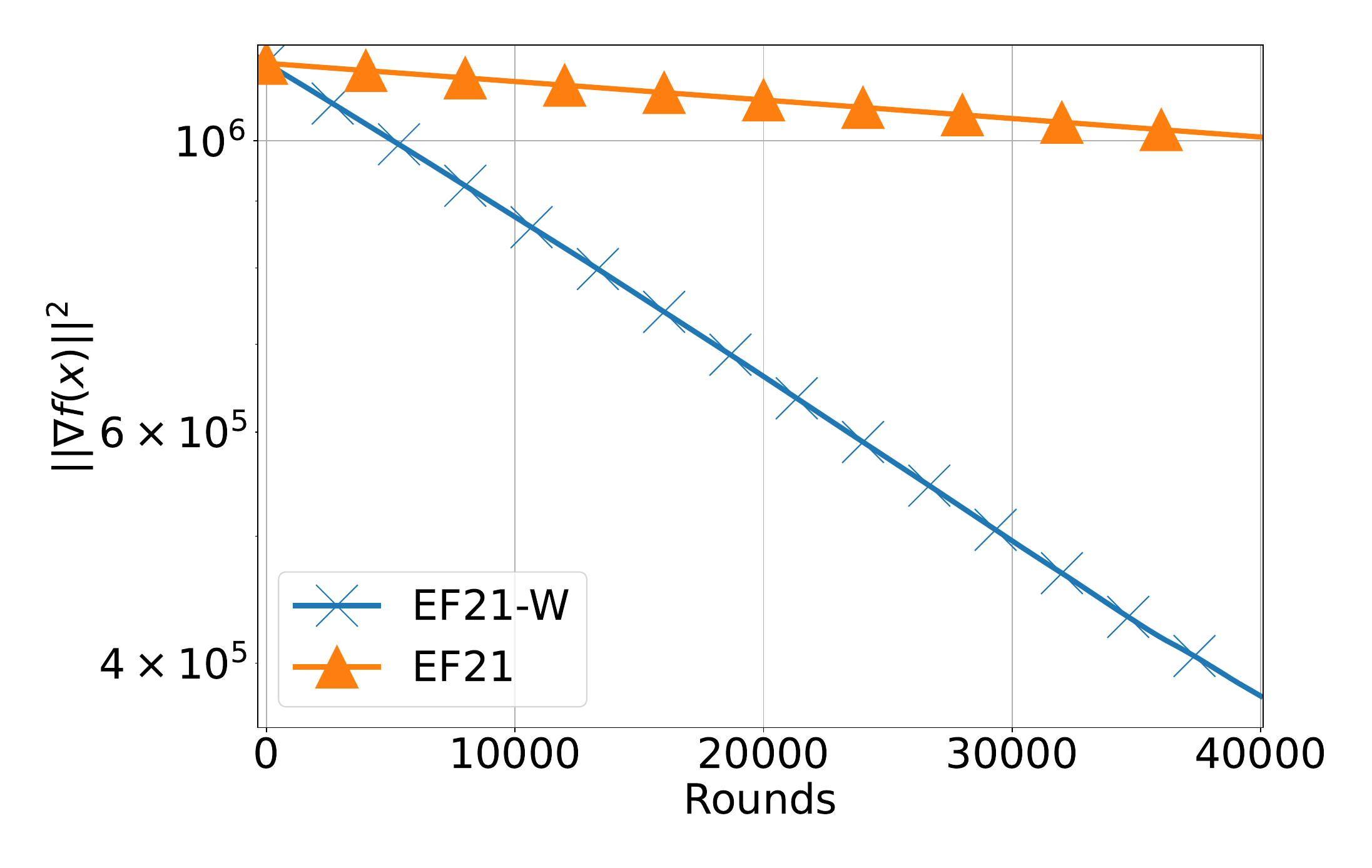} 
			\caption{ \small{ (a) \texttt{AUSTRALIAN}, $\Lvar \approx 10^{16}$ } }
		\end{subfigure}	
		\begin{subfigure}[ht]{\myVar}
			\includegraphics[width=\myVarN]{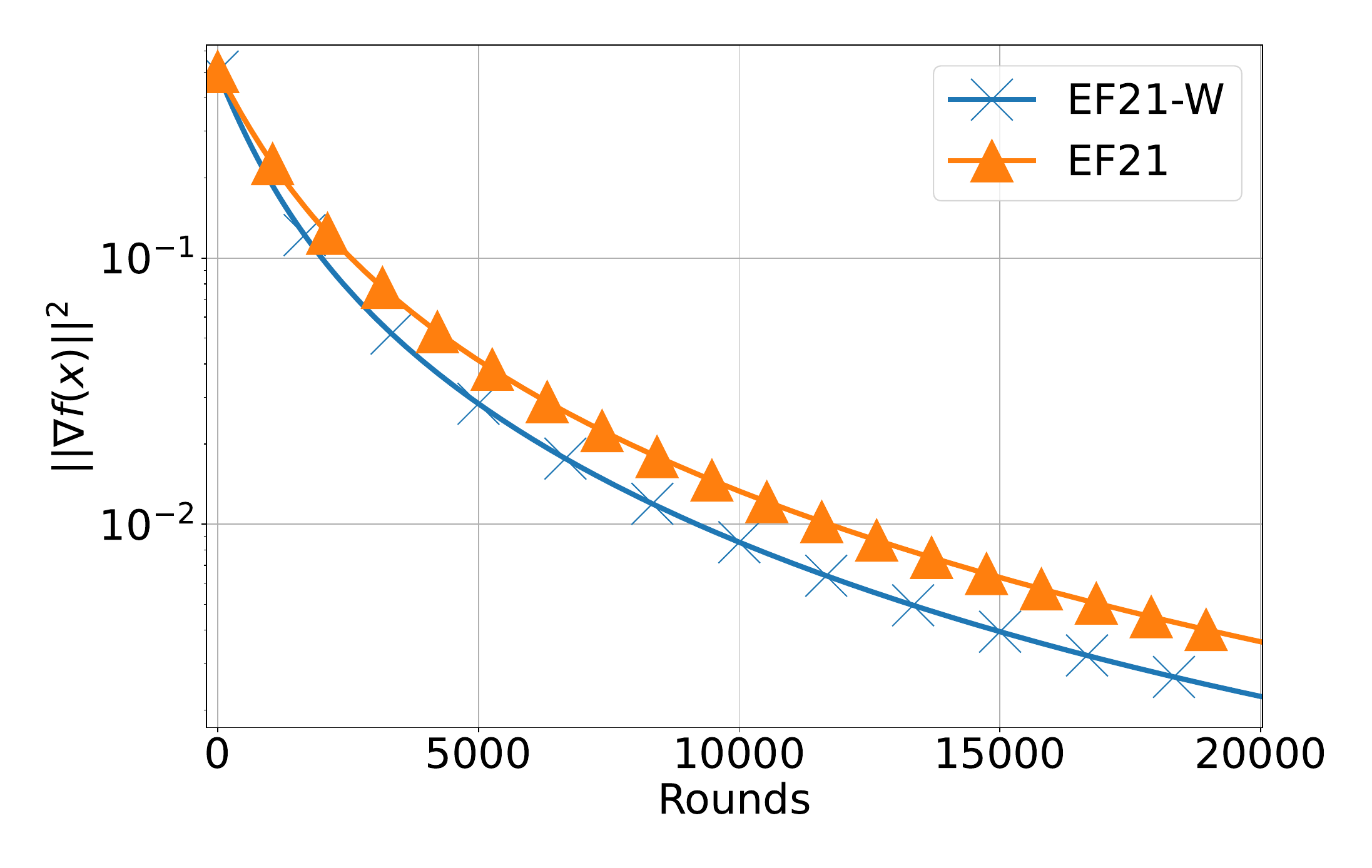} \caption{ \small{(b) \texttt{W1A}, $\Lvar\approx 3.28$ } }
		\end{subfigure}		
		\begin{subfigure}[ht]{\myVar}
			\includegraphics[width=\myVarN]{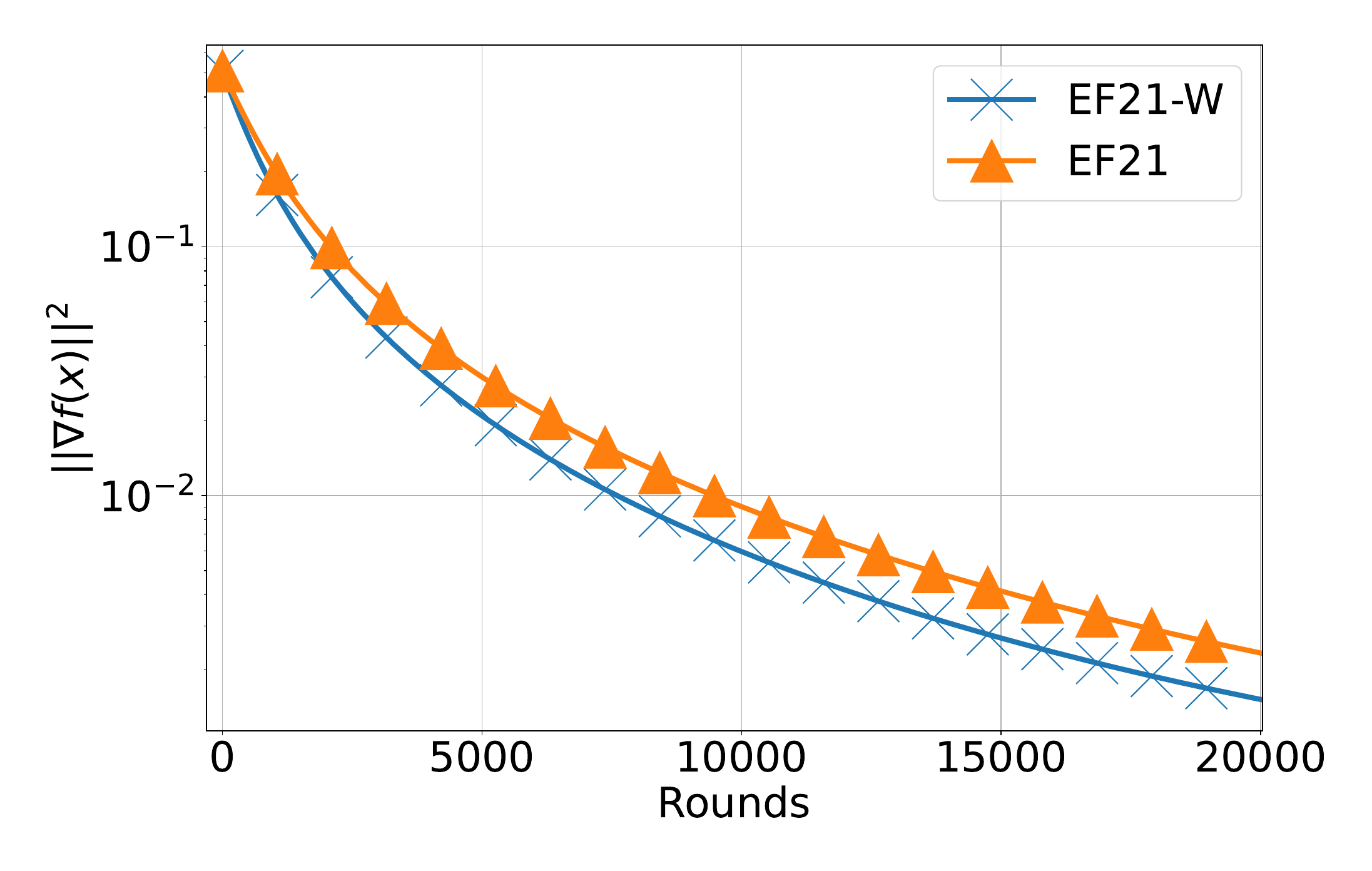} 
			\caption{ \small{(c) \texttt{W2A}, $\Lvar\approx2.04$ } }
		\end{subfigure}
		
		\begin{subfigure}[ht]{\myVar}
			\includegraphics[width=\myVarN]{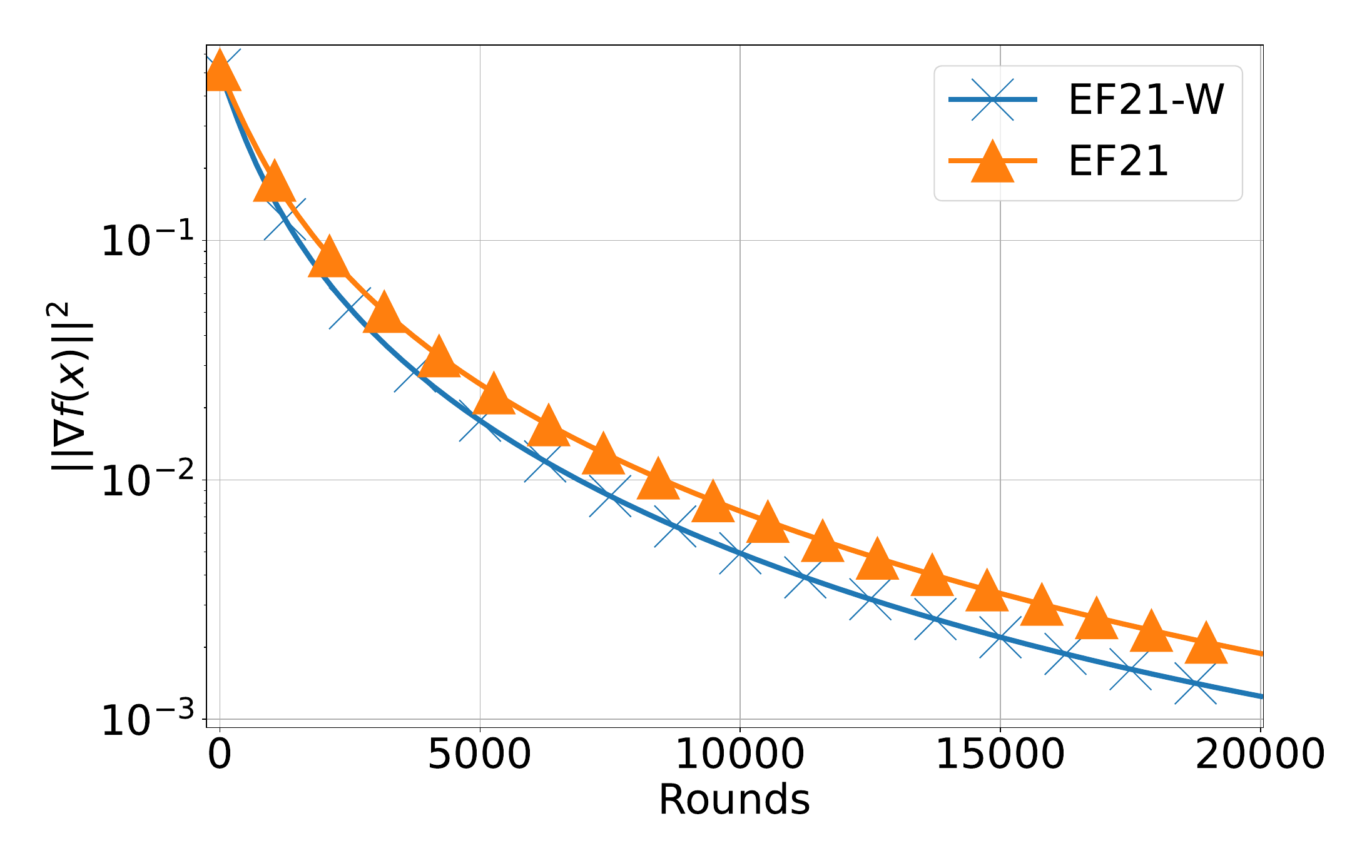} 
			\caption{ \small{(d) \texttt{W3A}, $\Lvar\approx1.58$ } }
		\end{subfigure}	
		\begin{subfigure}[ht]{\myVar}
			\includegraphics[width=\myVarN]{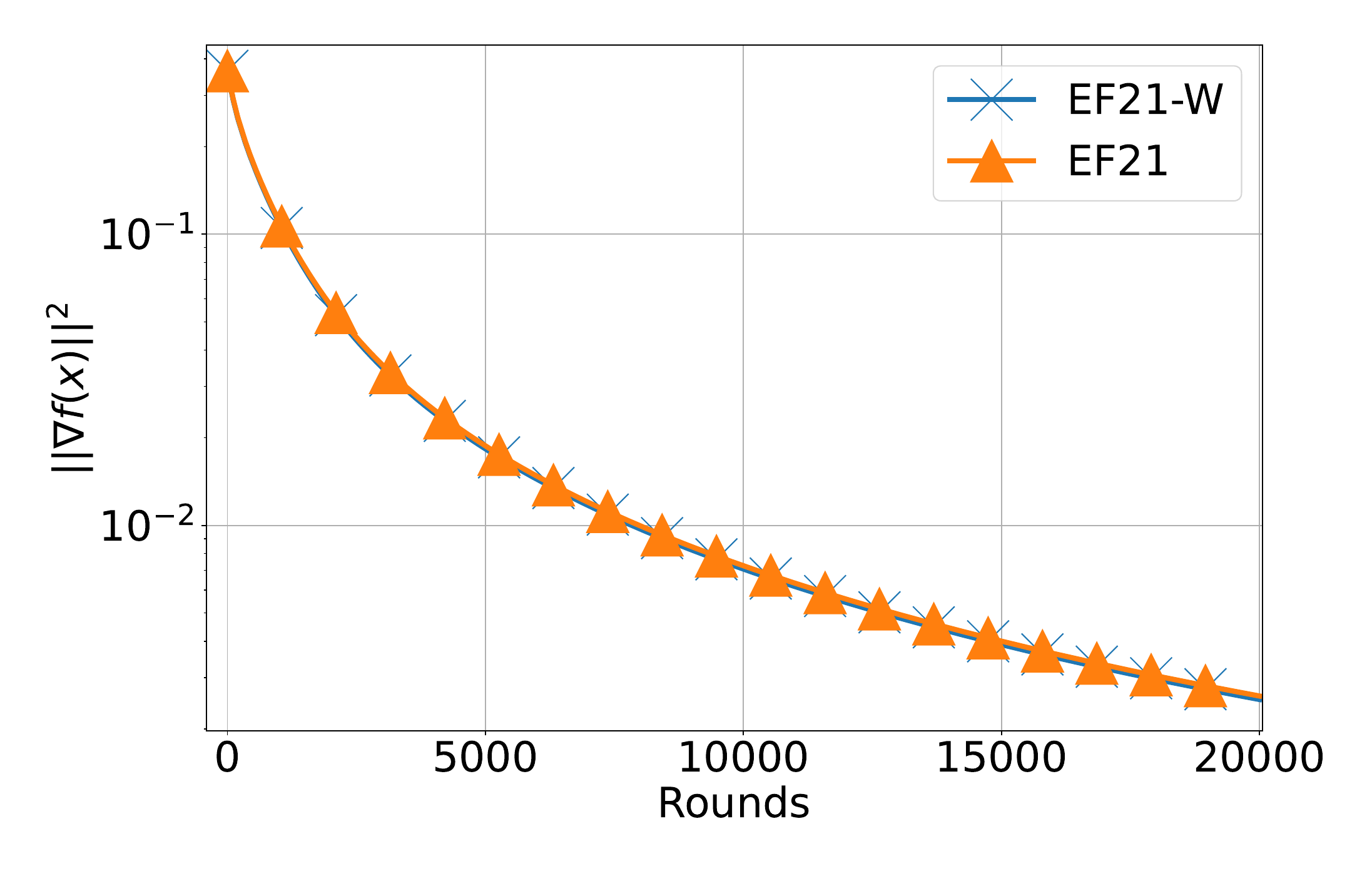} \caption{ \small{(e) \texttt{MUSHROOMS}, $\Lvar\approx 5 \times 10^{-1}$ } }
		\end{subfigure}
		\begin{subfigure}[ht]{\myVar}
			\includegraphics[width=\myVarN]{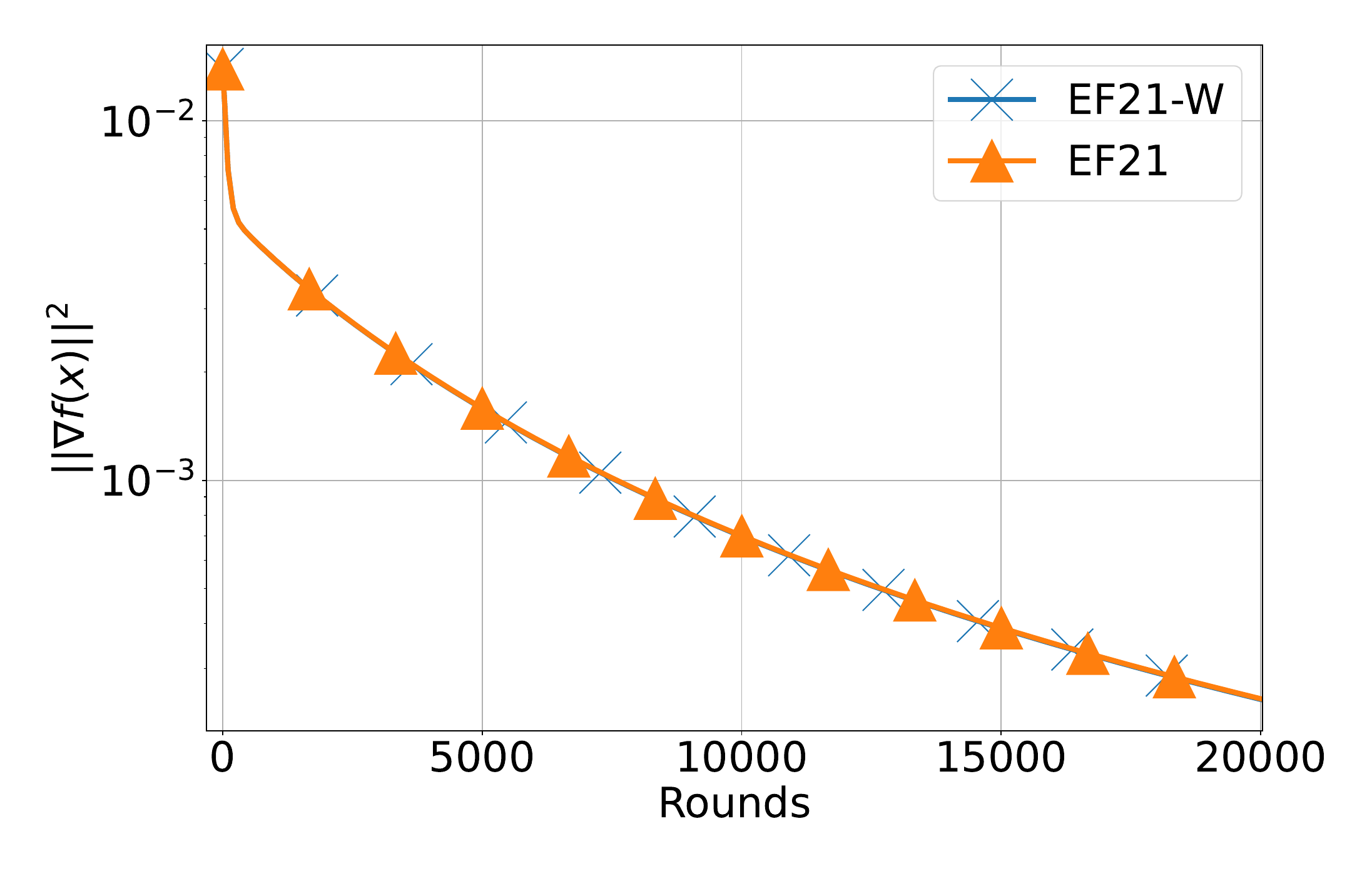} \caption{ \small{(f) \texttt{PHISHING}, $\Lvar=9 \times 10^{-4}$ } }
		\end{subfigure}
		
		\caption{\small{Comparison of  \algnamesmall{EF21} versus our new \algnamesmall{EF21-W} with the \algnamesmall{Top1} compressor on the non-convex logistic regression problem. The number of clients $n$ is $1,000$. The step size for~\algnamesmall{EF21} is set according to~\citep{EF21}, and the step size for~\algnamesmall{EF21-W} is set according to \Cref{thm:EF21-W}. The coefficient $\lambda$ for (b)--(f) is set to $0.001$, and for (a) is set to $1,000$ for numerical stability. We let $\Lvar \eqdef \LQMsq - \LAMsq = {\color{red}\avein L_i^2} - {\color{blue}\left(\avein L_i \right)^2}$.}}
		\label{fig:real-ef21-vc-ncvx}
	\end{figure*}
\end{center}
\begin{center}
	\begin{figure*}[t!]
		\centering
		\captionsetup[sub]{font=normalsize,labelfont={}}	
		\captionsetup[subfigure]{labelformat=empty}
		\newcommand{\sfwidth}{0.24\textwidth}
		\newcommand{\figwidth}{0.85\textwidth}
		
		\begin{subfigure}[ht]{\sfwidth}			\includegraphics[width=\figwidth]{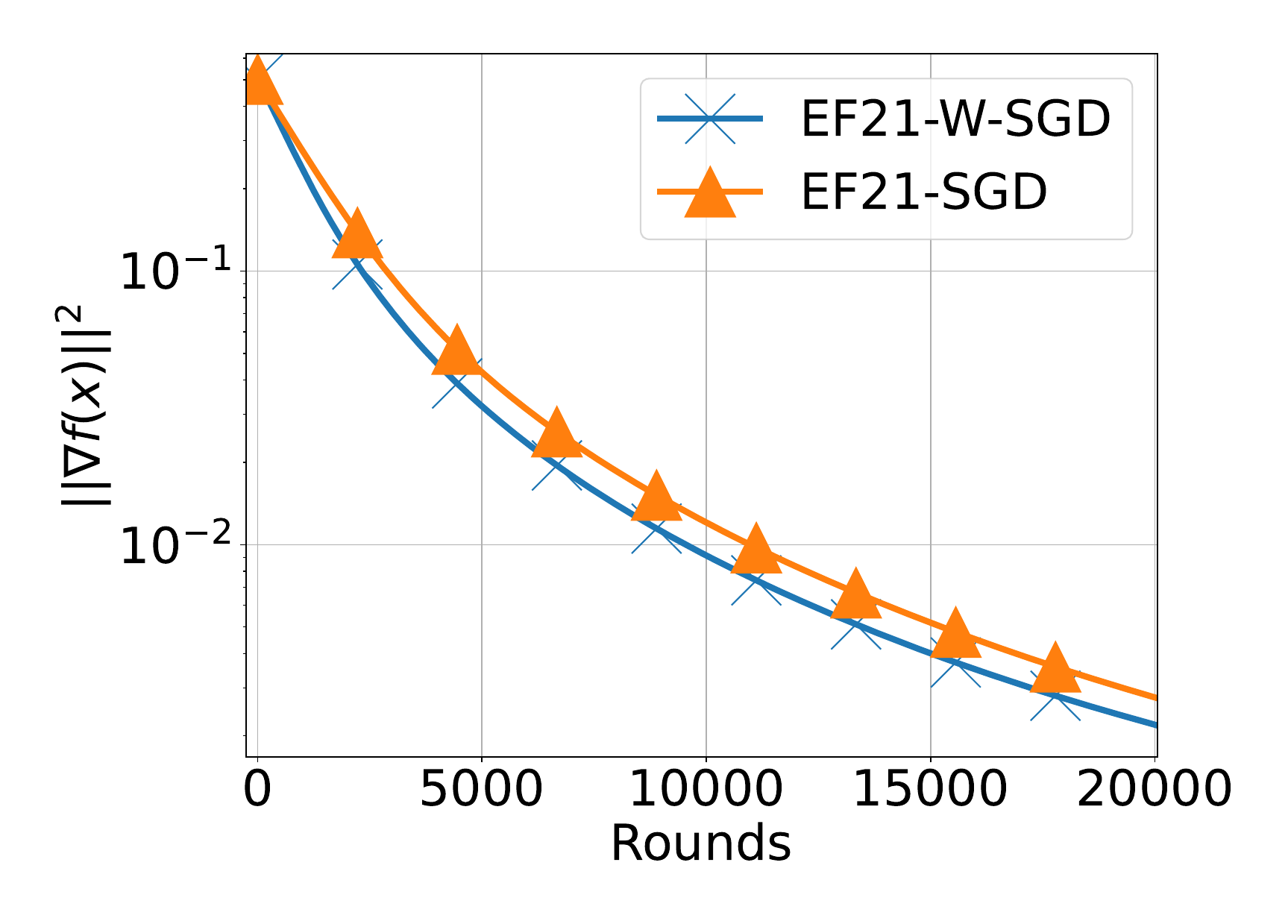} \caption{ \small{(a) \texttt{W1A}, \algnamesmall{SGD} } }
		\end{subfigure}
		\begin{subfigure}[ht]{\sfwidth}
			\includegraphics[width=\figwidth]{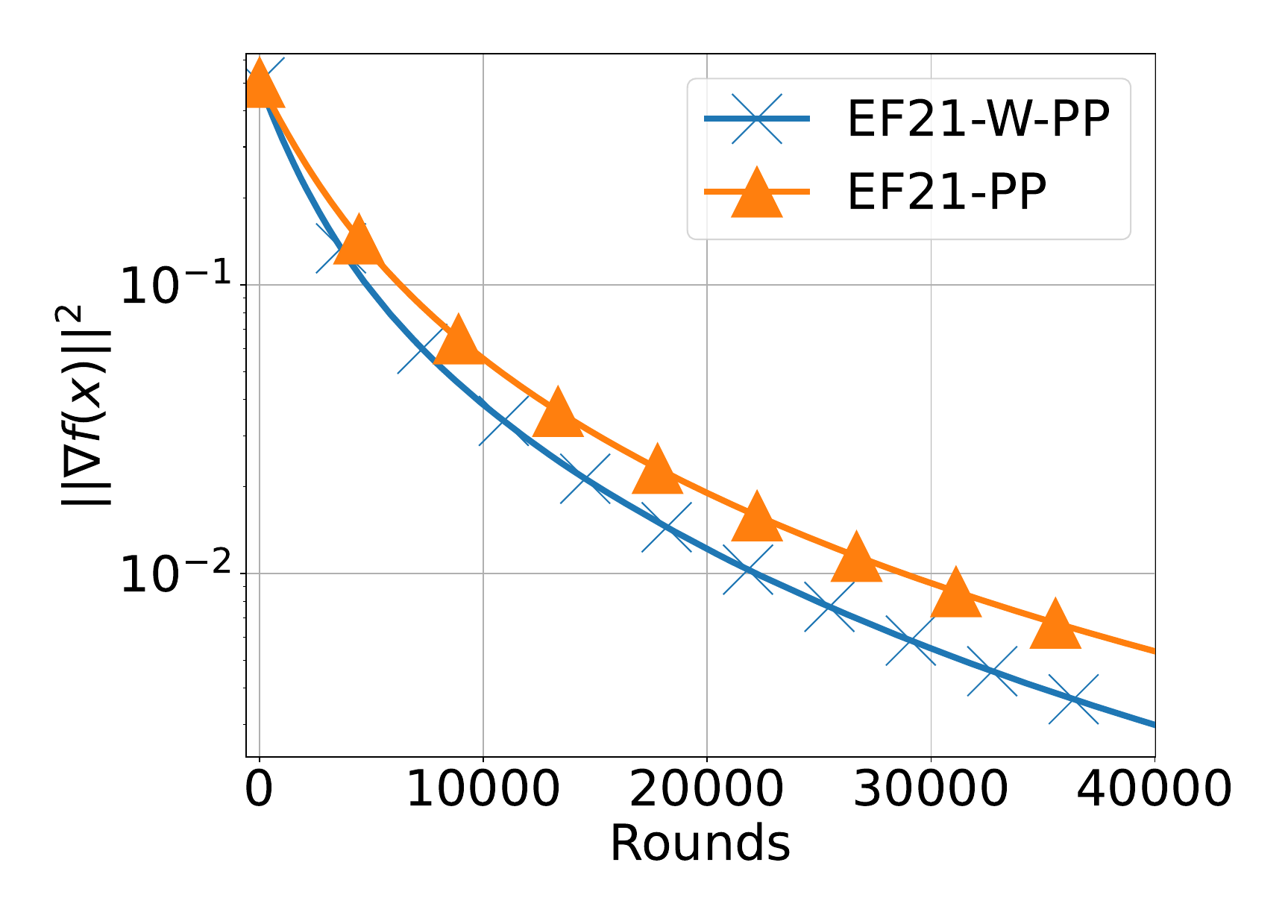} \caption{ \small{(b) \texttt{W1A}, \algnamesmall{PP} } }
		\end{subfigure}			
		\begin{subfigure}[ht]{\sfwidth}
			\includegraphics[width=\figwidth]{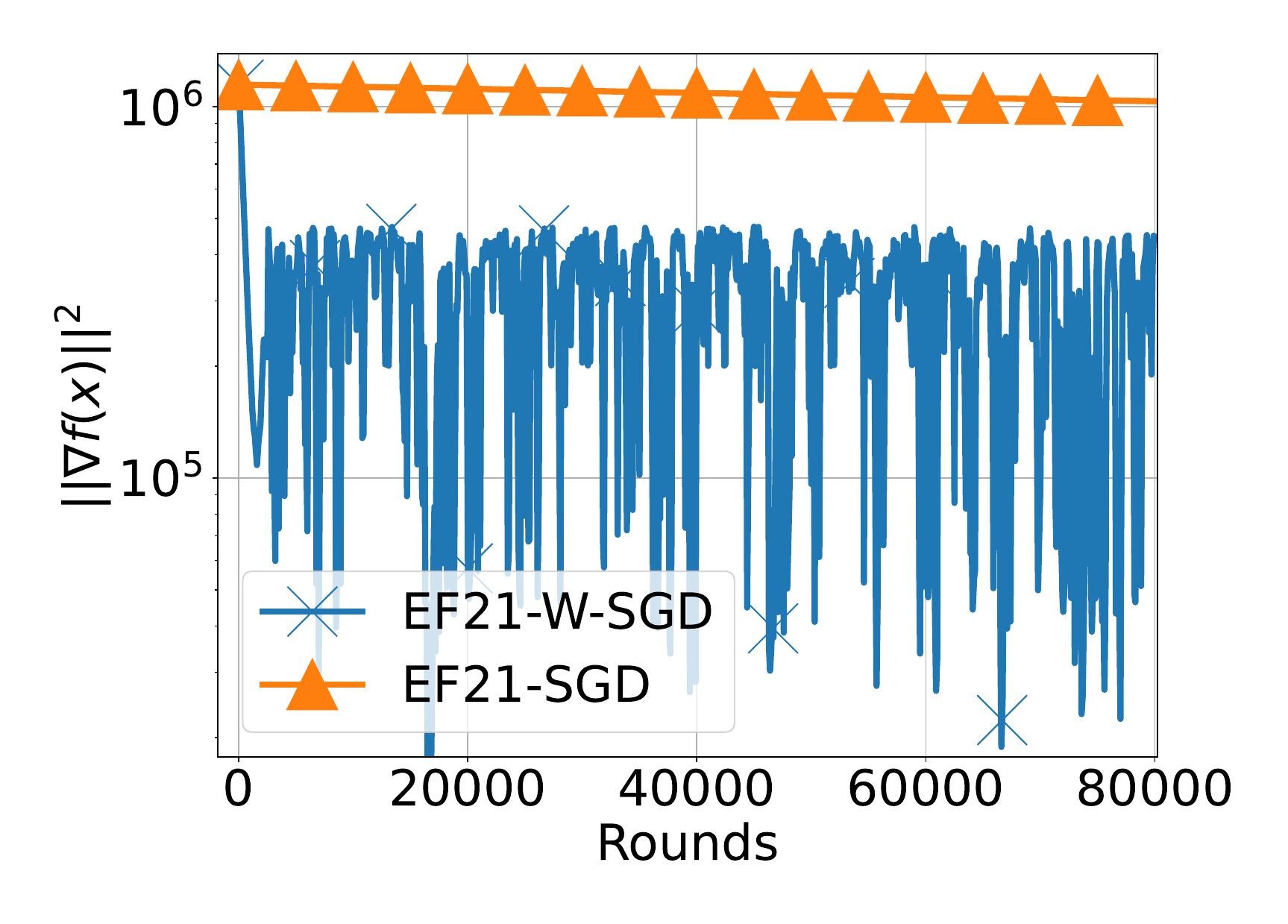} \caption{ \small{ (c) \texttt{AUSTRALIAN}, \algnamesmall{SGD} } }
		\end{subfigure}
		\begin{subfigure}[ht]{\sfwidth}
			\includegraphics[width=\figwidth]{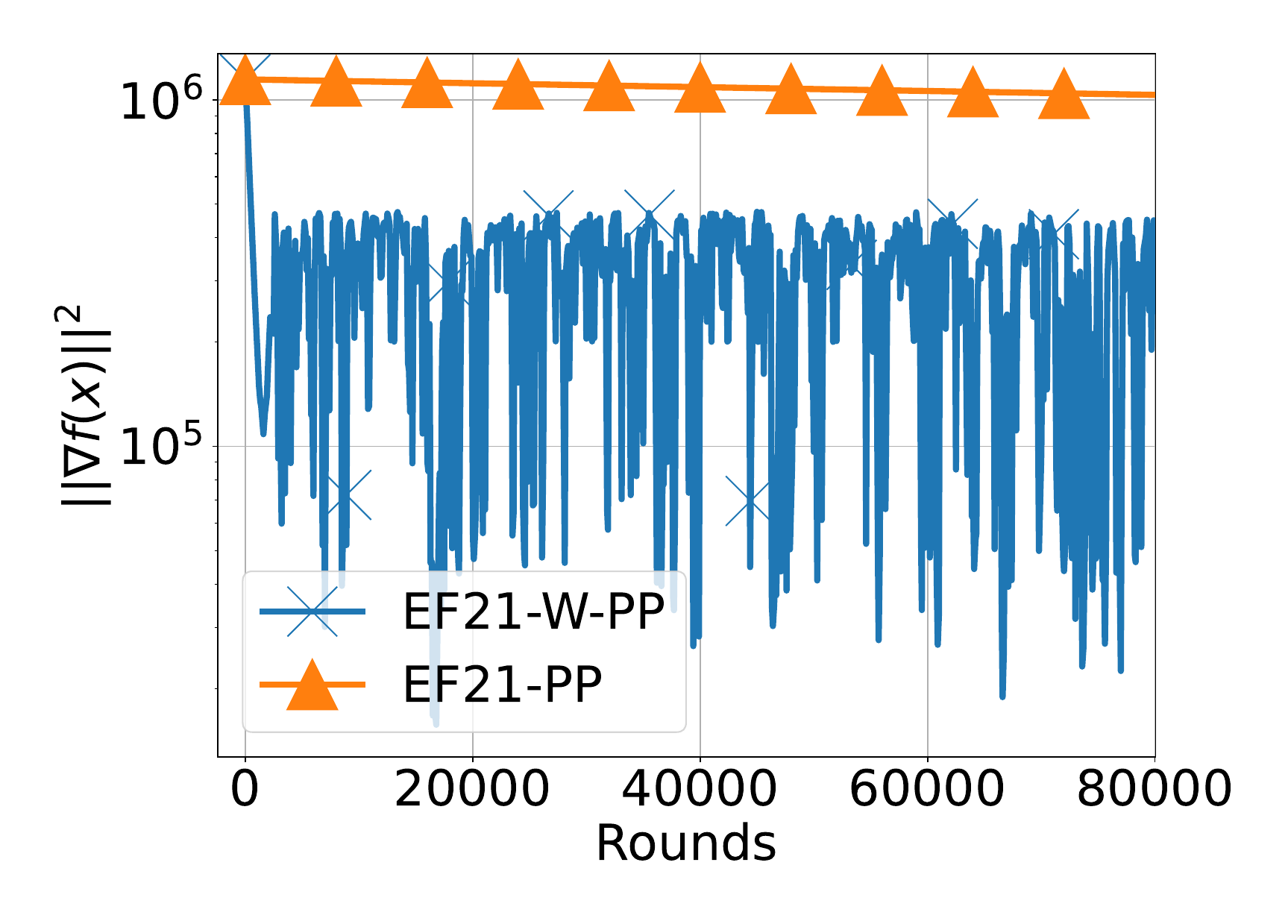} \caption{ \small{ (d) \texttt{AUSTRALIAN}, \algnamesmall{PP} } }
		\end{subfigure}
		
		\caption{\small{Comparison of \algnamesmall{EF21-W} with partial partial participation (\algnamesmall{EF21-W-PP}) or stochastic gradients (\algnamesmall{EF21-W-SGD}) versus \algnamesmall{EF21}  with partial partial participation (\algnamesmall{EF21-PP}) or stochastic gradients (\algnamesmall{EF21-SGD})~\citep{EF21BW}). The \algnamesmall{Top1} compressor was employed in all experiments. The number of clients $n=1,000$. All stepsizes are theoretical. The coefficient $\lambda$ was set to $0.001$ for (a), (b) and to $1,000$ for (c), (d).}}
		\label{fig:ext-real-ef21-vc-ncvx}
	\end{figure*}
\end{center}
\begin{center}	
	\begin{figure*}[t!]
		\centering
		\captionsetup[sub]{font=normalsize,labelfont={}}	
		\captionsetup[subfigure]{labelformat=empty}
		
		\begin{subfigure}[ht]{0.245\textwidth}
			\includegraphics[width=\textwidth]{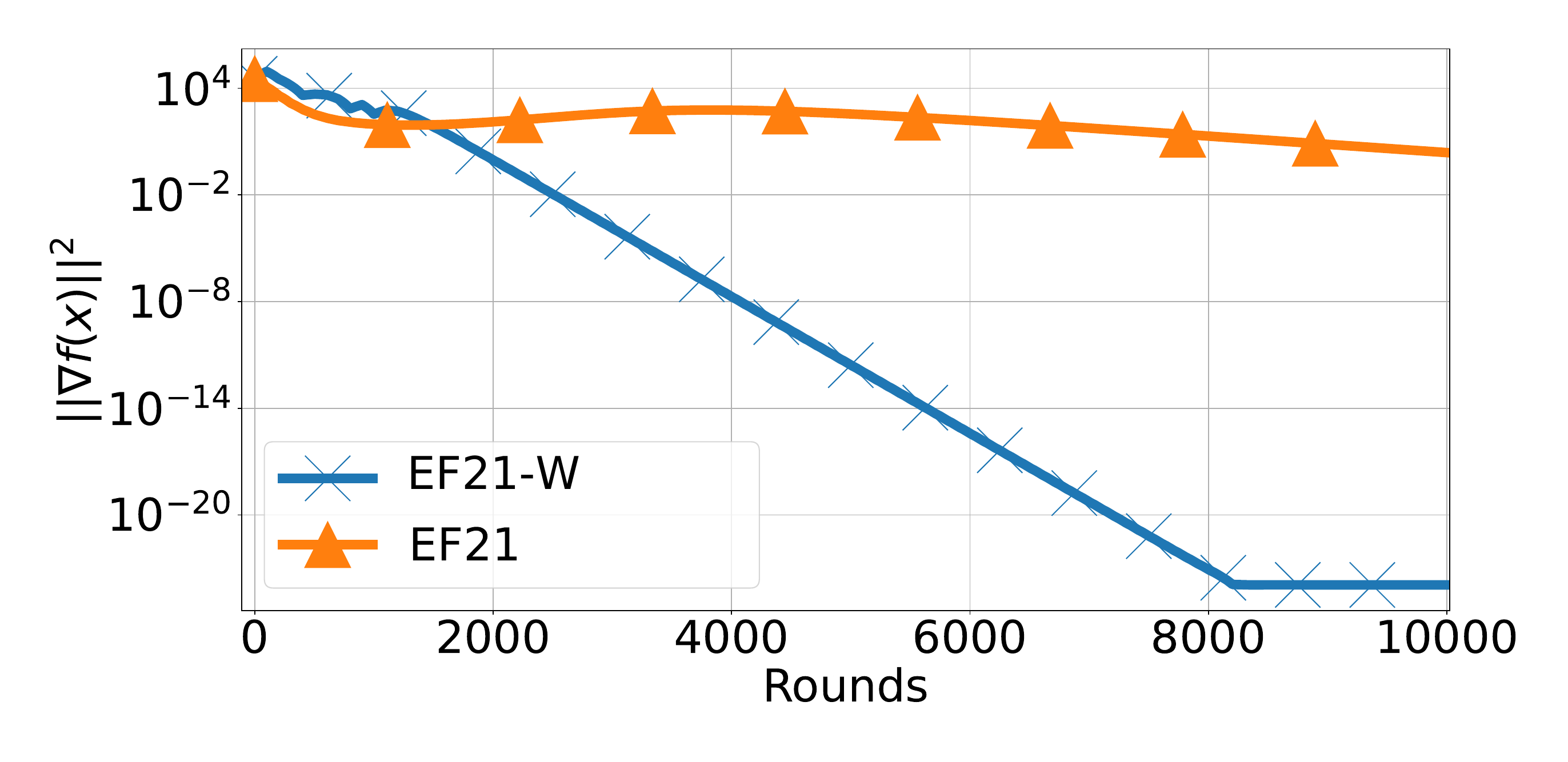} \caption{\small{(a) $\Lvar  \approx 4.4 \times 10^6$} \newline \scriptsize{$\LQM  \approx 2126, \LAM \approx 252$}}
		\end{subfigure}		
		\begin{subfigure}[ht]{0.245\textwidth}
			\includegraphics[width=\textwidth]{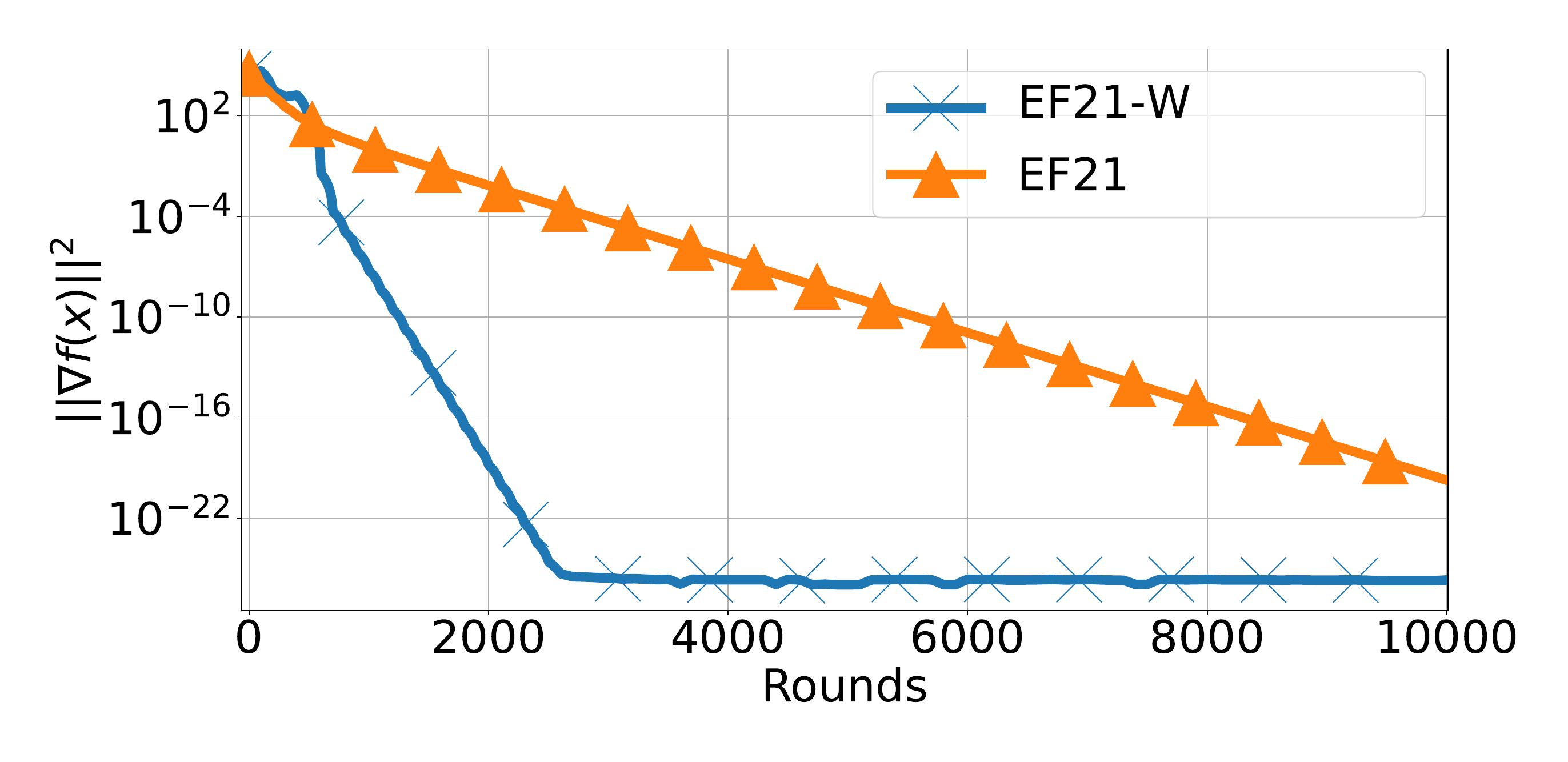} \caption{\small{(b) $\Lvar \approx 1.9 \times 10^6$} \newline \scriptsize{ $\LQM  \approx 1431, \LAM \approx 263$
			}} 
		\end{subfigure}
		\begin{subfigure}[ht]{0.245\textwidth}
			\includegraphics[width=\textwidth]{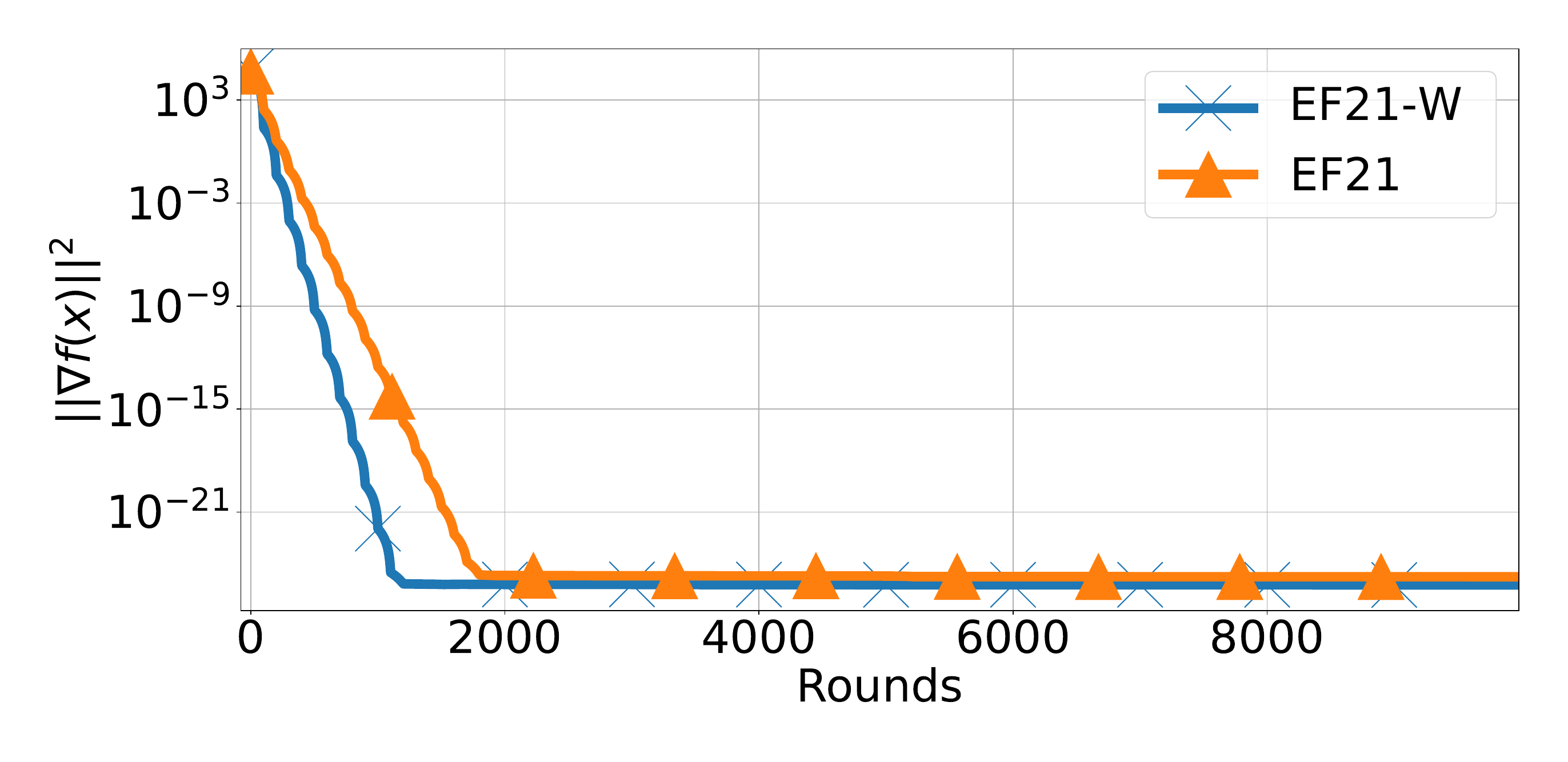} \caption{\small{(c) $\Lvar \approx 1.0 \times 10^5$} \newline \scriptsize{ $\LQM \approx 433, \LAM\approx 280$
			}}
		\end{subfigure}		
		\begin{subfigure}[ht]{0.245\textwidth}
			\includegraphics[width=\textwidth]{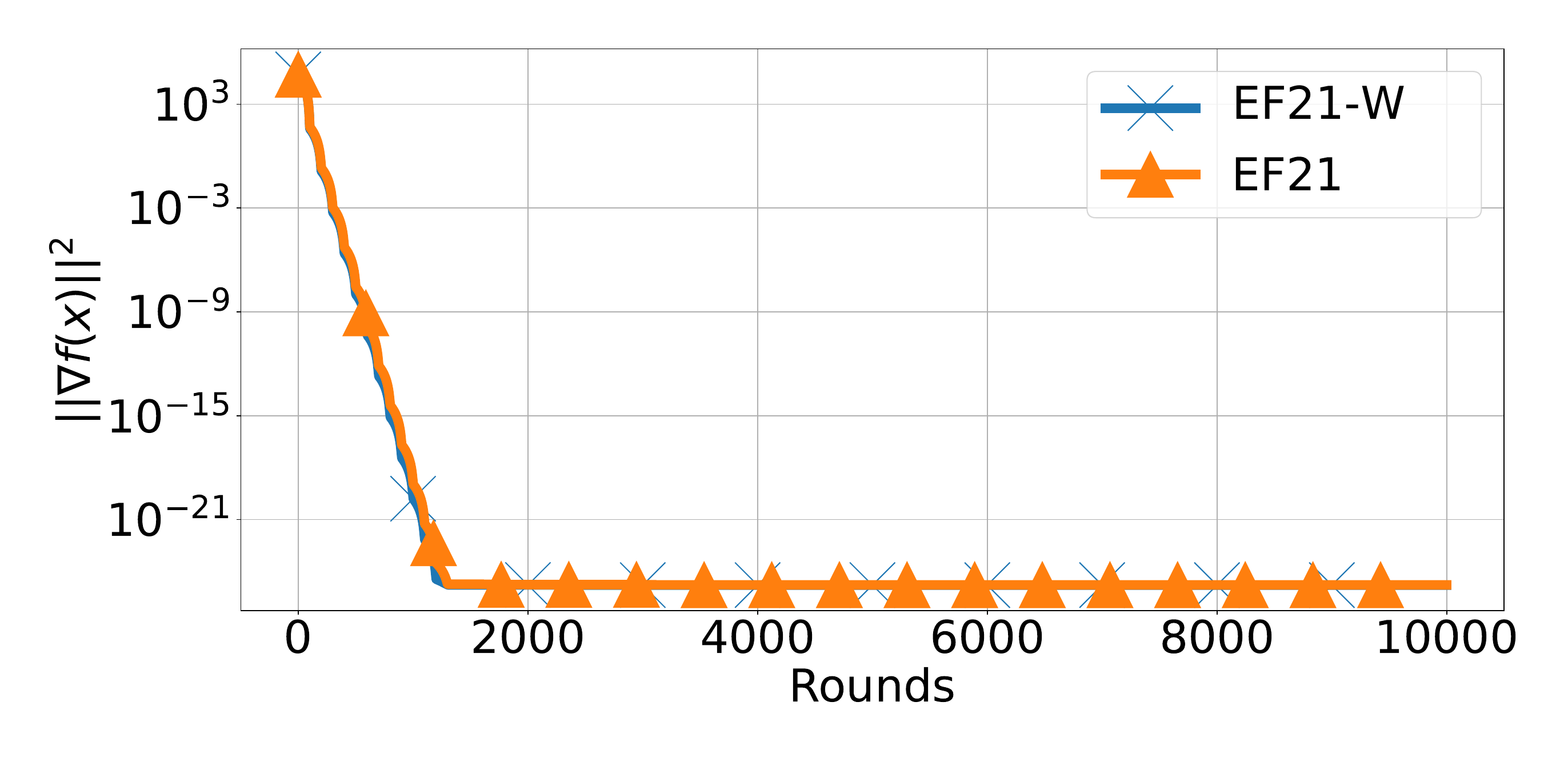} \caption{\small{(d) $\Lvar \approx 5.4 \times 10^3$} \newline \scriptsize{ $\LQM  \approx 294, \LAM \approx 285$
			}} 
		\end{subfigure}
		
		\caption{ \small{Comparison of~\algnamesmall{EF21} vs.~\algnamesmall{EF21-W} with the~\algnamesmall{Top1} compressor on the non-convex linear problem. The number of clients $n$ is $2,000$. The coefficient $\lambda$ has been set to $100$. The step size for~\algnamesmall{EF21} is set according to~\citep{EF21}, and the step size for~\algnamesmall{EF21-W} is set according to \Cref{thm:EF21-W}. In all cases, the smoothness constant $L$ equals $50$}.}
		\label{fig:syn-ef21-vc-noncvx}
	\end{figure*}
\end{center}

%


\clearpage
\section*{Acknowledgements}

This work of all authors was supported by the KAUST Baseline Research Scheme (KAUST BRF). The work Peter Richt\'{a}rik and Konstantin Burlachenko was also supported by the SDAIA-KAUST Center of Excellence in Data Science and Artificial Intelligence (SDAIA-KAUST AI). We wish to thank Babis Kostopoulos---a VSRP intern at KAUST who spent some time working on this project in Summer 2023---for helping with some parts of the project. We offered Babis co-authorship, but he declined. 


\bibliography{bibliography}
\bibliographystyle{iclr2024_conference}

\newpage
\appendix

\section{Basic Results and Lemmas}
In this section, we offer a few results that serve as essential prerequisites for establishing the main findings in the paper.

\subsection{Optimal client cloning frequencies}

\begin{lemma}[Optimal weights] \label{lem:ow}
Let $a_i > 0$ for $i \in [n]$. Then
\begin{equation}\label{eq:0998fddff-89y9fd}
\min\limits_{ \substack{w_1 > 0, \dots, w_n>0 \\ \sumin w_i = 1}} \sumin \frac{a_i^2}{w_i} =\left(\sumin a_i\right)^2,
\end{equation}
which is achieved when $w^\ast_i = \frac{a_i}{\sum_j a_j}$. This means that
\begin{equation}\label{eq:0998fddff}
\min\limits_{ \substack{w_1 > 0, \dots, w_n>0 \\ \sumin w_i = 1}} \frac{1}{n}\sqrt{\sumin \frac{a_i^2}{w_i}} =\frac{1}{n}\sumin a_i.
\end{equation}
\end{lemma}

We now show that the cloning frequencies given by $N^\star_i =  \left \lceil  \frac{L_i}{\LAM} \right \rceil$ form a $\sqrt{2}$-approximation for the optimization problem of finding the optimal integer client frequencies. 
 
\begin{lemma}[$\sqrt{2}$-approximation]\label{lem:sandwitch} If we let $N^\star_i =  \left \lceil  \frac{L_i}{\LAM} \right \rceil$ for all $i\in [n]$, then 
\[ \LAM \leq  \min_{N_1 \in \mathbb{N},\dots, N_n\in \mathbb{N}} M(N_1,\dots,N_n) \leq   M(N^\star_1,\dots,N^\star_n) \leq \sqrt{2} \LAM.\]
\end{lemma}
\begin{proof}  Recall that
\[M(N_1,\dots,N_n) \eqdef  \frac{1}{n}\sqrt{\sumin \frac{L_i^2}{N_i/N}} .\]
The first inequality in the claim follows by relaxing the integrality constraints, which gives us the bound
\[\min\limits_{ \substack{w_1 > 0, \dots, w_n>0 \\ \sumin w_i = 1}} \frac{1}{n}\sqrt{\sumin \frac{L_i^2}{w_i}}   \leq  \min_{N_1 \in \mathbb{N},\dots, N_n\in \mathbb{N}} \frac{1}{n}\sqrt{\sumin \frac{L_i^2}{N_i/N}}  , \] 
and subsequently applying Lemma~\ref{eq:0998fddff-89y9fd}. 

Next, we argue that the quantity $N^\star \eqdef \sum_i N^\star_i$ is at most $2n$. Indeed,
\begin{equation}\label{eq:98y08fd=-=f-d} N^\star = \sum_{i=1}^n N^\star_i = \sum_{i=1}^n \left \lceil  \frac{L_i}{\LAM} \right \rceil \leq \sum_{i=1}^n \left( \frac{L_i}{\LAM} + 1\right) = 2n.\end{equation}
We will now use this to bound $M(N^\star_1,\dots,N^\star_n)$ from above:
\begin{eqnarray*}
M(N^\star_1,\dots,N^\star_n) = \frac{1}{n}\sqrt{\sumin \frac{L_i^2}{N^\star_i/N^\star}} 
\overset{\eqref{eq:98y08fd=-=f-d}}{=}  \frac{\sqrt{2}}{\sqrt{n}}\sqrt{\sumin \frac{L_i^2}{N^\star_i}}  =  \frac{\sqrt{2}}{\sqrt{n}}\sqrt{\sumin \frac{\frac{L_i}{\LAM} }{N^\star_i}L_i \LAM} .
\end{eqnarray*}
Since $ \frac{\frac{L_i}{\LAM} }{N^\star_i} \leq 1$ for all $i\in [n]$, the proof is finished as follows: \begin{eqnarray*}
M(N^\star_1,\dots,N^\star_n) \leq \frac{\sqrt{2}}{\sqrt{n}}\sqrt{\sumin L_i \LAM}  =  \frac{\sqrt{2}}{\sqrt{n}} \sqrt{\LAM}\sqrt{\sumin L_i }  = \sqrt{2} \LAM.
\end{eqnarray*}
\end{proof}

\subsection{Descent lemma}

\begin{lemma}[\citet{PAGE2021}]\label{lm:descent_lemma}
	Let \Cref{as:smooth} hold and $x^{t+1} = x^t - \gamma g^t$, where $g^t \in \RR^d$ is any vector, and $\gamma > 0$ is any scalar. Then, we have
	\begin{equation}\label{eq:descent_lemma}
		f(x^{t+1}) \leq f(x^t) - \frac{\gamma}{2} \|\nabla f(x^t)\|^2 - \left(\frac{1}{2\gamma} - \frac{L}{2}\right) \|x^{t+1} - x^t\|^2 + \frac{\gamma}{2} \|g^t - \nabla f(x^t)\|^2.
	\end{equation}
\end{lemma}

\subsection{Young's inequality}

\begin{lemma}[Young's inequality]
For any $a, b \in \RR^d$ and any positive scalar $s > 0$ it holds that
\begin{equation}\label{eq:young}
\|a + b\|^2 \leq (1 + s) \|a\|^2 + (1 + s^{-1})\|b\|^2.
\end{equation}
\end{lemma}

\subsection{$2$-Suboptimal but simple stepsize rule}

\begin{lemma}[Lemma 5,~\citet{EF21}]\label{lm:stepsize_bound}
Let $a, b > 0$. If $0 < \gamma \leq \frac{1}{\sqrt{a} + b},$ then $a \gamma^2 + b \gamma \leq 1$. Moreover, the bound is tight up to the factor of 2 since $$\frac{1}{\sqrt{a} + b} \leq \min\left\{\frac{1}{\sqrt{a}}, \frac{1}{b}\right\} \leq \frac{2}{\sqrt{a} + b}.$$
\end{lemma}

\subsection{Optimal coefficient in Young's inequality}

\begin{lemma}[Lemma 3, \citet{EF21}]\label{lm:best_theta_beta_choice}
Let $0 < \alpha \leq 1$ and for $s> 0$, let  $\theta(\alpha,s) \eqdef 1- (1- \alpha )(1+s)$ and $\beta(\alpha,s)\eqdef (1- \alpha ) \left(1+ s^{-1} \right)$. Then, the solution of the optimization problem
\begin{equation}
\min\limits_{s} \left\{\frac{\beta(\alpha,s)}{\theta(\alpha,s)} : 0 < s < \frac{\alpha}{1 - \alpha} \right\}
\end{equation}
is given by $s^\ast = \frac{1}{\sqrt{1 - \alpha}} - 1$. Furthermore, $\theta(\alpha,s^\ast) = 1 - \sqrt{1-\alpha}$, and $\beta(\alpha,s^\ast) = \frac{1-\alpha}{1 - \sqrt{1 - \alpha}}$.
\end{lemma}

\clearpage
\section{Cloning reformulation for Polyak-\L ojaschewitz functions}

For completeness, we also provide a series of convergence results under Polyak-\L ojasiewicz condition. We commence our exposition with the subsequent definition.

\begin{assumption}[Polyak-\L ojasiewicz]\label{as:PL}
	There exists a positive scalar $\mu > 0$ such that for all points $x \in \RR^d$, the following inequality is satisfied: 
	\begin{equation}\label{eq:PL}
		f(x) - f(x^\ast) \leq \frac{1}{2\mu}\|\nabla f(x)\|^2,
	\end{equation} where $x^\ast \eqdef \argmin f(x)$.
\end{assumption}

\begin{theorem}
Let Assumptions~\ref{as:smooth},~\ref{as:L_i},~and~\ref{as:PL} hold. Assume that $\cC_i^t \in \mathbb{C}(\alpha)$ for all $i \in [n]$ and $t \geq 0$. Consider   \Cref{alg:EF21} (\algname{EF21}) applied to the ``cloning'' reformulation~\ref{eq:cloning} of the distributed optimization problem~\eqref{eq:main_problem}, where $N_i^\ast = \lceil \frac{L_i}{\LAM} \rceil$ for all $i \in [n]$. Let the stepsize be set as
$$
0 \leq \gamma \leq \min\left\{ \left(L + \sqrt{2} \LAM \sqrt{\frac{2\beta}{\theta}} \right)^{-1}, \frac{\theta}{2\mu}\right\},
$$
where $\theta = 1 - \sqrt{1 - \alpha}$ and $\beta = \frac{1 - \alpha}{1 - \sqrt{1 - \alpha}}$.
Let $$\Psi^t \eqdef f(x^t) - f(x^\ast) + \frac{\gamma}{\theta} G^t .$$ Then, for any $T \geq 0$, we have
$$
\ExpBr{\Psi^T} \leq (1 - \gamma \mu)^T \Psi^0.
$$
\end{theorem}
\begin{proof}
This theorem is a corollary of Theorem 2 in~\citep{EF21} and  \Cref{lem:sandwitch}.
\end{proof}

\clearpage
\section{Proof of  \Cref{thm:EF21-W} (Theory for EF21-W)}

In this section, we present a proof for  \Cref{thm:EF21-W}. To start this proof, we establish a corresponding contraction lemma. We define the following quantities:
\begin{equation}\label{eq:weighted_def_grad_distortion}
	G_i^{t} \eqdef \sqnorm{ g_i^t - \frac{\nabla f_i(x^{t})}{n w_i}}; \qquad G^t \eqdef \sumin w_i G_i^t,
\end{equation}
where the weights $w_i$ are defined as specified in  \Cref{alg:EF21-W}, that is, 
\begin{equation}\label{eq:weight_definition}
w_i = \frac{L_i}{\sum_{j=1}^n L_j}.
\end{equation}

\subsection{A lemma}

With these definitions in place, we are now prepared to proceed to the lemma.

\begin{lemma} Let $\cC_i^t\in \mathbb{C}(\alpha)$ for all $i\in [n]$ and $t\geq 0$. Let $W^t \eqdef \{g_1^t, g_2^t, \dots, g_n^t, x^t, x^{t+1}\}$. Then, for iterates of \Cref{alg:EF21-W} we have
	\begin{equation}\label{eq:weighted_ind_grad_dist_evolution}
	\ExpBr{ G_i^{t+1} \;|\; W^t} \leq (1-\theta(\alpha,s))   G_i^t + \beta(\alpha,s) \frac{1}{n^2 w_i^2} \sqnorm{\nabla f_i(x^{t+1}) - \nabla f_i(x^t)},
	\end{equation}
	and 
	\begin{equation}\label{eq:weighted_full_grad_dist_evolution}
	\ExpBr{G^{t+1} } {\leq} \rb{1 - \theta(\alpha,s)} \ExpBr{G^t} + \beta(\alpha,s) \LAMsq \ExpBr{ \sqnorm{x^{t+1} - x^t}},
	\end{equation}
	where $s > 0$ is an arbitrary positive scalar, and
	\begin{equation}\label{eq:theta-beta-def} \theta(\alpha,s) \eqdef 1- (1- \alpha )(1+s), \qquad \text{and} \qquad \beta(\alpha,s)\eqdef (1- \alpha ) \left(1+ s^{-1} \right). \end{equation}
\end{lemma}

\begin{proof}
	
The proof is straightforward and bears resemblance to a similar proof found in a prior work \citep{EF21}.
	\begin{eqnarray*}
		\ExpBr{ G_i^{t+1} \;|\; W^t} & \overset{\eqref{eq:weighted_def_grad_distortion}}{=} & \ExpBr{  \sqnorm{g_i^{t+1} 
				- \frac{\nabla f_i(x^{t+1})}{n w_i}}  \;|\; W^t}	 \\
		&= & \ExpBr{  \sqnorm{g_i^t + \cC_i^t \left( \frac{\nabla f_i(x^{t+1})}{n w_i} - g_i^t \right) 
				- \frac{\nabla f_i(x^{t+1})}{n w_i}}  \;|\; W^t}	 \\
		&\overset{\eqref{eq:compressor_contraction}}{\leq} &  (1-\alpha) 
		\sqnorm{\frac{\nabla f_i(x^{t+1})}{n w_i} - g_i^t} \\
		& = & (1-\alpha) 
		\sqnorm{\frac{\nabla f_i(x^{t+1})}{n w_i} - \frac{\nabla f_i(x^{t})}{n w_i} + \frac{\nabla f_i(x^{t})}{n w_i} -  g_i^t} \\
		&\overset{\eqref{eq:young}}{\leq}& (1-\alpha) (1+ s) \sqnorm{ \frac{\nabla f_i(x^{t})}{n w_i} - g_i^t}\\
		&& \qquad  + (1-\alpha)  \left(1+s^{-1}\right) \frac{1}{n^2 w_i^2} 
		\sqnorm{\nabla f_i(x^{t+1}) - \nabla f_i(x^t)}, 
	\end{eqnarray*}
with the final inequality holding for any positive scalar $s > 0$. Consequently, we have successfully established the first part of the lemma.

By employing~\eqref{eq:weighted_def_grad_distortion} and the preceding inequality, we can derive the subsequent bound for the conditional expectation of $G^{t+1}$:
	\begin{eqnarray}
	\ExpBr{G^{t+1} \mid W^t } &\overset{\eqref{eq:weighted_def_grad_distortion}}{=} & \ExpBr{\sumin w_i G_i^{t+1} \mid W^t} \notag \\
	&\overset{\eqref{eq:weighted_def_grad_distortion}}{=} & \sumin w_i \ExpBr{\sqnorm{ g_i^{t+1} - \frac{\nabla f_i(x^{t+1})}{n w_i}} \mid W^t}  \notag \\
	&\overset{\eqref{eq:weighted_ind_grad_dist_evolution}}{\leq} & \rb{1 - \theta(\alpha,s)} \sumin w_i 
	\sqnorm{g_i^t - \frac{\nabla f_i(x^{t})}{n w_i} } \notag \\
	&& \quad+ \beta(\alpha,s) \sumin \frac{w_i}{w_i^2 n^2} \sqnorm {\nabla f_i(x^{t+1}) - \nabla f_i(x^t)}  .\label{eq:weighted_aux_1}
	\end{eqnarray}
Applying \Cref{as:L_i} and~\eqref{eq:weight_definition}, we further proceed to:
\begin{eqnarray}
	\ExpBr{G^{t+1} \mid W^t } &\overset{\eqref{eq:weighted_aux_1}}{\leq} &  \rb{1 - \theta(\alpha,s)} \sumin w_i 
	\sqnorm{g_i^t - \frac{\nabla f_i(x^{t})}{n w_i} } 
	+ \beta(\alpha,s) \sumin \frac{w_i}{w_i^2 n^2} \sqnorm {\nabla f_i(x^{t+1}) - \nabla f_i(x^t)}  \notag \\
	&\overset{\eqref{eq:weighted_def_grad_distortion}}{=}  & \rb{1 - \theta(\alpha,s)}G^t+ \beta(\alpha,s) \sumin \frac{1}{w_i n^2} \sqnorm {\nabla f_i(x^{t+1}) - \nabla f_i(x^t)}   \notag \\
	&\overset{\eqref{eq:L_i}}{\leq} &  \rb{1 - \theta(\alpha,s)}G^t
	+ \beta(\alpha,s) \rb{\sumin \frac{L_i^2}{w_i n^2} }\sqnorm{x^{t+1} - x^t}   \notag \\
	&\overset{\eqref{eq:weight_definition}}{=} & \rb{1 - \theta(\alpha,s)}G^t
	+ \beta(\alpha,s) \rb{\sumin \frac{L_i^2}{\frac{L_i}{\sumjn L_j} n^2}}\sqnorm{x^{t+1} - x^t}   \notag \\
	&=& \rb{1 - \theta(\alpha,s)}G^t
	+ \beta(\alpha,s) \rb{\sumin \frac{L_i \sumjn L_j}{n^2}}\sqnorm{x^{t+1} - x^t}  \notag  \\
	&= & \rb{1 - \theta(\alpha,s)}G^t
	+ \beta(\alpha,s) \LAMsq \sqnorm{x^{t+1} - x^t} .\label{eq:weighted_aux_2}
\end{eqnarray}

Using the tower property,  we get 
\begin{eqnarray*}
	\ExpBr{G^{t+1}} = \ExpBr{\ExpBr{G^{t+1} \mid W^t}} \overset{\eqref{eq:weighted_aux_2}}{\leq} \rb{1 - \theta(\alpha,s)}\ExpBr{G^t}+ \beta(\alpha,s) \LAMsq \ExpBr{ \sqnorm{x^{t+1} - x^t}},
\end{eqnarray*}
and this finalizes the proof.
\end{proof}

\subsection{Main result}

We are now prepared to establish the proof for  \Cref{thm:EF21-W}.
\begin{proof} Note that, according to~\eqref{eq:AlgStep1X}, the gradient estimate for \Cref{alg:EF21-W} gets the following form:
\begin{equation}\label{eq:weighted_grad_estimate_def}
g^t=\sum\limits_{i=1}^{n} w_i g_i^t.
\end{equation}
Using \Cref{lm:descent_lemma} and Jensen's inequality applied to the function $x\mapsto \sqnorm{x}$ (since $\sumin w_i = 1$), we obtain the following bound:
\begin{eqnarray}
	f(x^{t+1}) &\overset{\eqref{eq:descent_lemma}}{\leq} & 
	f(x^{t})-\frac{\gamma}{2}\sqnorm{\nabla f(x^{t})}-\left(\frac{1}{2 \gamma}-\frac{L}{2}\right)\sqnorm{x^{t+1}-x^{t}}+\frac{\gamma}{2}\sqnorm{
		g^t- \sumin \nabla f_i(x^{t}) } \notag \\
	& \overset{\eqref{eq:weighted_grad_estimate_def}}{=} & 
	f(x^{t})-\frac{\gamma}{2}\sqnorm{\nabla f(x^{t})}-\left(\frac{1}{2 \gamma}-\frac{L} 
	{2}\right)\sqnorm{x^{t+1}-x^{t}}+\frac{\gamma}{2}\sqnorm{
		\sumin w_i \left(g_i^t-  \frac{\nabla f_i(x^{t})}{n w_i} \right) }   \notag \\
	& \leq & 
	f(x^{t})-\frac{\gamma}{2}\sqnorm{\nabla f(x^{t})}-\left(\frac{1}{2 \gamma}-\frac{L}       
	{2}\right)\sqnorm{x^{t+1}-x^{t}}+\frac{\gamma}{2}
	\sumin w_i \sqnorm{ g_i^t-  \frac{\nabla f_i(x^{t})}{n w_i}  }  \notag \\
	& \overset{\eqref{eq:weighted_def_grad_distortion}}{=}  & f(x^{t})-\frac{\gamma}{2}\sqnorm{\nabla f(x^{t})}-\left(\frac{1}{2 \gamma}-\frac{L}       
	{2}\right)\sqnorm{x^{t+1}-x^{t}}+\frac{\gamma}{2} G^t.\label{eq:weighted_aux_3}
\end{eqnarray}

Subtracting $f^\ast$ from both sides and taking expectation, we get
\begin{align}
	\ExpBr{f(x^{t+1})-f^\ast} & \leq  \ExpBr{f(x^{t})-f^\ast}
	-\frac{\gamma}{2} \ExpBr{\sqnorm{\nabla f(x^{t})}}  \notag \\
	& \qquad -\left(\frac{1}{2 \gamma}
	-\frac{L}{2}\right) \ExpBr{\sqnorm{x^{t+1}-x^{t}}}+ \frac{\gamma}{2}\ExpBr{G^t}.\label{eq:weighted_function_descent}
\end{align}

Let $\delta^{t} \eqdef \ExpBr{f(x^{t})-f^\ast}$, $s^{t} \eqdef \ExpBr{G^t }$ and 
$r^{t} \eqdef\ExpBr{\sqnorm{x^{t+1}-x^{t}}}.$
Subsequently, by adding \eqref{eq:weighted_full_grad_dist_evolution} with a $\frac{\gamma}{2 \theta(\alpha,s)}$ multiplier, we obtain 
\begin{eqnarray*}
	\delta^{t+1}+\frac{\gamma}{2 \theta(\alpha,s)} s^{t+1} &\overset{\eqref{eq:weighted_function_descent}}{\leq} & \delta^{t}-\frac{\gamma}{2}\sqnorm{\nabla f(x^{t})}-\left(\frac{1}{2 \gamma}-\frac{L}{2}\right) r^{t}+\frac{\gamma}{2} s^{t} + \frac{\gamma}{2 \theta} s^{t+1}\\
	&\overset{\eqref{eq:weighted_full_grad_dist_evolution}}{\leq}& \delta^{t}-\frac{\gamma}{2}\sqnorm{\nabla f(x^{t})}-          \left(\frac{1}{2 \gamma}-\frac{L}{2}\right) r^{t}+\frac{\gamma}{2} s^{t} \notag \\ 
& & \qquad	+\frac{\gamma}{2 \theta(\alpha,s)}\left(\beta(\alpha,s) \LAMsq r^t + (1 - \theta(\alpha,s)) s^{t}\right) \\
	&=&\delta^{t}+\frac{\gamma}{2\theta(\alpha,s)} s^{t}-\frac{\gamma}{2}\sqnorm{\nabla f(x^{t})}-\left(\frac{1}{2\gamma} -\frac{L}{2} - \frac{\gamma}{2\theta(\alpha,s)}\beta(\alpha,s) \LAMsq \right) r^{t} \\
	& \leq& \delta^{t}+\frac{\gamma}{2\theta(\alpha,s)} s^{t} -\frac{\gamma}{2}\sqnorm{\nabla f(x^{t})}.
\end{eqnarray*}
The last inequality is a result of the bound $\gamma^2\frac{\beta(\alpha,s) \LAMsq}{\theta(\alpha,s)} + L\gamma \leq 1,$ which is satisfied for the stepsize $$\gamma \leq \frac{1}{L + \LAM \xi(\alpha,s)},$$ 
where $ \xi(\alpha,s)\eqdef \sqrt{\frac{\beta(\alpha,s)}{\theta(\alpha,s)}}$.
Maximizing the stepsize bound over the choice of $s$ using  \Cref{lm:best_theta_beta_choice}, we obtain the final stepsize.
By summing up inequalities for $t =0, \ldots, T-1,$ we get
$$
0 \leq \delta^{T}+\frac{\gamma}{2 \theta} s^{T} \leq \delta^{0}+\frac{\gamma}{2 \theta} s^{0}-\frac{\gamma}       {2} \sum_{t=0}^{T-1} \ExpBr{\sqnorm{\nabla f(x^{t})}}.
$$
Multiplying both sides by $\frac{2}{\gamma T}$, after rearranging we get
$$
\sum_{t=0}^{T-1} \frac{1}{T} \ExpBr{\sqnorm{\nabla f (x^{t})}} \leq \frac{2 \delta^{0}}{\gamma T} + \frac{s^0}      {\theta T}.
$$
It remains to notice that the left hand side can be interpreted as $\ExpBr{ \sqnorm{\nabla f(\hat{x}^{T})} }$, where $\hat{x}^{T}$ is chosen from $\{ x^{0}, x^{1}, \ldots, x^{T-1} \}$ uniformly at random.
\end{proof}

\subsection{Main result for Polyak-\L ojasiewicz functions}

The main result is presented next.

\begin{theorem}
Let Assumptions~\ref{as:smooth},~\ref{as:L_i}, and~\ref{as:PL} hold. Assume that $\cC_i^t \in \mathbb{C}(\alpha)$ for all $i\in [n]$ and $t\geq 0$. Let the stepsize in \Cref{alg:EF21-W} be set as 
$$
0 < \gamma \leq \min\left\{ \frac{1}{L + \sqrt{2} \LAM  \xi(\alpha)}, \frac{\theta(\alpha)}{2\mu}\right\}.
$$
Let $$\Psi^t \eqdef f(x^t) - f(x^\ast) + \frac{\gamma}{\theta} G^t.$$ Then, for any $T>0$ the following inequality holds:
\begin{equation}
\ExpBr{\Psi^T} \leq (1 - \gamma \mu)^T \Psi^0.
\end{equation}
\end{theorem}

	\begin{proof}
	We proceed as in the previous proof, starting from the descent lemma with the same vector but using the PL inequality and            subtracting $f(x^\star)$ from both sides:
	\begin{eqnarray}
		\ExpBr{f(x^{t+1})- f(x^\star)}  &\overset{\eqref{eq:descent_lemma}}{\leq} & \ExpBr{ f(x^{t})- f(x^\star) }-\frac{\gamma}{2}\sqnorm{\nabla f(x^{t})}-\left(\frac{1}{2 \gamma}-\frac{L}{2}\right)\sqnorm{x^{t+1}-x^{t}}+\frac{\gamma}{2} G^t \notag \\ 
		& \overset{\eqref{eq:PL}}{\leq}  & (1-\gamma \mu) \ExpBr{ f(x^{t})- f(x^\star) } - \left(\frac{1}{2 \gamma}-\frac{L}{2}\right)\sqnorm{x^{t+1}-x^{t}}+\frac{\gamma}{2} G^t. \label{eq:weighted_aux_4}
	\end{eqnarray}
	
	Let $\delta^{t} \eqdef \ExpBr{f(x^{t})-f(x^\star)}$, $s^{t} \eqdef \ExpBr{G^t }$ and $ r^{t} \eqdef \ExpBr{\sqnorm{x^{t+1}-x^{t}}}$. Thus, 
	\begin{eqnarray*}
		\delta^{t+1}+\frac{\gamma}{ \theta} s^{t+1} &\overset{\eqref{eq:weighted_aux_4}}{\leq}& (1-\gamma \mu)\delta^{t} -\left(\frac{1}{2 \gamma} - \frac{L}{2}\right) r^{t} + \frac{\gamma}{2} s^{t} + \frac{\gamma}{ \theta} s^{t+1} \\
		& \overset{\eqref{eq:weighted_full_grad_dist_evolution}}{\leq}  & (1-\gamma \mu)\delta^{t} -\left(\frac{1}{2 \gamma} - \frac{L}{2}\right) r^{t} + \frac{\gamma}{2} s^{t} + \frac{\gamma}{ \theta}\left( (1-\theta) s^t + \beta \left(\avein L_i \right)^2 r^t       \right) \\
		&=&  (1-\gamma \mu) \delta^{t}  + \frac{\gamma}{\theta} \left(1-\frac{\theta}{2}\right) s^t  -\left(\frac{1}{2 \gamma} - \frac{L}{2} - \frac{\beta \LAMsq \gamma}{\theta}\right) r^{t},
	\end{eqnarray*}
	where $\theta$ and $\beta$ are set as in \Cref{lm:best_theta_beta_choice}.
	Note that our extra assumption on the stepsize implies that 	$ 1 - \frac{\theta}{2} \leq 1 -\gamma \mu$ and $$\frac{1}{2 \gamma} - \frac{L}{2} - \frac{\beta \LAMsq \gamma}{\theta} \geq 0.$$ The last inequality follows from the bound $\gamma^2\frac{2\beta \LAMsq}{\theta} + \gamma L \leq 1$. Thus,
	\begin{eqnarray*}
		\delta^{t+1}+\frac{\gamma}{ \theta} s^{t+1} \leq (1-\gamma \mu) \left(\delta^{t}+\frac{\gamma}{ \theta} s^{t} \right).
	\end{eqnarray*}
	It remains to unroll the recurrence.
\end{proof}

\clearpage
\section{Proof of  \Cref{thm:ef21_new_result} (Improved Theory for EF21)}
We commence by redefining gradient distortion as follows:
\begin{equation}\label{eq:new_def_distortion}
	G^{t} \eqdef \frac{1}{n^2} \sumin \frac{1}{w_i} \|\nabla f_i(x^t) - g^t_i\|^2.
\end{equation}

We recall that the gradient update step for standard \algname{EF21} (\Cref{alg:EF21}) takes the following form:
\begin{align}
&g_i^{t+1} = g_i^t + \cC_i^t(\nabla f_i(x^{t+1}) - g_i^t), \label{eq:standard_ind_grad_upd} \\
&g^{t+1} = \avein g_i^{t+1} \label{eq:standard_average_grad}.
\end{align}

\subsection{Two lemmas}

Once more, we start our proof with the contraction lemma.
\begin{lemma}
	Let $\cC_i^t \in \mathbb{C}(\alpha)$ for all $i \in [n]$ and $t\geq 0$. Define $W^t \eqdef \{g_1^t, g_2^t, \dots, g_n^t, x^t, x^{t+1}\}$. Let \Cref{as:L_i} hold. Then
	\begin{equation}\label{eq:grad_distortion_evolution}
		\ExpBr{G^{t+1}\; | \;  W^t} \leq (1 - \theta(\alpha,s)) G^t + \beta(\alpha,s) \LAMsq \|x^{t+1} - x^t \|^2,
	\end{equation}
	where $\theta(\alpha,s) \eqdef 1 -  (1 - \alpha)(1+s)$ and $\beta(\alpha,s) \eqdef (1 - \alpha) (1 + s^{-1})$ for any $s > 0$.
\end{lemma}

\begin{proof}
	The proof of this lemma starts as the similar lemma in the standard analysis of \algname{EF21}:
	\begin{eqnarray}
			\ExpBr{G^{t+1}\; | \;  W^t} & \overset{\eqref{eq:new_def_distortion}}{=}  & \frac{1}{n^2}\sumin \frac{1}{w_i}\ExpBr{ \|\nabla f_i(x^{t+1}) - g^{t+1}_i\|^2\; | \;  W^t}  \notag \\
			& \overset{\eqref{eq:standard_ind_grad_upd}}{=} & \frac{1}{n^2} \sumin \frac{1}{w_i} \ExpBr{\|g_i^{t} + \cC_i^t(\nabla f_i(x^{t+1}) - g_i^t) -\nabla f_i(x^{t+1}) \|^2 \;|\; W^t}  \notag \\
			& \overset{\eqref{eq:compressor_contraction}}{\leq} & \frac{1}{n^2} \sumin \frac{1 - \alpha}{w_i}  \|\nabla f_i(x^{t+1}) - g_i^t) \|^2  \notag \\
			& = & \frac{1}{n^2} \sumin \frac{1 - \alpha}{w_i}  \|\nabla f_i(x^{t+1}) - \nabla f_i(x^t) + \nabla f_i(x^t) - g_i^t) \|^2 \notag \\
			& \overset{\eqref{eq:young}}{\leq} & \frac{1}{n^2} \sumin \frac{1 - \alpha}{w_i}  \left((1 + s^{-1}) \|\nabla f_i(x^{t+1}) - \nabla f_i(x^t)) \|^2 + (1 + s) \|g_i^t - \nabla f_i(x^t) \|^2 \right) \notag \\
&& 			 \label{eq:aux1}
	\end{eqnarray}
for all  $s > 0$.		
We proceed the proof as follows:
	\begin{eqnarray}
			\ExpBr{G^{t+1} \;|\; W^t}  &\overset{\eqref{eq:aux1}}{\leq} & \frac{1}{n^2} \sumin \frac{1 - \alpha}{w_i}  \left((1 + s^{-1}) \|\nabla f_i(x^{t+1}) - \nabla f_i(x^t)) \|^2 + (1 + s) \|g_i^t - \nabla f_i(x^t) \|^2 \right) \notag \\
			& =  & (1 - \theta(\alpha,s)) \frac{1}{n^2} \sumin \frac{1}{w_i} \|g_i^t - \nabla f_i(x^t)\|^2  + \frac{\beta(\alpha,s)}{n^2} \sumin \frac{1}{w_i} \|\nabla f_i(x^{t+1}) - \nabla f_i(x^t)) \|^2  \notag \\
			& \overset{\eqref{eq:new_def_distortion}}{=} & (1 - \theta(\alpha,s)) G^t + \frac{\beta(\alpha,s)}{n^2} \sumin \frac{1}{w_i } \|\nabla f_i(x^{t+1}) - \nabla f_i(x^t) \|^2  \notag \\
			& \overset{\eqref{eq:L_i}}{\leq} & (1 - \theta(\alpha,s)) G^t + \frac{\beta(\alpha,s)}{n^2} \sumin \frac{L_i^2}{w_i} \|x^{t+1} - x^t \|^2.\label{eq:aux2}
	\end{eqnarray}
	Note that this is the exact place where the current analysis differs from the standard one. It fully coincides with it when $w_i = \frac{1}{n}$, i.e., when we assign the same weight for each individual gradient distortion $\|g_i^t - \nabla f_i(x^t) \|^2$. However, applying weights according to ``importance'' of each function, we proceed as follows:
	\begin{eqnarray*}
		\ExpBr{G^{t+1}\; | \; W^t} &\overset{\eqref{eq:aux2}}{\leq}  & (1 - \theta(\alpha,s)) G^t + \frac{\beta(\alpha,s)}{n^2} \sumin \frac{L_i^2}{w_i} \|x^{t+1} - x^t \|^2 \\
		& \overset{\eqref{eq:weight_definition}}{=}  & (1 - \theta(\alpha,s)) G^t + \frac{\beta(\alpha,s)}{n^2} \sumin \frac{L_i^2}{L_i} \left(\sum\limits_{i=1}^n L_i\right) \|x^{t+1} - x^t \|^2 \\
		& =  & (1 - \theta(\alpha,s)) G^t + \frac{\beta(\alpha,s)}{n^2} \sum_j L_j \left(\sum\limits_{i=1}^n L_i\right) \|x^{t+1} - x^t \|^2 \\
		& = & (1 - \theta(\alpha,s)) G^t + \beta(\alpha,s) \LAMsq \|x^{t+1} - x^t \|^2,
	\end{eqnarray*}
	what finishes the proof.
\end{proof}
To prove the main convergence theorem, we also need the following lemma.
\begin{lemma}
	For the variable $g^t$ from \Cref{alg:EF21}, the following inequality holds:
	\begin{equation}\label{eq:distortion_connection}
		\|g^t - \nabla f(x^t) \|^2 \leq G^t.
	\end{equation}
\end{lemma}
\begin{proof}
	The proof is straightforward:
	\begin{eqnarray*}
		\|g^t - \nabla f(x^t) \|^2 &\overset{\eqref{eq:standard_average_grad}}{=} & \left\| \sumin  \frac{1}{n}\left( g_i^t - \nabla f_i(x^t) \right) \right\|^2 	\\	&=& \left\| \sumin w_i \frac{1}{w_i n} \left( g_i^t - \nabla f_i(x^t) \right) \right\|^2 \\ 
		&\leq & \sumin w_i \left\| \frac{1}{w_i n} \left( g_i^t - \nabla f_i(x^t) \right) \right\|^2  \\
		&=& \sumin \frac{1}{w_i n^2} \|g^t - \nabla f_i(x^t)\|^2 \quad \overset{\eqref{eq:new_def_distortion}}{=} \quad G^t,
	\end{eqnarray*}
	where the only inequality in this series of equations is derived using Jensen's inequality.
\end{proof}

\subsection{Main result}

We are now equipped with all the necessary tools to establish the convergence theorem.
\begin{proof}
	Let us define the Lyapunov function $$\Phi^t \eqdef f(x^t) - f^\ast + \frac{\gamma}{2 \theta(\alpha,s)} G^t .$$ Let us also define $W^t \eqdef \{g_1^t, g_2^t, \dots, g_n^t, x^t, x^{t+1}\}$. We start as follows:
	\begin{eqnarray*}
		& \ExpBr{\Phi^{t+1}\; | \;  W^t}  \\
		& =  & \ExpBr{f(x^{t+1}) - f^\ast\; | \;  W^t}  + \frac{\gamma}{2\theta(\alpha,s)} \ExpBr{G^{t+1}\; | \;  W^t}\\
		& \overset{\eqref{eq:descent_lemma}}{\leq}  & f(x^t) - f^\ast - \frac{\gamma}{2} \|\nabla f(x^t)\|^2 - \left(\frac{1}{2\gamma} - \frac{L}{2} \right) \|x^{t+1} - x^t \|^2 + \frac{\gamma}{2}\|g^t - \nabla f(x^t) \|^2 \\
		&& \qquad + \frac{\gamma}{2\theta(\alpha,s)} \ExpBr{G^{t+1}\; | \;  W^t}\\
		& \overset{\eqref{eq:distortion_connection}}{\leq}  &  f(x^t) - f^\ast - \frac{\gamma}{2} \|\nabla f(x^t)\|^2 - \left(\frac{1}{2\gamma} - \frac{L}{2} \right) \|x^{t+1} - x^t \|^2 + \frac{\gamma}{2} G^t \\
		&& \qquad + \frac{\gamma}{2\theta(\alpha,s)} \ExpBr{G^{t+1}\; | \;  W^t}\\
		&\overset{\eqref{eq:grad_distortion_evolution}}{\leq}  &  f(x^t) - f^\ast - \frac{\gamma}{2} \|\nabla f(x^t)\|^2 - \left(\frac{1}{2\gamma} - \frac{L}{2} \right) \|x^{t+1} - x^t \|^2 + \frac{\gamma}{2} G^t\\
		& &  \qquad + \frac{\gamma}{2\theta(\alpha,s)} \left( (1 - \theta(\alpha,s)) G^t + \beta(\alpha,s) \LAMsq \|x^{t+1} - x^t \|^2 \right) \\
		& =   & f(x^t) - f^\ast + \frac{\gamma}{2\theta(\alpha,s)} G^t - \frac{\gamma}{2} \|\nabla f(x^t)\|^2 - \underbrace{\left(\frac{1}{2\gamma} - \frac{L}{2} - \frac{\gamma \beta(\alpha,s)}{2\theta(\alpha,s)} \LAMsq \right)}_{\geq 0} \|x^{t+1} - x^t \|^2 \\
		& \leq  & f(x^t) - f^\ast + \frac{\gamma}{2\theta(\alpha,s)} G^t - \frac{\gamma}{2} \|\nabla f(x^t)\|^2 \\
		& = & \Phi^t  - \frac{\gamma}{2} \|\nabla f(x^t)\|^2.
	\end{eqnarray*}
	The inequality in the last but one line is valid if $$\gamma \leq \frac{1}{ L + \LAM \sqrt{\frac{\beta(\alpha,s)}{\theta(\alpha,s)}}},$$ according to \Cref{lm:stepsize_bound}. By optimizing the stepsize bound through the selection of $s$ in accordance with \Cref{lm:best_theta_beta_choice}, we derive the final stepsize and establish the optimal value for $\theta$ in defining the Lyapunov function. Applying the tower property and unrolling the recurrence, we finish the proof.
\end{proof}

\subsection{Main result for Polyak-\L ojasiewicz functions}

For completeness, we also provide a convergence result under Polyak-\L ojasiewicz condition (Assumption~\ref{as:PL}). The main result is presented next.

\begin{theorem}
	Let Assumptions~\ref{as:smooth},~\ref{as:L_i}, and~\ref{as:PL} hold. Assume that $\cC_i^t \in \mathbb{C}(\alpha)$ for all $i\in [n]$ and $t\geq 0$. Let the stepsize in \Cref{alg:EF21-W} be set as 
	$$
	0 < \gamma \leq \min\left\{ \frac{1}{L + \sqrt{2} \LAM  \xi(\alpha)}, \frac{\theta(\alpha,s)}{2\mu}\right\}.
	$$
	Let $$\Psi^t \eqdef f(x^t) - f(x^\ast) + \frac{\gamma}{\theta(\alpha,s)} G^t .$$ Then, for any $T>0$ the following inequality holds:
	\begin{equation}
		\ExpBr{\Psi^T} \leq (1 - \gamma \mu)^T \Psi^0.
	\end{equation}
\end{theorem}

\begin{proof}
	We proceed as in the previous proof, starting from the descent lemma with the same vector but using the PL inequality and            subtracting $f(x^\star)$ from both sides:
	\begin{eqnarray}
			\ExpBr{f(x^{t+1})- f(x^\star)}  &\overset{\eqref{eq:descent_lemma}}{\leq} &  \ExpBr{ f(x^{t})- f(x^\star) }-\frac{\gamma}{2}\sqnorm{\nabla f(x^{t})}-\left(\frac{1}{2 \gamma}-\frac{L}{2}\right)\sqnorm{x^{t+1}-x^{t}}+\frac{\gamma}{2} G^t \notag \\ 
			& \overset{\eqref{eq:PL}}{\leq} & (1-\gamma \mu) \ExpBr{ f(x^{t})- f(x^\star) } - \left(\frac{1}{2 \gamma}-\frac{L}{2}\right)\sqnorm{x^{t+1}-x^{t}}+\frac{\gamma}{2} G^t. \label{eq:weighted_aux_4}
	\end{eqnarray}
	
	Let $\delta^{t} \eqdef \ExpBr{f(x^{t})-f(x^\star)}$, $s^{t} \eqdef \ExpBr{G^t }$ and $ r^{t} \eqdef \ExpBr{\sqnorm{x^{t+1}-x^{t}}}$. Thus, 
	\begin{eqnarray*}
		\delta^{t+1}+\frac{\gamma}{ \theta(\alpha,s)} s^{t+1} &\overset{\eqref{eq:weighted_aux_4}}{\leq}& (1-\gamma \mu)\delta^{t} -\left(\frac{1}{2 \gamma} - \frac{L}{2}\right) r^{t} + \frac{\gamma}{2} s^{t} + \frac{\gamma}{ \theta(\alpha,s)} s^{t+1} \\
		& \overset{\eqref{eq:grad_distortion_evolution}}{\leq}  & (1-\gamma \mu)\delta^{t} -\left(\frac{1}{2 \gamma} - \frac{L}{2}\right) r^{t} + \frac{\gamma}{2} s^{t}  \\
		&& \qquad + \frac{\gamma}{ \theta (\alpha,s)}\left( (1-\theta(\alpha,s)) s^t + \beta \left(\avein L_i \right)^2 r^t       \right) \\
		&=&  (1-\gamma \mu) \delta^{t} + \frac{\gamma}{\theta(\alpha,s)} \left(1-\frac{\theta(\alpha,s)}{2}\right) s^t  -\left(\frac{1}{2 \gamma} - \frac{L}{2} - \frac{\beta \LAMsq \gamma}{\theta (\alpha,s)}\right) r^{t}.
	\end{eqnarray*}

	Note that our extra assumption on the stepsize implies that 	$ 1 - \frac{\theta(\alpha,s)}{2} \leq 1 -\gamma \mu$ and $$\frac{1}{2 \gamma} - \frac{L}{2} - \frac{\beta \LAMsq \gamma}{\theta(\alpha,s)} \geq 0.$$ The last inequality follows from the bound $\gamma^2\frac{2\beta \LAMsq}{\theta(\alpha,s)} + \gamma L \leq 1$. Thus,
	\begin{eqnarray*}
		\delta^{t+1}+\frac{\gamma}{ \theta (\alpha,s)} s^{t+1} \leq (1-\gamma \mu) \left(\delta^{t}+\frac{\gamma}{ \theta (\alpha,s)} s^{t} \right).
	\end{eqnarray*}
	It remains to unroll the recurrence which finishes the prove.
\end{proof}


\clearpage
\section{EF21-W-SGD: Weighted Error Feedback 2021 with Stochastic Subsampled Gradients} \label{sec:EF21-W-SGD}

The \algname{EF21-W} algorithm assumes that all clients can compute the exact gradient in each round. In some scenarios, the exact gradients may be unavailable or too costly to compute, and only approximate gradient estimators can be obtained. In this section, we present the convergence result for \algname{EF21-W} in the setting where the gradient computation on the clients, $\nabla f_i(x^{t+1})$, is replaced by a specific stochastic gradient estimator. For a variation of  \algname{EF21-W-SGD} which is working under a more general setting please see Appendix~\ref{sec:EF21-W-SGD-ABC}.


\subsection{Algorithm}




In this section, we extend \algname{EF21-W} to handle stochastic gradients, and we call the resulting algorithm \algname{EF21-W-SGD} (Algorithm~\ref{alg:weighted_ef21_sgd}). Our analysis of this extension follows a similar approach as the one used by \citet{EF21BW} for studying the stochastic gradient version of the vanilla \algname{EF21} algorithm, which they called \algname{EF21-SGD}. Analysis of \algname{EF21-W-SGD} has two important differences with vanilla \algname{EF21-SGD}:

\begin{enumerate}
	\item Vanilla \algname{EF21-SGD} provides maximum theoretically possible $\gamma = \left( L + \LQM \sqrt{\frac{\beta_1}{\theta}} \right)^{-1}$, where \algname{EF21-W-SGD} has $\gamma = \left( L + \LAM \sqrt{\frac{\beta_1}{\theta}} \right)^{-1}$
	\item Vanilla \algname{EF21-SGD} and \algname{EF21-W-SGD} formally differs in a way how it reports iterate $x^T$ which minimizes $\ExpBr{\sqnorm{\nabla f(x^{T})}}$ due to a slightly different definition of $\widetilde{A}$. The \algname{EF21-W-SGD} (Algorithm \ref{alg:weighted_ef21_sgd}) requires output iterate $\hat{x}^T$ randomly according to the probability mass function described by \eqref{eq:09u09fd-0ff}.
\end{enumerate}


\begin{algorithm}
	\begin{algorithmic}[1]
		\STATE {\bfseries Input:} initial model $x^0 \in \RR^d$; initial gradient estimates $g_1^0, g_2^0, \dots,g_n^0 \in \R^d$ stored at the server and the clients; stepsize $\gamma>0$; number of iterations $T > 0$; weights ${\color{ForestGreen}w_i}= \frac{L_i}{\sum_j L_j}$ for $i\in [n]$
		\STATE {\bfseries Initialize:} $g^0 = \sumin {\color{ForestGreen}w_i} g_i^0$ on the server
		\FOR{$t = 0, 1, 2, \dots, T - 1 $}
		\STATE Server computes $x^{t+1} = x^t - \gamma g^t$ and  broadcasts  $x^{t+1}$ to all $n$ clients		
		\FOR{$i = 1, \dots, n$ {\bf on the clients in parallel}}
		\STATE Compute a stochastic estimator  $\hat{g_i} (x^{t+1})  = \frac{1}{\tau_i} \sum_{j=1}^{\tau_i}\nabla f_{\xi_{ij}^t}(x^{t+1})$ of the gradient $\nabla f_i(x^{t+1})$
		\STATE  Compute $u_i^t=\cC_i^t\left(\frac{1}{n {\color{ForestGreen}w_i}} \hat{g_i} (x^{t+1}) - g_i^t\right)$ and update $g_i^{t+1} = g_i^t +u_i^t$ \label{line:weighted_ef21_sgd_grad_update}
		\STATE Send the compressed message $u_i^{t}$ to the server		
		\ENDFOR
		\STATE Server updates $g_i^{t+1} = g_i^t +u_i^t$ for all $i\in [n]$, and computes $g^{t+1} = \sum_{i=1}^n {\color{ForestGreen}w_i} g_i^{t+1}$		
		\ENDFOR
		\STATE {\bfseries Output:} Point $\hat{x}^T$ chosen from the set $\{x^0, \dots, x^{T-1}\}$ randomly according to the law \eqref{eq:09u09fd-0ff}
	\end{algorithmic}
	\caption{\algname{EF21-W-SGD}: Weighted Error Feedback 2021 with Stochastic Gradients}
	\label{alg:weighted_ef21_sgd}
\end{algorithm}






\begin{assumption}[General assumption for stochastic gradient estimators]\label{as:general_as_for_stoch_gradients}
	We assume that for all $i \in [n]$ there exist parameters $A_i, C_i \ge 0$, $B_i \ge 1$ such that
	\begin{equation}
		\ExpBr{\|\nabla f_{\xi_{i j}^{t}}(x)\|^2} \le 2A_i\left(f_i(x) - f_i^{\inf}\right) + B_i\|\nabla f_i(x)\|^2 + C_i, \label{eq:general_second_mom_upp_bound}
	\end{equation}
	holds for all $x\in \RR^d$, 
	where\footnote{When $A_i = 0$ one can ignore the first term in the right-hand side of \eqref{eq:general_second_mom_upp_bound}, i.e., assumption $\inf_{x\in\R^d}f_i(x) > -\infty$ is not required in this case.} $f_i^{\inf} = \inf_{x\in\R^d}f_i(x) > -\infty$.
\end{assumption}


We study \algname{EF21-W-SGD} under the same assumption as was used for analyzing \algname{Vanilla EF21-SGD}, which we denote as Assumption~\ref{as:general_as_for_stoch_gradients}. To the best of our knowledge, this assumption, which was originally presented as Assumption 2 by \citet{khaled2020better}, is the most general assumption for a stochastic gradient estimator in a non-convex setting.

Next, to be aligned with original \algname{Vanilla EF21-SGD} \citep{EF21BW} we have considered a specific form of gradient estimator. This specific form of gradient estimator from \algname{Vanilla EF21-SGD} presented in Section 4.1.2. of \citet{EF21BW} where the stochastic gradient $\hat{g_i}$ has been computed as follows:
\begin{eqnarray*}
	\hat{g_i} (x^{t+1}) = \frac{1}{\tau_i} \sum_{j=1}^{\tau_i}\nabla f_{\xi_{ij}^t}(x^{t+1}),
\end{eqnarray*}

Here $\tau_i$ is a minibatch size of sampled datapoint indexed by $\xi_{ij}^t$ of client $i$ in iteration $t$. And  $\xi_{ij}^t$ are independent random variables. For a version of  \algname{EF21-W-SGD} which is working under a more general setting please see Appendix~\ref{sec:EF21-W-SGD-ABC}.


\subsection{A lemma}

The contraction lemma in this case gets the following form:
\begin{lemma} 
	Let $\cC_i^t \in \mathbb{C}(\alpha)$ for all $i\in [n]$ and $t\geq 0$. Define
	$$G_i^t \eqdef  \sqnorm{ g_i^t - \frac{\nabla f_i(x^{t})}{n w_i} } , \qquad G^t \eqdef \sumin w_i G_i^t.$$ Let Assumptions ~\ref{as:L_i} and  ~\ref{as:general_as_for_stoch_gradients} hold. Then, for any $s, \nu >0$ we have
	\begin{equation}
		\label{eq:weighted_ef21_sgd_grad_full_contraction}
\ExpBr{G^{t+1}} \leq (1-\hat{\theta}) \ExpBr{G^t} + \hat{\beta_1} \LAMsq  \ExpBr{\sqnorm{ x^{t+1} - x^t}}  + { \widetilde{A} {\hat{\beta_2}}} \ExpBr{f(x^{t+1}) - f^{\inf}}
+  {\widetilde{C} {\hat{\beta_2}}},
	\end{equation}
	where 
	\begin{eqnarray*}
		w_i &\eqdef & \frac{L_i}{\sum_j L_j}, \\
		 \hat{\theta} & \eqdef  & 1 - \rb{1-\alpha} (1+s) (1+\nu),\\
		  \hat{\beta_1} &\eqdef  & 2(1- \alpha ) \left(1+ s \right)\left(s+\nu^{-1}\right), \\
		   \hat{\beta_2} & \eqdef  & 2(1 - \alpha) (1 + s) (1+\nu^{-1}) + (1 + s^{-1}),\\
		\widetilde{A} &\eqdef  & \max_{i=1,\ldots,n} \left( \frac{2(A_i+L_i(B_i-1))}{\tau_i}  \frac{1}{n w_i} \right), \\
		\widetilde{C}  & \eqdef  &  \max_{i=1,\ldots,n} \left( \frac{C_i}{\tau_i}  \frac{1}{n w_i} \right).
	\end{eqnarray*}

\end{lemma}

\begin{proof}
	Define  $W^t \eqdef \{g_1^t, \dots,    g_n^t, x^t, x^{t+1}\}$. The proof starts as follows:
	\begin{eqnarray*}
		\ExpBr{ G_i^{t+1} \;|\; W^t} &\overset{\eqref{eq:weighted_def_grad_distortion}}{=}& \ExpBr{  \sqnorm{g_i^{t+1} 
				- \frac{\nabla f_i(x^{t+1})}{n w_i}}  \;|\; W^t}	 \\
		&\overset{\text{line}~\ref{line:weighted_ef21_sgd_grad_update}}{=}& \ExpBr{  \sqnorm{g_i^t + \cC_i^t \left( \frac{\hat{g_i}(x^{t+1})}{n w_i} - g_i^t \right) 
				- \frac{\nabla f_i(x^{t+1})}{n w_i}}  \;|\; W^t}	 \\
		&=& \ExpBr{  \sqnorm{\cC_i^t \left( \frac{\hat{g_i}(x^{t+1})}{n w_i} - g_i^t \right) - \left(\frac{\hat{g_i}(x^{t+1})}{n w_i} - g_i^t\right) + \frac{\hat{g_i}(x^{t+1})}{n w_i}
				- \frac{\nabla f_i(x^{t+1})}{n w_i}}  \;|\; W^t}	 \\
		&\overset{\eqref{eq:young}}{\leq}& (1+s) \ExpBr{ \sqnorm{\cC_i^t \left( \frac{\hat{g_i}(x^{t+1})}{n w_i} - g_i^t \right) - \left(\frac{\hat{g_i}(x^{t+1})}{n w_i} - g_i^t\right)}\;|\; W^t}  \\ 
		& & \qquad + (1+s^{-1})  \ExpBr{\sqnorm{\frac{\hat{g_i}(x^{t+1})}{n w_i}
				- \frac{\nabla f_i(x^{t+1})}{n w_i}}\;|\; W^t} \\
		&\overset{\eqref{eq:compressor_contraction}}{\leq}&  (1-\alpha) (1+s) \ExpBr{\sqnorm{\frac{\hat{g_i}(x^{t+1})}{n w_i} - \frac{\nabla f_i(x^t)}{n w_i} +\frac{\nabla f_i(x^t)}{n w_i} - g_i^t}\;|\; W^t}  \\ 
		& & \qquad + (1+s^{-1})  \ExpBr{\sqnorm{\frac{\hat{g_i}(x^{t+1})}{n w_i}
				- \frac{\nabla f_i(x^{t+1})}{n w_i}}\;|\; W^t} \\
		&\overset{\eqref{eq:young}}{\leq}& (1 - \alpha) (1 + s) (1+\nu) \ExpBr{\sqnorm{{g_i}^{t} - \frac{\nabla f_i(x^t)}{n w_i}}} \\
		& & \qquad + (1 - \alpha) (1 + s) (1+\nu^{-1}) \ExpBr{\sqnorm{\frac{\nabla f_i(x^t)}{n w_i} - \frac{\hat{g_i}(x^{t+1})}{n w_i}}\;|\; W^t} \\
		& & \qquad + (1+s^{-1}) \ExpBr{\sqnorm{\frac{\hat{g_i}(x^{t+1})}{n w_i}
				- \frac{\nabla f_i(x^{t+1})}{n w_i}}\;|\; W^t} \\
		&\overset{\eqref{eq:young}}{\leq}& (1 - \alpha) (1 + s) (1+\nu) \ExpBr{\sqnorm{ {g_i}^{t} - \frac{\nabla f_i(x^t)}{n w_i}}\;|\; W^t} \\
		& & \qquad + {2(1 - \alpha) (1 + s) (1+\nu^{-1})} \ExpBr{\sqnorm{\frac{\nabla f_i(x^{t+1})}{n w_i} - \frac{\hat{g_i}(x^{t+1})}{n w_i}}\;|\; W^t} \\
		& & \qquad + 2(1 - \alpha) (1 + s) (1+\nu^{-1}) {\sqnorm{\frac{\nabla f_i(x^{t+1})}{n w_i} - \frac{\nabla f_i(x^{t})}{n w_i}}} \\
		& & \qquad + {(1+s^{-1})} \ExpBr{\sqnorm{\frac{\hat{g_i}(x^{t+1})}{n w_i}
				- \frac{\nabla f_i(x^{t+1})}{n w_i}}\;|\; W^t}.
	\end{eqnarray*}
	
To further bound the last term, which contains multiple ${(1+s^{-1})}$ factors, we leverage the property that $\hat g_i(x^{t+1})$ is a random variable serving as an unbiased estimator of $\nabla f_i(x^{t+1})$, taking the form $$\hat{g_i} (x^{t+1}) = \frac{1}{\tau_i} \sum_{j=1}^{\tau_i}\nabla f_{\xi_{ij}^t}(x^{t+1}),$$ where $\xi_{ij}^t$ are independent random variables. Next, we can continue as follows:
	\begin{eqnarray*}
		\ExpBr{ G_i^{t+1} \;|\; W^t} &\leq& (1-\hat{\theta}) \ExpBr{G_i^t \;|\; W^t} + \hat{\beta_1} \frac{1}{n^2 w_i^2} \sqnorm{\nabla f_i(x^{t+1}) - \nabla f_i(x^t)} \\
		& & \qquad + \frac{{\hat{\beta_2}}}{(n w_i)^2} \left( \ExpBr{\sqnorm{\frac{1}{\tau_i} \sum_{j=1}^{\tau_i}\nabla f_{\xi_{ij}^t}(x^{t+1}) - \frac{1}{\tau_i} \sum_{j=1}^{\tau_i} {\nabla f_i}(x^{t+1})}\;|\; W^t} \right) \\
		& = & (1-\hat{\theta}) \ExpBr{G_i^t \;|\; W^t} + \hat{\beta_1} \frac{1}{n^2 w_i^2} \sqnorm{\nabla f_i(x^{t+1}) - \nabla f_i(x^t)} \\
		& & \qquad + \frac{{\hat{\beta_2}}}{(n w_i)^2 \tau^2} \left( \ExpBr{\sqnorm{ \sum_{j=1}^{\tau_i} \left( \nabla f_{\xi_{ij}^t}(x^{t+1}) - {\nabla f_i}(x^{t+1})\right) }\;|\; W^t} \right)  \\
		& = & (1-\hat{\theta}) \ExpBr{G_i^t \;|\; W^t} + \hat{\beta_1} \frac{1}{n^2 w_i^2} \sqnorm{\nabla f_i(x^{t+1}) - \nabla f_i(x^t)} \\
		& & \qquad + \frac{{\hat{\beta_2}}}{(n w_i)^2 {\tau_i}^2} \sum_{j=1}^{\tau_i} \left( \ExpBr{\sqnorm{ \nabla f_{\xi_{ij}^t}(x^{t+1})}\;|\; W^t} - \sqnorm{\ExpBr{ \nabla f_{\xi_{ij}^t}(x^{t+1})\;|\; W^t}} \right) \\
		&\le& (1-\hat{\theta}) \ExpBr{G_i^t \;|\; W^t} + \hat{\beta_1} \frac{1}{n^2 w_i^2} \sqnorm{\nabla f_i(x^{t+1}) - \nabla f_i(x^t)} \\
		& &  + \frac{{\hat{\beta_2}}}{(n w_i)^2 {\tau_i}^2} \sum_{j=1}^{\tau_i} \left(
		2A_i\left(f_i(x^{t+1}) - f_i^{\inf}\right) + B_i\|\nabla f_i(x^{t+1})\|^2 + C_i - \sqnorm{\nabla f_i(x^{t+1}}) \right) \\
		&=& (1-\hat{\theta}) \ExpBr{G_i^t \;|\; W^t} + \hat{\beta_1} \frac{1}{n^2 w_i^2} \sqnorm{\nabla f_i(x^{t+1}) - \sqnorm{\nabla f_i(x^t)}} \\
		& & \qquad + \frac{ 2A_i {\hat{\beta_2}}}{(n w_i)^2 \tau_i} \left(f_i(x^{t+1}) - f_i^{\inf}\right)
		+  \frac{2(B_i-1){\hat{\beta_2}}}{(n w_i)^2 \tau_i} \left(\frac{1}{2} \|\nabla f_i(x^{t+1})\|^2 \right)
		+  \frac{C_i {\hat{\beta_2}}}{(n w_i)^2 \tau_i} \\
		&\le& (1-\hat{\theta}) \ExpBr{G_i^t \;|\; W^t} + \hat{\beta_1} \frac{1}{n^2 w_i^2} \sqnorm{\nabla f_i(x^{t+1}) - {\nabla f_i(x^t)}} \\
		& & \qquad + \frac{ 2A_i {\hat{\beta_2}}}{(n w_i)^2 \tau_i} \left(f_i(x^{t+1}) - f_i^{\inf}\right)
		+  \frac{2(B_i-1){\hat{\beta_2}}}{(n w_i)^2 \tau_i} L_i \left(f_i(x^{t+1}) - f_i^{\inf} \right)
		+  \frac{C_i {\hat{\beta_2}}}{(n w_i)^2 \tau_i} \\
		&=& (1-\hat{\theta}) \ExpBr{G_i^t \;|\; W^t} + \hat{\beta_1} \frac{1}{n^2 w_i^2} \sqnorm{\nabla f_i(x^{t+1}) - {\nabla f_i(x^t)}} \\
		& & \qquad + \frac{ 2(A_i + L_i(B_i - 1)) {\hat{\beta_2}}}{(n w_i)^2 \tau_i} \left(f_i(x^{t+1}) - f_i^{\inf}\right)
		+  \frac{C_i {\hat{\beta_2}}}{(n w_i)^2 \tau_i}.
	\end{eqnarray*}
	
Furthermore, as a result of leveraging \Cref{as:L_i}, we can derive the subsequent bound:
\begin{eqnarray*}
	\ExpBr{ G_i^{t+1} \;|\; W^t} & \leq & (1-\hat{\theta})   G_i^t +  \frac{\hat{\beta_1} L_i^2}{n^2 w_i^2} \sqnorm{ x^{t+1} - x^t}  \\	
&& \qquad 
	+ \frac{ 2(A_i + L_i(B_i - 1)) {\hat{\beta_2}}}{(n w_i)^2 \tau_i} \left(f_i(x^{t+1}) - f_i^{\inf}\right)
	+  \frac{C_i {\hat{\beta_2}}}{(n w_i)^2 \tau_i}.
\end{eqnarray*}
Applying the tower property and subsequently taking the expectation, we obtain:
\begin{equation}\label{eq:weighted_ef21_sgd_aux_1}
\begin{aligned}
	\ExpBr{ G_i^{t+1}} &\leq (1-\hat{\theta}) \ExpBr{G_i^t} + \hat{\beta_1} \frac{1}{n^2 w_i^2} L_i^2 \ExpBr{\sqnorm{ x^{t+1} - x^t}} 
	\\
	& \qquad + \frac{ 2(A_i + L_i(B_i - 1)) {\hat{\beta_2}}}{(n w_i)^2 \tau_i} \ExpBr{f_i(x^{t+1}) - f_i^{\inf}}
	+  \frac{C_i {\hat{\beta_2}}}{(n w_i)^2 \tau_i}.
\end{aligned}
\end{equation}
Regarding the expectation of $G^{t+1}$, we derive the subsequent bound:
\begin{eqnarray*}
	\ExpBr{G^{t+1}} &=& \ExpBr{\sumin w_i G_i^{t+1}} \\
	& =& \sumin w_i \ExpBr{G_i^{t+1}} \\
	&\overset{\eqref{eq:weighted_ef21_sgd_aux_1}}{\leq} & (1-\hat{\theta}) \sumin w_i \ExpBr{G_i^t} + \sumin w_i \hat{\beta_1} \frac{1}{n^2 w_i^2} L_i^2 \cdot \ExpBr{\sqnorm{ x^{t+1} - x^t}} \\
	&& \qquad + \sumin w_i \frac{ 2(A_i + L_i(B_i - 1)) {\hat{\beta_2}}}{(n w_i)^2 \tau_i} \cdot \ExpBr{f_i(x^{t+1}) - f_i^{\inf}}
	+  \sumin w_i \frac{C_i {\hat{\beta_2}}}{(n w_i)^2 \tau_i}	\\
	&= &(1-\hat{\theta}) \ExpBr{G^t} + \sumin \hat{\beta_1} \frac{1}{n^2 w_i} L_i^2 \cdot \ExpBr{\sqnorm{ x^{t+1} - x^t}} \\
	&& \qquad + \sumin \frac{ 2(A_i + L_i(B_i - 1)) {\hat{\beta_2}}}{n^2 w_i \tau_i} \cdot \ExpBr{f_i(x^{t+1}) - f_i^{\inf}}
	+  \sumin \frac{C_i {\hat{\beta_2}}}{n^2 w_i \tau_i}.
\end{eqnarray*}


Employing quantities $\tilde{A}$ and $\tilde{C}$, the final bound can be reformulated as follows:
\begin{eqnarray*}
	\ExpBr{G^{t+1}} & \leq & (1-\hat{\theta}) \ExpBr{G^t} + \sumin \hat{\beta_1} \frac{1}{n^2 w_i} L_i^2 \cdot \ExpBr{\sqnorm{ x^{t+1} - x^t}} \notag \\
&& \qquad 	+ \frac{1}{n} \sumin { \widetilde{A} {\hat{\beta_2}}} \cdot \ExpBr{f_i(x^{t+1}) - f_i^{\inf}} +   {\widetilde{C} {\hat{\beta_2}}} \notag \\
	& \leq & (1-\hat{\theta}) \ExpBr{G^t} + \sumin \hat{\beta_1} \frac{1}{n^2 w_i} L_i^2 \cdot \ExpBr{\sqnorm{ x^{t+1} - x^t}} \notag \\
&& \qquad 	+ \frac{1}{n} \sumin { \widetilde{A} {\hat{\beta_2}}} \cdot \ExpBr{f_i(x^{t+1}) - f^{\inf}} +   {\widetilde{C} {\hat{\beta_2}}} \notag \\
	& \leq & (1-\hat{\theta}) \ExpBr{G^t} + \sumin \hat{\beta_1} \frac{1}{n^2 w_i} L_i^2 \cdot \ExpBr{\sqnorm{ x^{t+1} - x^t}} \\
&& \qquad 	+ { \widetilde{A} {\hat{\beta_2}}} \notag \ExpBr{f(x^{t+1}) - f^{\inf}}
	+  {\widetilde{C} {\hat{\beta_2}}}.
\end{eqnarray*}


Given that $w_i=\frac{L_i}{\sum_j L_j}$, we have:
\begin{eqnarray*}
	\ExpBr{G^{t+1}} &\leq & (1-\hat{\theta}) \ExpBr{G^t} + \frac{1}{n}\sumin \hat{\beta_1} \frac{\sum_j L_j}{n} L_i \ExpBr{\sqnorm{ x^{t+1} - x^t}} \notag \\
	&&  \qquad + { \widetilde{A} {\hat{\beta_2}}} \ExpBr{f(x^{t+1}) - f^{\inf}}
	+  {\widetilde{C} {\hat{\beta_2}}} \notag \\
	& = & (1-\hat{\theta}) \ExpBr{G^t} + \hat{\beta_1} \left(\avein L_i\right)^2 \cdot \ExpBr{\sqnorm{ x^{t+1} - x^t}} \notag \\
	&& \qquad + { \widetilde{A} {\hat{\beta_2}}} \ExpBr{f(x^{t+1}) - f^{\inf}}
	+  {\widetilde{C} {\hat{\beta_2}}},
\end{eqnarray*}
what completes the proof.
\end{proof}

\subsection{Main result}

Now we are ready to prove the main convergence theorem.
\begin{theorem} Let $\cC_i^t \in\mathbb{C}(\alpha)$ for all $\in [n]$ and $t\geq 0$ in \Cref{alg:weighted_ef21_sgd}. Set the following quantities:
\begin{eqnarray*} 
		\hat{\theta} &\eqdef & 1 - \rb{1-\alpha} (1+s) (1+\nu),\\
		  \hat{\beta_1} &\eqdef  & 2(1- \alpha ) \left(1+ s \right)\left(s+\nu^{-1}\right),\\
		 \hat{\beta_2} &\eqdef  & 2(1 - \alpha) (1 + s) (1+\nu^{-1}) + (1 + s^{-1}),\\
		w_i &\eqdef &  \frac{L_i}{\sum_j L_j},\\
		\widetilde{A} & \eqdef  & \max_{i=1,\ldots,n} \left( \frac{2(A_i+L_i(B_i-1))}{\tau_i}  \frac{1}{n w_i} \right), \\
		\widetilde{C} & \eqdef  & \max_{i=1,\ldots,n} \left( \frac{C_i}{\tau_i}  \frac{1}{n w_i} \right).
\end{eqnarray*} 
Under Assumptions~\ref{as:smooth},~\ref{as:L_i}, and~\ref{as:general_as_for_stoch_gradients}, and selection of $s>0$, $\mu>0$ such that $(1+s)(1+\mu) < \frac{1}{1-\alpha}$ set the stepsize in the following way:
\begin{equation}
\gamma \leq \frac{1}{L + \LAM \sqrt{\frac{\hat{\beta}_1}{\hat{\theta}}}}.
\end{equation}
Choose an iterate $\hat{x}^T$ from $\{x^0, x^1, \dots, x^{T-1}\}$ with probability 
\begin{equation} \label{eq:09u09fd-0ff}
\Prob(\hat{x}^T = x^t) = \frac{v_t}{V_T},\end{equation} where
$$ v_t \eqdef \left(1 - \frac{\gamma \tilde{A} \tilde{\beta}_2}{2\theta}\right)^t; \qquad V_T \eqdef \sum\limits_{t=0}^{T-1} v_t.
$$
Then,
\begin{equation}
\ExpBr{\sqnorm{\nabla f(\hat{x}^{T})}} \leq \frac{2 (f(x^0) - f^\text{inf})}{\gamma T  \left(1 - \frac{\gamma \widetilde{A} \hat{\beta_2}}{2 \theta}\right)^T } + \frac{G^0}{ \hat{\theta}T  \left(1 - \frac{\gamma \widetilde{A} \hat{\beta_2}}{2 \theta}\right)^T} + \frac{\widetilde{C}\beta_2}{\hat{\theta}},
\end{equation}
where $G^0 \eqdef \sumin w_i \|g_i^0 - \frac{1}{nw_i}\nabla f_i(x^0)\|^2$.
\end{theorem}
\begin{proof}
	In the derivation below, we  use \Cref{lm:descent_lemma} for 
	\begin{equation}\label{eq:sgd_weighted_grad_estimate_def}
	g^t=\sum\limits_{i=1}^{n} w_i g_i^t.
	\end{equation}
	We start as follows:
\begin{eqnarray}
		f(x^{t+1}) &\overset{\eqref{eq:descent_lemma}}{\leq} &
		f(x^{t})-\frac{\gamma}{2}\sqnorm{\nabla f(x^{t})}-\left(\frac{1}{2 \gamma}-\frac{L}{2}\right)\sqnorm{x^{t+1}-x^{t}}+\frac{\gamma}{2}\sqnorm{
			g^t- \sumin \nabla f_i(x^{t}) } \notag \\
		& \overset{\eqref{eq:sgd_weighted_grad_estimate_def}}{=} &
		f(x^{t})-\frac{\gamma}{2}\sqnorm{\nabla f(x^{t})}-\left(\frac{1}{2 \gamma}-\frac{L} 
		{2}\right)\sqnorm{x^{t+1}-x^{t}}+\frac{\gamma}{2}\sqnorm{
			\sumin w_i \left(g_i^t-  \frac{\nabla f_i(x^{t})}{n w_i} \right) }  \notag  \\
		& \leq &
		f(x^{t})-\frac{\gamma}{2}\sqnorm{\nabla f(x^{t})}-\left(\frac{1}{2 \gamma}-\frac{L}       
		{2}\right)\sqnorm{x^{t+1}-x^{t}}+\frac{\gamma}{2}
		\sumin w_i \sqnorm{g_i^t-  \frac{\nabla f_i(x^{t})}{n w_i} }  \notag \\
		& = & f(x^{t})-\frac{\gamma}{2}\sqnorm{\nabla f(x^{t})}-\left(\frac{1}{2 \gamma}-\frac{L}       
		{2}\right)\sqnorm{x^{t+1}-x^{t}}+\frac{\gamma}{2} G^t.
\end{eqnarray}

Subtracting $f^\ast$ from both sides and taking expectation, we get
\begin{eqnarray}
	\ExpBr{f(x^{t+1})-f^\ast} & \leq &  \ExpBr{f(x^{t})-f^\ast}
	-\frac{\gamma}{2} \ExpBr{\sqnorm{\nabla f(x^{t})}} \notag \\
	&& \qquad  -\left(\frac{1}{2 \gamma}
	-\frac{L}{2}\right) \ExpBr{\sqnorm{x^{t+1}-x^{t}}}+ \frac{\gamma}{2}\ExpBr{G^t}.\label{eq:weighted_function_descent_sgd}
\end{eqnarray}

Let $\delta^{t} \eqdef \ExpBr{f(x^{t})-f^\ast}$, $s^{t} \eqdef \ExpBr{G^t }$ and 
$r^{t} \eqdef\ExpBr{\sqnorm{x^{t+1}-x^{t}}}.$
Then by adding $\frac{\gamma}{2\theta} s^{t+1}$ and employing \eqref{eq:weighted_ef21_sgd_grad_full_contraction}, we obtain:
\begin{align*}
	\delta^{t+1}+\frac{\gamma}{2 \hat{\theta}} s^{t+1} 
	&\leq \delta^{t}-\frac{\gamma}{2}\ExpBr{\sqnorm{\nabla f(x^{t})}} -          \left(\frac{1}{2 \gamma}-\frac{L}{2}\right) r^{t}+\frac{\gamma}{2} s^{t} \\
	& \qquad +\frac{\gamma}{2 \hat{\theta}}\left( \hat{\beta_1} \LAMsq  r^t + (1-\hat{\theta}) s^t + { \widetilde{A} {\hat{\beta_2}}} \delta^{t+1} +  {\widetilde{C} {\hat{\beta_2}}} \right) \\
	&=\delta^{t}+\frac{\gamma}{2\hat{\theta}} s^{t}-\frac{\gamma}{2}\ExpBr{\sqnorm{\nabla f(x^{t})}}  -\left(\frac{1}{2\gamma} -\frac{L}{2} -  \frac{\gamma}{2\hat{\theta}} \hat{\beta_1}  \LAMsq  \right) r^{t} + \frac{\gamma \widetilde{A} \beta_2}{2\hat{\theta}} \delta^{t+1} + \frac{\gamma \widetilde{C}}{2 \hat{\theta}} \beta_2\\
	& \leq \delta^{t}+\frac{\gamma}{2\hat{\theta}} s^{t} -\frac{\gamma}{2}\ExpBr{\sqnorm{\nabla f(x^{t})}} + \frac{\gamma \widetilde{A} \beta_2}{2\hat{\theta}} \delta^{t+1} + \frac{\gamma \widetilde{C}}{2 \hat{\theta}} \beta_2.
\end{align*}


The last inequality follows from the bound $\gamma^2\frac{\hat{\beta_1} \LAMsq }{\hat{\theta}} + L\gamma \leq 1$, which holds due to \Cref{lm:stepsize_bound} for $\gamma \leq \frac{1}{L +  \LAM \sqrt{\frac{\hat{\beta}_1}{\hat{\theta}}}}.$ Subsequently, we will reconfigure the final inequality and perform algebraic manipulations, taking into account that $\frac{2}{\gamma} > 0$. In the final step of these algebraic transformations, we will leverage the fact that $s^t \ge 0$:
\begin{eqnarray*}
	\delta^{t+1}+\frac{\gamma}{2 \hat{\theta}} s^{t+1} &\leq& \delta^{t}+\frac{\gamma}{2\hat{\theta}} s^{t} -\frac{\gamma}{2}\ExpBr{\sqnorm{\nabla f(x^{t})}} + \frac{\gamma \widetilde{A} \beta_2}{2\hat{\theta}} \delta^{t+1} + \frac{\gamma \widetilde{C}}{2 \hat{\theta}} \beta_2 .
\end{eqnarray*}	

Therefore,
\begin{eqnarray*}
	\frac{2}{\gamma} \delta^{t+1}+\frac{2}{\gamma} \frac{\gamma}{2 \hat{\theta}} s^{t+1} &\leq& \frac{2}{\gamma} \delta^{t} + \frac{2}{\gamma} \frac{\gamma}{2\hat{\theta}} s^{t} -\ExpBr{\sqnorm{\nabla f(x^{t})}} + \frac{2}{\gamma} \frac{\gamma \widetilde{A} \beta_2}{2\hat{\theta}} \delta^{t+1} + \frac{2}{\gamma} \frac{\gamma \widetilde{C}}{2 \hat{\theta}} \beta_2.
	\end{eqnarray*}	

Further,
\begin{eqnarray*}	
	\ExpBr{\sqnorm{\nabla f(x^{t})}} &\leq& -\frac{2}{\gamma} \delta^{t+1} - \frac{2}{\gamma} \frac{\gamma}{2 \hat{\theta}} s^{t+1} + \frac{2}{\gamma} \delta^{t} + \frac{2}{\gamma} \frac{\gamma}{2\hat{\theta}} s^{t} + \frac{2}{\gamma} \frac{\gamma \widetilde{A} \beta_2}{2\hat{\theta}} \delta^{t+1} + \frac{2}{\gamma} \frac{\gamma \widetilde{C}}{2 \hat{\theta}} \beta_2  \\
	 &\leq&
	-\frac{2}{\gamma} \delta^{t+1} - \frac{2}{\gamma} \frac{\gamma}{2\hat{\theta}} s^{t+1} + \frac{2}{\gamma} \left( \delta^{t} + \frac{\gamma}{2\hat{\theta}} s^{t} \right) + \frac{2}{\gamma} \frac{\gamma \widetilde{A} \beta_2}{2 \hat{\theta}} \delta^{t+1} + \frac{\widetilde{C}\beta_2}{ \hat{\theta}}  \\
 &\leq&
	\frac{2}{\gamma} \left( \left( \delta^{t} + \frac{\gamma}{2\hat{\theta}} s^{t} \right) -1\left(1 - \frac{\gamma \widetilde{A} \beta_2}{2 \hat{\theta}} \right) \delta^{t+1} -\left(\frac{\gamma}{2\hat{\theta}} s^{t+1} \right) \right) + \frac{\widetilde{C}\beta_2}{ \hat{\theta}} \\
	&\leq& \frac{2}{\gamma} \left( \left( \delta^{t} + \frac{\gamma}{2\hat{\theta}} s^{t} \right) -\left(1 - \frac{\gamma \widetilde{A} \beta_2}{2 \hat{\theta}} \right) \left(\delta^{t+1} + \frac{\gamma}{2\hat{\theta}} s^{t+1} \right) \right) + \frac{\widetilde{C}\beta_2}{ \hat{\theta}}.
\end{eqnarray*}

We sum up inequalities above with weights $v_t/V_T$, where $v_t \eqdef (1 - \frac{\gamma \widetilde{A} \hat{\beta_2}}{2 \theta})^t$ and $V_T \eqdef \sum_{i=1}^{T} v_i$:
\begin{eqnarray*}
	\ExpBr{\sqnorm{\nabla f(\hat{x}^{T})}} &=& \sum_{t=0}^{T} \frac{v_t}{V_T} \ExpBr{\sqnorm{\nabla f(x^{t})}} \\
	&=& \frac{1}{V_T} \sum_{t=0}^{T} v_t \ExpBr{\sqnorm{\nabla f(x^{t})}} \\
	&\leq& \frac{1}{V_T} \sum_{t=0}^{T} v_t \left(\frac{2}{\gamma} \left( \left( \delta^{t} + \frac{\gamma}{2\hat{\theta}} s^{t} \right) -\left(1 - \frac{\gamma \widetilde{A} \beta_2}{2 \hat{\theta}} \right) \left(\delta^{t+1} + \frac{\gamma}{2\hat{\theta}} s^{t+1} \right) \right) + \frac{\widetilde{C}\beta_2}{ \hat{\theta}} \right) \\
	&=& \frac{2}{\gamma V_T} \sum_{t=0}^{T} w_t \left( \left( \delta^{t} + \frac{\gamma}{2\hat{\theta}} s^{t} \right) -\left(1 - \frac{\gamma \widetilde{A} \beta_2}{2 \hat{\theta}} \right) \left(\delta^{t+1} + \frac{\gamma}{2\hat{\theta}} s^{t+1} \right) \right)	+ \sum_{t=0}^{T} \frac{w_t}{W_T} \cdot \frac{\widetilde{C}\beta_2}{ \hat{\theta}} \\
	&=& \frac{2}{\gamma V_T} \sum_{t=0}^{T} w_t \left( \left( \delta^{t} + \frac{\gamma}{2\hat{\theta}} s^{t} \right) -\left(1 - \frac{\gamma \widetilde{A} \beta_2}{2 \hat{\theta}} \right) \left(\delta^{t+1} + \frac{\gamma}{2\hat{\theta}} s^{t+1} \right) \right)	+ \frac{\widetilde{C}\beta_2}{ \hat{\theta}} \\
	&=& \frac{2}{\gamma V_T} \sum_{t=0}^{T} \left(w_t \left( \delta^{t} + \frac{\gamma}{2\hat{\theta}} s^{t} \right) -w_{t+1} \left(\delta^{t+1} + \frac{\gamma}{2\hat{\theta}} s^{t+1} \right) \right)	+ \frac{\widetilde{C}\beta_2}{ \hat{\theta}} \\
	&\leq& \frac{2 \delta^0}{\gamma V_T} + \frac{s^0}{ \hat{\theta}V_T} + \frac{\widetilde{C}\beta_2}{\hat{\theta}}.
\end{eqnarray*}
Finally, we notice that $V_T = \sum\limits_{t=1}^T (1 - \frac{\gamma \widetilde{A} \hat{\beta_2}}{2 \theta})^t \geq T \cdot (1 - \frac{\gamma \widetilde{A} \hat{\beta_2}}{2 \theta})^T$, what concludes the proof.
\end{proof}

\clearpage
\section{EF21-W-SGD: Weighted Error Feedback 2021 with Stochastic Gradients under the ABC Assumption} \label{sec:EF21-W-SGD-ABC}

In this section, we present the convergence result for \algname{Weighted EF21} in the setting where the gradient computation on the clients is replaced with a pretty general unbiased stochastic gradient estimator.

\subsection{Algorithm}


The \algname{EF21-W} algorithm assumes that all clients can compute the exact gradient in each round. In some scenarios, the exact gradients may be unavailable or too costly to compute, and only approximate gradient estimators can be obtained. To have the ability for \algname{EF21-W} to work in such circumstances we extended \algname{EF21-W} to handle stochastic gradients. We called the resulting algorithm \algname{EF21-W-SGD} (Algorithm~\ref{alg:weighted_ef21_sgd_abc}).

\begin{algorithm}
	\begin{algorithmic}[1]
		\STATE {\bfseries Input:} initial model $x^0 \in \RR^d$; initial gradient estimates $g_1^0, g_2^0, \dots,g_n^0 \in \R^d$ stored at the server and the clients; stepsize $\gamma>0$; number of iterations $T > 0$; weights ${\color{ForestGreen}w_i}= \frac{L_i}{\sum_j L_j}$ for $i\in [n]$
		\STATE {\bfseries Initialize:} $g^0 = \sumin {\color{ForestGreen}w_i} g_i^0$ on the server
		\FOR{$t = 0, 1, 2, \dots, T - 1 $}
		\STATE Server computes $x^{t+1} = x^t - \gamma g^t$ and  broadcasts  $x^{t+1}$ to all $n$ clients		
		\FOR{$i = 1, \dots, n$ {\bf on the clients in parallel}}
		\STATE Compute a stochastic gradient $\hat{g_i} (x^{t+1})$ estimator of the gradient $\nabla f_i(x^{t+1})$
		\STATE  Compute $u_i^t=\cC_i^t\left(\frac{1}{n {\color{ForestGreen}w_i}} \hat{g_i} (x^{t+1}) - g_i^t\right)$ and update $g_i^{t+1} = g_i^t +u_i^t$ \label{line:weighted_ef21_sgd_grad_update_abc}
		\STATE Send the compressed message $u_i^{t}$ to the server		
		\ENDFOR
		\STATE Server updates $g_i^{t+1} = g_i^t +u_i^t$ for all $i\in [n]$, and computes $g^{t+1} = \sum_{i=1}^n {\color{ForestGreen}w_i} g_i^{t+1}$		
		\ENDFOR
		\STATE {\bfseries Output:} Point $\hat{x}^T$ chosen from the set $\{x^0, \dots, x^{T-1}\}$ randomly according to the law \eqref{eq:09u09fd-0ff-abc}
	\end{algorithmic}
	\caption{\algname{EF21-W-SGD}: Weighted EF-21 with Stochastic Gradients under ABC assumption}
	\label{alg:weighted_ef21_sgd_abc}
\end{algorithm}

Our analysis of this extension follows a similar approach as the one used by \citet{EF21BW} for studying the stochastic gradient version under the name \algname{EF21-SGD}. However, \algname{EF21-W-SGD} has four important differences with vanilla \algname{EF21-SGD}:

\begin{enumerate}
	\item Vanilla \algname{EF21-SGD} algorithm analyzed by \citet{EF21BW} worked under a specific sampling schema for a stochastic gradient estimator. Our analysis works under a more general ABC Assumption~\ref{as:general_as_for_stoch_gradients_abc}.
	\item  Vanilla \algname{EF21-SGD} provides maximum theoretically possible $\gamma = \left( L + \LQM \sqrt{\frac{\beta_1}{\theta}} \right)^{-1}$, where \algname{EF21-W-SGD} has $\gamma = \left( L + \LAM \sqrt{\frac{\beta_1}{\theta}} \right)^{-1}$.
	\item In contrast to the original analysis \algname{Vanilla EF21-SGD} our analysis provides a more aggressive $\beta_1$ parameter which is smaller by a factor of $2$.	
	\item Vanilla \algname{EF21-SGD} and \algname{EF21-W-SGD} formally differs in a way how it reports iterate $x^T$ which minimizes $\ExpBr{\sqnorm{\nabla f(x^{T})}}$ due to a slightly different definition of $\widetilde{A}$. The \algname{EF21-W-SGD} (Algorithm \ref{alg:weighted_ef21_sgd_abc}) requires output iterate $\hat{x}^T$ randomly according to the probability mass function described by Equation \eqref{eq:09u09fd-0ff-abc}.
\end{enumerate}

\begin{assumption}[General assumption for stochastic gradient estimators]
	\label{as:general_as_for_stoch_gradients_abc}
	We assume that for all $i \in [n]$ there exist parameters $A_i, C_i \ge 0$, $B_i \ge 1$ such that
	\begin{equation}
		\ExpBr{\|\nabla \hat{g_i} (x)\|^2} \le 2A_i\left(f_i(x) - f_i^{\inf}\right) + B_i\|\nabla f_i(x)\|^2 + C_i, \label{eq:general_second_mom_upp_bound_abc}
	\end{equation}
	holds for all $x\in \RR^d$, 
	where\footnote{When $A_i = 0$ one can ignore the first term in the right-hand side of \eqref{eq:general_second_mom_upp_bound_abc}, i.e., assumption $\inf_{x\in\R^d}f_i(x) > -\infty$ is not required in this case.} $f_i^{\inf} = \inf_{x\in\R^d}f_i(x) > -\infty$.
\end{assumption}

\begin{assumption}[Unbiased assumption for stochastic gradient estimators]\label{as:general_as_unbiased_for_stoch_gradients_abc}
	We assume that for all $i \in [n]$ there following holds for all $x\in \RR^d$:
	\begin{equation*}
		\ExpBr{\hat{g_i} (x)} = \nabla f_{{i}}(x) .
	\end{equation*}
\end{assumption}

We study \algname{EF21-W-SGD} under Assumption~\ref{as:general_as_for_stoch_gradients_abc} and Assumption~\ref{as:general_as_unbiased_for_stoch_gradients_abc}.To the best of our knowledge, this Assumption~\ref{as:general_as_for_stoch_gradients_abc}, which was originally presented as Assumption 2 by \citet{khaled2020better}, is the most general assumption for a stochastic gradient estimator in a non-convex setting. For a detailed explanation of the generality of this assumption see Figure 1 of \citet{khaled2020better}.
 
\subsection{A lemma}

The contraction lemma in this case gets the following form:
\begin{lemma} 
	Let $\cC_i^t \in \mathbb{C}(\alpha)$ for all $i\in [n]$ and $t\geq 0$. Define
	$$G_i^t \eqdef  \sqnorm{ g_i^t - \frac{\nabla f_i(x^{t})}{n w_i} } , \qquad G^t \eqdef \sumin w_i G_i^t.$$ Let Assumptions ~\ref{as:L_i}, ~\ref{as:general_as_for_stoch_gradients_abc}, ~\ref{as:general_as_unbiased_for_stoch_gradients_abc} hold. Then, for any $s >0, \nu >0$ during execution of the Algorithm~\ref{alg:weighted_ef21_sgd_abc} the following holds:
	\begin{equation}\label{eq:weighted_ef21_sgd_grad_full_contraction_abc}
		\ExpBr{G^{t+1}} \leq (1-\hat{\theta}) \ExpBr{G^t} + \hat{\beta_1} \LAMsq  \ExpBr{\sqnorm{ x^{t+1} - x^t}}  + { \widetilde{A} {\hat{\beta_2}}} \ExpBr{f(x^{t+1}) - f^{\inf}}
		+  {\widetilde{C} {\hat{\beta_2}}},
	\end{equation}
	where
	\begin{eqnarray*}
		w_i &\eqdef & \frac{L_i}{\sum_j L_j}, \\
		\hat{\theta} & \eqdef & 1 - \rb{1-\alpha} (1+s) (1+\nu) \\
		\hat{\beta_1} &\eqdef  & (1- \alpha ) \left(1+ s \right)\left(s+\nu^{-1}\right), \\
		 \hat{\beta_2}  & \eqdef  & (1 - \alpha) (1 + s) + (1 + s^{-1}),\\
		\widetilde{A} & \eqdef &  \max_{i=1,\ldots,n} \left( {2(A_i+L_i(B_i-1))}  \frac{1}{n w_i} \right), \\  
		\widetilde{C}  &  \eqdef  & \max_{i=1,\ldots,n} \left( {C_i}  \frac{1}{n w_i} \right).
	\end{eqnarray*} 
\end{lemma}

\begin{proof}
	Define  $W^t \eqdef \{g_1^t, \dots,    g_n^t, x^t, x^{t+1}\}$. The proof starts as follows:
\begin{eqnarray*}
		\ExpBr{ G_i^{t+1} \;|\; W^t} &\overset{\eqref{eq:weighted_def_grad_distortion}}{=}& \ExpBr{  \sqnorm{g_i^{t+1} 
				- \frac{\nabla f_i(x^{t+1})}{n w_i}}  \;|\; W^t}	 \\
		&\overset{\text{line}~\ref{line:weighted_ef21_sgd_grad_update_abc}}{=}& \ExpBr{  \sqnorm{g_i^t + \cC_i^t \left( \frac{\hat{g_i}(x^{t+1})}{n w_i} - g_i^t \right) 
				- \frac{\nabla f_i(x^{t+1})}{n w_i}}  \;|\; W^t}	 \\
		&=& \ExpBr{  \sqnorm{\cC_i^t \left( \frac{\hat{g_i}(x^{t+1})}{n w_i} - g_i^t \right) - \left(\frac{\hat{g_i}(x^{t+1})}{n w_i} - g_i^t\right) + \frac{\hat{g_i}(x^{t+1})}{n w_i}
				- \frac{\nabla f_i(x^{t+1})}{n w_i}}  \;|\; W^t}	 \\
		&\overset{\eqref{eq:young}}{\leq}& (1+s) \ExpBr{ \sqnorm{\cC_i^t \left( \frac{\hat{g_i}(x^{t+1})}{n w_i} - g_i^t \right) - \left(\frac{\hat{g_i}(x^{t+1})}{n w_i} - g_i^t\right)}\;|\; W^t}  \\ 
		& & \qquad + (1+s^{-1})  \ExpBr{\sqnorm{\frac{\hat{g_i}(x^{t+1})}{n w_i}
				- \frac{\nabla f_i(x^{t+1})}{n w_i}}\;|\; W^t} \\
		&\overset{\eqref{eq:compressor_contraction}}{\leq}&  (1-\alpha) (1+s) \ExpBr{\sqnorm{ \left( \frac{\hat{g_i}(x^{t+1})}{n w_i} - \frac{\nabla f_i(x^{t+1})}{n w_i} \right) + \left( \frac{\nabla f_i(x^{t+1})}{n w_i} - g_i^t\right) } \;|\; W^t}  \\ 
		& & \qquad + (1+s^{-1})  \ExpBr{\sqnorm{\frac{\hat{g_i}(x^{t+1})}{n w_i}
				- \frac{\nabla f_i(x^{t+1})}{n w_i}}\;|\; W^t}\\
&=& (1 - \alpha) (1 + s) \ExpBr{\sqnorm{{g_i}^{t} - \frac{\nabla f_i(x^{t+1})}{n w_i}}|\; W^t} \\
		& & \qquad + (1 - \alpha) (1 + s) \ExpBr{\sqnorm{\frac{\nabla f_i(x^{t+1})}{n w_i} - \frac{\hat{g_i}(x^{t+1})}{n w_i}}\;|\; W^t} \\
		& & \qquad + (1+s^{-1}) \ExpBr{\sqnorm{\frac{\hat{g_i}(x^{t+1})}{n w_i}
				- \frac{\nabla f_i(x^{t+1})}{n w_i}}\;|\; W^t} \\
&=& (1 - \alpha) (1 + s) \ExpBr{\sqnorm{{g_i}^{t} - \frac{\nabla f_i(x^{t})}{n w_i} +
				\frac{\nabla f_i(x^{t})}{n w_i} - \frac{\nabla f_i(x^{t+1})}{n w_i}}|\; W^t} \\
		& & \qquad + (1 - \alpha) (1 + s) \ExpBr{\sqnorm{\frac{\nabla f_i(x^{t+1})}{n w_i} - \frac{\hat{g_i}(x^{t+1})}{n w_i}}\;|\; W^t} \\
		& & \qquad + (1+s^{-1}) \ExpBr{\sqnorm{\frac{\hat{g_i}(x^{t+1})}{n w_i}
				- \frac{\nabla f_i(x^{t+1})}{n w_i}}\;|\; W^t} .			
\end{eqnarray*}				

Further, we continue as follows				
\begin{eqnarray*}
		\ExpBr{ G_i^{t+1} \;|\; W^t} 							
		&\overset{\eqref{eq:young}}{\leq}& (1 - \alpha) (1 + s) (1+\nu) \ExpBr{\sqnorm{ {g_i}^{t} - \frac{\nabla f_i(x^t)}{n w_i}}\;|\; W^t} \\
		& & \qquad + (1 - \alpha) (1 + s) (1+\nu^{-1}) {\sqnorm{\frac{\nabla f_i(x^{t+1})}{n w_i} - \frac{\nabla f_i(x^{t})}{n w_i}}} \\
		& & \qquad + {\left((1+s^{-1})+ (1-\alpha)(1 + s) \right)} \ExpBr{\sqnorm{\frac{\hat{g_i}(x^{t+1})}{n w_i}
				- \frac{\nabla f_i(x^{t+1})}{n w_i}}\;|\; W^t}.
	\end{eqnarray*}
	
	To further bound the last term, which contains multiple ${(1+s^{-1})}$ factors, we leverage the property that $\hat g_i(x^{t+1})$ is a random variable serving as an unbiased estimator of $\nabla f_i(x^{t+1})$. Our approach is as follows:
	\begin{eqnarray*}
		\ExpBr{ G_i^{t+1} \;|\; W^t} &\leq& (1-\hat{\theta}) \ExpBr{G_i^t \;|\; W^t} + \hat{\beta_1} \frac{1}{n^2 w_i^2} \sqnorm{\nabla f_i(x^{t+1}) - \nabla f_i(x^t)} \\
		& & \qquad + \frac{{\hat{\beta_2}}}{(n w_i)^2} \ExpBr{\sqnorm{ \hat{g_i} (x^{t+1}) - {\nabla f_i}(x^{t+1})}\;|\; W^t} .	\end{eqnarray*}

	Now due to the requirement of unbiasedness of gradient estimators expressed in the form of Assumption~\ref{as:general_as_unbiased_for_stoch_gradients_abc} we have the following:
	\begin{eqnarray}
	\label{eq:var-decompoposition-abc}
	\ExpBr{\sqnorm{ \hat{g_i} (x^{t+1}) - {\nabla f_i}(x^{t+1})}\;|\; W^t} &=& \ExpBr{\sqnorm{ \hat{g_i} (x^{t+1})}\;|\; W^t} - {\sqnorm{ \nabla f_{{i}} (x^{t+1})}}
 	\end{eqnarray}
 
	Using this variance decomposition, we can proceed as follows.
	
	\begin{eqnarray*}
		\ExpBr{ G_i^{t+1} \;|\; W^t} &\overset{\eqref{eq:var-decompoposition-abc}}{\leq}& (1-\hat{\theta}) \ExpBr{G_i^t \;|\; W^t} + \hat{\beta_1} \frac{1}{n^2 w_i^2} \sqnorm{\nabla f_i(x^{t+1}) - \nabla f_i(x^t)} \\
		& & \qquad + \frac{{\hat{\beta_2}}}{(n w_i)^2 } \left( \ExpBr{\sqnorm{ \hat{g_i} (x^{t+1})}\;|\; W^t} - {\sqnorm{ \nabla f_{i}(x^{t+1})}} \right) \\
		&\overset{\eqref{eq:general_second_mom_upp_bound_abc}}{\leq}&	(1-\hat{\theta}) \ExpBr{G_i^t \;|\; W^t} + \hat{\beta_1} \frac{1}{n^2 w_i^2} \sqnorm{\nabla f_i(x^{t+1}) - \nabla f_i(x^t)} \\
		& &  + \frac{{\hat{\beta_2}}}{(n w_i)^2} \left(
		2A_i\left(f_i(x^{t+1}) - f_i^{\inf}\right) + B_i\|\nabla f_i(x^{t+1})\|^2 +   C_i - \sqnorm{\nabla f_i(x^{t+1}}) \right) \\
		&=& (1-\hat{\theta}) \ExpBr{G_i^t \;|\; W^t} + \hat{\beta_1} \frac{1}{n^2 w_i^2} \sqnorm{\nabla f_i(x^{t+1}) - {\nabla f_i(x^t)}} \\
		& & \qquad + \frac{ 2A_i {\hat{\beta_2}}}{(n w_i)^2} \left(f_i(x^{t+1}) - f_i^{\inf}\right)
		+  \frac{2(B_i-1){\hat{\beta_2}}}{(n w_i)^2} \left(\frac{1}{2} \|\nabla f_i(x^{t+1})\|^2 \right)
		+  \frac{ C_i {\hat{\beta_2}}}{(n w_i)^2} \\
		&\le& (1-\hat{\theta}) \ExpBr{G_i^t \;|\; W^t} + \hat{\beta_1} \frac{1}{n^2 w_i^2} \sqnorm{\nabla f_i(x^{t+1}) - {\nabla f_i(x^t)}} \\
		& & \qquad + \frac{ 2A_i {\hat{\beta_2}}}{(n w_i)^2} \left(f_i(x^{t+1}) - f_i^{\inf}\right)
		+  \frac{2(B_i-1){\hat{\beta_2}}}{(n w_i)^2} L_i \left(f_i(x^{t+1}) - f_i^{\inf} \right)
		+  \frac{C_i {\hat{\beta_2}}}{(n w_i)^2} \\
		&=& (1-\hat{\theta}) \ExpBr{G_i^t \;|\; W^t} + \hat{\beta_1} \frac{1}{n^2 w_i^2} \sqnorm{\nabla f_i(x^{t+1}) - {\nabla f_i(x^t)}} \\
		& & \qquad + \frac{ 2(A_i + L_i(B_i - 1)) {\hat{\beta_2}}}{(n w_i)^2} \left(f_i(x^{t+1}) - f_i^{\inf}\right)
		+  \frac{C_i {\hat{\beta_2}}}{(n w_i)^2 }.
	\end{eqnarray*}
	
	Next leveraging \Cref{as:L_i} we replace the second term in the last expression, and we can derive the subsequent bound:
	\begin{eqnarray*}
		\ExpBr{ G_i^{t+1} \;|\; W^t} &\overset{\eqref{eq:L_i}}{\leq}& (1-\hat{\theta})   G_i^t +  \frac{\hat{\beta_1} L_i^2}{n^2 w_i^2} \sqnorm{ x^{t+1} - x^t}  \\
		&& \qquad 
		+ \frac{ 2(A_i + L_i(B_i - 1)) {\hat{\beta_2}}}{(n w_i)^2} \left(f_i(x^{t+1}) - f_i^{\inf}\right)
		+  \frac{C_i {\hat{\beta_2}}}{(n w_i)^2}.
	\end{eqnarray*}
	Applying the tower property and subsequently taking the expectation, we obtain:
	\begin{equation}\label{eq:weighted_ef21_sgd_aux_1_abc}
		\begin{aligned}
			\ExpBr{ G_i^{t+1}} &\leq (1-\hat{\theta}) \ExpBr{G_i^t} + \hat{\beta_1} \frac{1}{n^2 w_i^2} L_i^2 \ExpBr{\sqnorm{ x^{t+1} - x^t}} 
			\\
			& \qquad + \frac{ 2(A_i + L_i(B_i - 1)) {\hat{\beta_2}}}{(n w_i)^2} \ExpBr{f_i(x^{t+1}) - f_i^{\inf}}
			+  \frac{C_i {\hat{\beta_2}}}{(n w_i)^2}.
		\end{aligned}
	\end{equation}

	Next for the expectation of the main quantity of our interest $G^{t+1}$, we derive the subsequent bound:
	\begin{eqnarray*}
		\ExpBr{G^{t+1}} &= & \ExpBr{\sumin w_i G_i^{t+1}} \\
		& = &  \sumin w_i \ExpBr{G_i^{t+1}} \\
		&\overset{\eqref{eq:weighted_ef21_sgd_aux_1_abc}}{\leq}  & (1-\hat{\theta}) \sumin w_i \ExpBr{G_i^t} + \sumin w_i \hat{\beta_1} \frac{1}{n^2 w_i^2} L_i^2 \cdot \ExpBr{\sqnorm{ x^{t+1} - x^t}} \\
		& &  \qquad + \sumin w_i \frac{ 2(A_i + L_i(B_i - 1)) {\hat{\beta_2}}}{(n w_i)^2} \cdot \ExpBr{f_i(x^{t+1}) - f_i^{\inf}}
		+  \sumin w_i \frac{C_i {\hat{\beta_2}}}{(n w_i)^2}	\\
		&=  & (1-\hat{\theta}) \ExpBr{G^t} + \sumin \hat{\beta_1} \frac{1}{n^2 w_i} L_i^2 \cdot \ExpBr{\sqnorm{ x^{t+1} - x^t}} \\
		& &  \qquad + \sumin \frac{ 2(A_i + L_i(B_i - 1)) {\hat{\beta_2}}}{(n)^2 w_i} \cdot \ExpBr{f_i(x^{t+1}) - f_i^{\inf}}
		+  \sumin \frac{C_i {\hat{\beta_2}}}{n^2 w_i}
	\end{eqnarray*}
	
	
	Employing quantities $\tilde{A}$ and $\tilde{C}$, the final bound can be reformulated as follows:
	\begin{eqnarray}
		\ExpBr{G^{t+1}} \leq (1-\hat{\theta}) \ExpBr{G^t} + \sumin \hat{\beta_1} \frac{1}{n^2 w_i} L_i^2 \cdot \ExpBr{\sqnorm{ x^{t+1} - x^t}} \notag \\
		+ \frac{1}{n} \sumin { \widetilde{A} {\hat{\beta_2}}} \cdot \ExpBr{f_i(x^{t+1}) - f_i^{\inf}}
		+   {\widetilde{C} {\hat{\beta_2}}} \notag \\
		\leq (1-\hat{\theta}) \ExpBr{G^t} + \sumin \hat{\beta_1} \frac{1}{n^2 w_i} L_i^2 \cdot \ExpBr{\sqnorm{ x^{t+1} - x^t}} \notag \\
		+ \frac{1}{n} \sumin { \widetilde{A} {\hat{\beta_2}}} \cdot \ExpBr{f_i(x^{t+1}) - f^{\inf}}
		+   {\widetilde{C} {\hat{\beta_2}}} \notag \\
		\le (1-\hat{\theta}) \ExpBr{G^t} + \sumin \hat{\beta_1} \frac{1}{n^2 w_i} L_i^2 \cdot \ExpBr{\sqnorm{ x^{t+1} - x^t}} \notag \\
		+ { \widetilde{A} {\hat{\beta_2}}} \notag \ExpBr{f(x^{t+1}) - f^{\inf}}
		+  {\widetilde{C} {\hat{\beta_2}}}. \notag
	\end{eqnarray}

	
	Given that $w_i=\frac{L_i}{\sum_j L_j}$, we have:
	\begin{equation}
		\begin{aligned}
			\ExpBr{G^{t+1}} &\leq (1-\hat{\theta}) \ExpBr{G^t} + \frac{1}{n}\sumin \hat{\beta_1} \frac{\sum_j L_j}{n} L_i \ExpBr{\sqnorm{ x^{t+1} - x^t}} \notag \\
			& \qquad + { \widetilde{A} {\hat{\beta_2}}} \ExpBr{f(x^{t+1}) - f^{\inf}}
			+  {\widetilde{C} {\hat{\beta_2}}} \notag \\
			& = (1-\hat{\theta}) \ExpBr{G^t} + \hat{\beta_1} \left(\avein L_i\right)^2 \cdot \ExpBr{\sqnorm{ x^{t+1} - x^t}} \notag \\
			& \qquad + { \widetilde{A} {\hat{\beta_2}}} \ExpBr{f(x^{t+1}) - f^{\inf}}
			+  {\widetilde{C} {\hat{\beta_2}}},
		\end{aligned}
	\end{equation}
	what completes the proof.
\end{proof}

\subsection{Main result}

Now we are ready to prove the main convergence theorem.
\begin{theorem} Let $\cC_i^t \in\mathbb{C}(\alpha)$ for all $\in [n]$ and $t\geq 0$ in \Cref{alg:weighted_ef21_sgd_abc}. set the following quantities:
	\begin{eqnarray*}
		\hat{\theta} &\eqdef & 1 - \rb{1-\alpha} (1+s) (1+\nu),\\
		\hat{\beta_1} &\eqdef  & (1- \alpha ) \left(1+ s \right)\left(s+\nu^{-1}\right),\\
		\hat{\beta_2} &\eqdef  & (1 - \alpha) (1 + s) + (1 + s^{-1}),\\
		w_i &\eqdef  & \frac{L_i}{\sum_{j=1}^n L_j},\\		
		\widetilde{A} & \eqdef  & \max_{i=1,\ldots,n}  \frac{2(A_i+L_i(B_i-1))}{n w_i}, \\
		\widetilde{C} & \eqdef  & \max_{i=1,\ldots,n}  \frac{C_i}{n w_i} .
	\end{eqnarray*}
	Under Assumptions~\ref{as:smooth},~\ref{as:L_i}, ~\ref{as:general_as_for_stoch_gradients_abc}, ~\ref{as:general_as_unbiased_for_stoch_gradients_abc}, and  selection of $s >0, \nu >0$ small enough such that
	$(1+s)(1+\nu) < \frac{1}{1-\alpha}$ holds, set the stepsize in the following way:
	\begin{equation}
		\gamma \leq \frac{1}{L + \LAM \sqrt{\frac{\hat{\beta}_1}{\hat{\theta}}}}.
	\end{equation}
	Choose an iterate $\hat{x}^T$ from $\{x^0, x^1, \dots, x^{T-1}\}$ with probability 
	\begin{equation} 
		\label{eq:09u09fd-0ff-abc}
		\Prob(\hat{x}^T = x^t) = \frac{v_t}{V_T},
	\end{equation} 
	where
	$$ v_t \eqdef \left(1 - \frac{\gamma \tilde{A} \tilde{\beta}_2}{2\theta}\right)^t; \qquad V_T \eqdef \sum\limits_{t=0}^{T-1} v_t.
	$$
	Then,
	\begin{equation}
		\ExpBr{\sqnorm{\nabla f(\hat{x}^{T})}} \leq \frac{2 (f(x^0) - f^\text{inf})}{\gamma T  \left(1 - \frac{\gamma \widetilde{A} \hat{\beta_2}}{2 \theta}\right)^T } + \frac{G^0}{ \hat{\theta}T  \left(1 - \frac{\gamma \widetilde{A} \hat{\beta_2}}{2 \theta}\right)^T} + \frac{\widetilde{C}\beta_2}{\hat{\theta}},
	\end{equation}
	where $G^0 \eqdef \sumin w_i \norm{g_i^0 - \frac{1}{nw_i}\nabla f_i(x^0)}^2$.
\end{theorem}
\begin{proof}
	In the derivation below, we  use \Cref{lm:descent_lemma} for 
	\begin{equation}
		\label{eq:sgd_weighted_grad_estimate_def_abc}
		g^t=\sum\limits_{i=1}^{n} w_i g_i^t.
	\end{equation}
	We start as follows:
	\begin{eqnarray*}
			f(x^{t+1}) &\overset{\eqref{eq:descent_lemma}}{\leq} & 
			f(x^{t})-\frac{\gamma}{2}\sqnorm{\nabla f(x^{t})}-\left(\frac{1}{2 \gamma}-\frac{L}{2}\right)\sqnorm{x^{t+1}-x^{t}}+\frac{\gamma}{2}\sqnorm{
				g^t- \frac{1}{n} \sumin \nabla f_i(x^{t}) } \\
			& \overset{\eqref{eq:sgd_weighted_grad_estimate_def}}{=} & 
			f(x^{t})-\frac{\gamma}{2}\sqnorm{\nabla f(x^{t})}-\left(\frac{1}{2 \gamma}-\frac{L} 
			{2}\right)\sqnorm{x^{t+1}-x^{t}}+\frac{\gamma}{2}\sqnorm{
				\sumin w_i \left(g_i^t-  \frac{\nabla f_i(x^{t})}{n w_i} \right) }  \\
			& \leq & 
			f(x^{t})-\frac{\gamma}{2}\sqnorm{\nabla f(x^{t})}-\left(\frac{1}{2 \gamma}-\frac{L}       
			{2}\right)\sqnorm{x^{t+1}-x^{t}}+\frac{\gamma}{2}
			\sumin w_i \sqnorm{ g_i^t-  \frac{\nabla f_i(x^{t})}{n w_i}  } \\
			& = & f(x^{t})-\frac{\gamma}{2}\sqnorm{\nabla f(x^{t})}-\left(\frac{1}{2 \gamma}-\frac{L}       
			{2}\right)\sqnorm{x^{t+1}-x^{t}}+\frac{\gamma}{2} G^t.
	\end{eqnarray*}

	Subtracting $f^\ast$ from both sides and taking expectation, we get
	\begin{equation*}
		\ExpBr{f(x^{t+1})-f^\ast} \leq  \ExpBr{f(x^{t})-f^\ast}
		-\frac{\gamma}{2} \ExpBr{\sqnorm{\nabla f(x^{t})}}   -\left(\frac{1}{2 \gamma}
		-\frac{L}{2}\right) \ExpBr{\sqnorm{x^{t+1}-x^{t}}}+ \frac{\gamma}{2}\ExpBr{G^t}.
	\end{equation*}

	Let $\delta^{t} \eqdef \ExpBr{f(x^{t})-f^\ast}$, $s^{t} \eqdef \ExpBr{G^t }$ and 
	$r^{t} \eqdef\ExpBr{\sqnorm{x^{t+1}-x^{t}}}.$	Then by adding $\frac{\gamma}{2\theta} s^{t+1}$ and employing inequality  \eqref{eq:weighted_ef21_sgd_grad_full_contraction}, we obtain:
	\begin{align*}
		\delta^{t+1}+\frac{\gamma}{2 \hat{\theta}} s^{t+1} 
		&\leq \delta^{t}-\frac{\gamma}{2}\ExpBr{\sqnorm{\nabla f(x^{t})}} -          \left(\frac{1}{2 \gamma}-\frac{L}{2}\right) r^{t}+\frac{\gamma}{2} s^{t} \\
		& \qquad +\frac{\gamma}{2 \hat{\theta}}\left( \hat{\beta_1} \LAMsq  r^t + (1-\hat{\theta}) s^t + { \widetilde{A} {\hat{\beta_2}}} \delta^{t+1} +  {\widetilde{C} {\hat{\beta_2}}} \right) \\
		&=\delta^{t}+\frac{\gamma}{2\hat{\theta}} s^{t}-\frac{\gamma}{2}\ExpBr{\sqnorm{\nabla f(x^{t})}}  -\left(\frac{1}{2\gamma} -\frac{L}{2} -  \frac{\gamma}{2\hat{\theta}} \hat{\beta_1}  \LAMsq  \right) r^{t} + \frac{\gamma \widetilde{A} \beta_2}{2\hat{\theta}} \delta^{t+1} + \frac{\gamma \widetilde{C}}{2 \hat{\theta}} \beta_2\\
		& \leq \delta^{t}+\frac{\gamma}{2\hat{\theta}} s^{t} -\frac{\gamma}{2}\ExpBr{\sqnorm{\nabla f(x^{t})}} + \frac{\gamma \widetilde{A} \beta_2}{2\hat{\theta}} \delta^{t+1} + \frac{\gamma \widetilde{C}}{2 \hat{\theta}} \beta_2.
	\end{align*}
	
	
	The last inequality follows from the bound $\gamma^2\frac{\hat{\beta_1} \LAMsq }{\hat{\theta}} + L\gamma \leq 1$, which holds due to \Cref{lm:stepsize_bound} for $$\gamma \leq \frac{1}{L +  \LAM \sqrt{\frac{\hat{\beta}_1}{\hat{\theta}}}}.$$ Subsequently, we will reconfigure the final inequality and perform algebraic manipulations, taking into account that $\frac{2}{\gamma} > 0$. In the final step of these algebraic transformations, we will leverage the fact that $s^t \ge 0$:
	\begin{eqnarray*}
		\delta^{t+1}+\frac{\gamma}{2 \hat{\theta}} s^{t+1} &\leq& \delta^{t}+\frac{\gamma}{2\hat{\theta}} s^{t} -\frac{\gamma}{2}\ExpBr{\sqnorm{\nabla f(x^{t})}} + \frac{\gamma \widetilde{A} \beta_2}{2\hat{\theta}} \delta^{t+1} + \frac{\gamma \widetilde{C}}{2 \hat{\theta}} \beta_2.
	\end{eqnarray*}		

Therefore,		
	\begin{eqnarray*}		
		\frac{2}{\gamma} \delta^{t+1}+\frac{2}{\gamma} \frac{\gamma}{2 \hat{\theta}} s^{t+1} &\leq& \frac{2}{\gamma} \delta^{t} + \frac{2}{\gamma} \frac{\gamma}{2\hat{\theta}} s^{t} -\ExpBr{\sqnorm{\nabla f(x^{t})}} + \frac{2}{\gamma} \frac{\gamma \widetilde{A} \beta_2}{2\hat{\theta}} \delta^{t+1} + \frac{2}{\gamma} \frac{\gamma \widetilde{C}}{2 \hat{\theta}} \beta_2 .
	\end{eqnarray*}				

Further,	
	\begin{eqnarray*}	
		\ExpBr{\sqnorm{\nabla f(x^{t})}} &\leq& -\frac{2}{\gamma} \delta^{t+1} - \frac{2}{\gamma} \frac{\gamma}{2 \hat{\theta}} s^{t+1} + \frac{2}{\gamma} \delta^{t} + \frac{2}{\gamma} \frac{\gamma}{2\hat{\theta}} s^{t} + \frac{2}{\gamma} \frac{\gamma \widetilde{A} \beta_2}{2\hat{\theta}} \delta^{t+1} + \frac{2}{\gamma} \frac{\gamma \widetilde{C}}{2 \hat{\theta}} \beta_2  \\
 &\leq&
		-\frac{2}{\gamma} \delta^{t+1} - \frac{2}{\gamma} \frac{\gamma}{2\hat{\theta}} s^{t+1} + \frac{2}{\gamma} \left( \delta^{t} + \frac{\gamma}{2\hat{\theta}} s^{t} \right) + \frac{2}{\gamma} \frac{\gamma \widetilde{A} \beta_2}{2 \hat{\theta}} \delta^{t+1} + \frac{\widetilde{C}\beta_2}{ \hat{\theta}}  \\
 &\leq&
		\frac{2}{\gamma} \left( \left( \delta^{t} + \frac{\gamma}{2\hat{\theta}} s^{t} \right) -1\left(1 - \frac{\gamma \widetilde{A} \beta_2}{2 \hat{\theta}} \right) \delta^{t+1} -\left(\frac{\gamma}{2\hat{\theta}} s^{t+1} \right) \right) + \frac{\widetilde{C}\beta_2}{ \hat{\theta}} \\
		&\leq& \frac{2}{\gamma} \left( \left( \delta^{t} + \frac{\gamma}{2\hat{\theta}} s^{t} \right) -\left(1 - \frac{\gamma \widetilde{A} \beta_2}{2 \hat{\theta}} \right) \left(\delta^{t+1} + \frac{\gamma}{2\hat{\theta}} s^{t+1} \right) \right) + \frac{\widetilde{C}\beta_2}{ \hat{\theta}}.
	\end{eqnarray*}
	
	We sum up inequalities above with weights $v_t/V_T$, where $v_t \eqdef (1 - \frac{\gamma \widetilde{A} \hat{\beta_2}}{2 \theta})^t$ and $V_T \eqdef \sum_{i=1}^{T} v_i$:
	\begin{eqnarray*}
		\ExpBr{\sqnorm{\nabla f(\hat{x}^{T})}} &=& \sum_{t=0}^{T} \frac{v_t}{V_T} \ExpBr{\sqnorm{\nabla f(x^{t})}} \\
		&=& \frac{1}{V_T} \sum_{t=0}^{T} v_t \ExpBr{\sqnorm{\nabla f(x^{t})}} \\
		&\leq& \frac{1}{V_T} \sum_{t=0}^{T} v_t \left(\frac{2}{\gamma} \left( \left( \delta^{t} + \frac{\gamma}{2\hat{\theta}} s^{t} \right) -\left(1 - \frac{\gamma \widetilde{A} \beta_2}{2 \hat{\theta}} \right) \left(\delta^{t+1} + \frac{\gamma}{2\hat{\theta}} s^{t+1} \right) \right) + \frac{\widetilde{C}\beta_2}{ \hat{\theta}} \right) \\
		&=& \frac{2}{\gamma V_T} \sum_{t=0}^{T} w_t \left( \left( \delta^{t} + \frac{\gamma}{2\hat{\theta}} s^{t} \right) -\left(1 - \frac{\gamma \widetilde{A} \beta_2}{2 \hat{\theta}} \right) \left(\delta^{t+1} + \frac{\gamma}{2\hat{\theta}} s^{t+1} \right) \right)	+ \sum_{t=0}^{T} \frac{w_t}{W_T} \cdot \frac{\widetilde{C}\beta_2}{ \hat{\theta}} \\
		&=& \frac{2}{\gamma V_T} \sum_{t=0}^{T} w_t \left( \left( \delta^{t} + \frac{\gamma}{2\hat{\theta}} s^{t} \right) -\left(1 - \frac{\gamma \widetilde{A} \beta_2}{2 \hat{\theta}} \right) \left(\delta^{t+1} + \frac{\gamma}{2\hat{\theta}} s^{t+1} \right) \right)	+ \frac{\widetilde{C}\beta_2}{ \hat{\theta}} \\
		&=& \frac{2}{\gamma V_T} \sum_{t=0}^{T} \left(w_t \left( \delta^{t} + \frac{\gamma}{2\hat{\theta}} s^{t} \right) -w_{t+1} \left(\delta^{t+1} + \frac{\gamma}{2\hat{\theta}} s^{t+1} \right) \right)	+ \frac{\widetilde{C}\beta_2}{ \hat{\theta}} \\
		&\leq& \frac{2 \delta^0}{\gamma V_T} + \frac{s^0}{ \hat{\theta}V_T} + \frac{\widetilde{C}\beta_2}{\hat{\theta}}.
	\end{eqnarray*}
	Finally, we notice that $V_T = \sum\limits_{t=1}^T (1 - \frac{\gamma \widetilde{A} \hat{\beta_2}}{2 \theta})^t \geq T \cdot (1 - \frac{\gamma \widetilde{A} \hat{\beta_2}}{2 \theta})^T$, what concludes the proof.
\end{proof}

\clearpage
\section{EF21-W-PP: Weighted Error Feedback 2021 with Partial Participation} \label{sec:EF21-W-PP}

In this section, we present another extension of error feedback. Again, to maintain brevity, we show our results for~\algname{EF21-W}, however, we believe getting an enhanced rate for standard \algname{EF21} should be straightforward. 

\subsection{Algorithm}
Building upon the delineation of \algname{EF21-W} in  \Cref{alg:EF21-W}, we turn our attention to its partial participation variant, \algname{EF21-W-PP}, and seek to highlight the primary distinctions between them. One salient difference is the introduction of a distribution, denoted as $\cD$, across the clients. For clarity, consider the power set ${\cal P}$ of the set $[n] \eqdef \{1, 2, \dots, n\}$, representing all possible subsets of $[n]$. Then, the distribution ${\cal D}$ serves as a discrete distribution over ${\cal P}$.

While \algname{EF21-W-PP} runs, at the start of each communication round $t$, the master, having computed a descent step as $x^{t+1} = x^t - \gamma g^t$, samples a client subset $S^t$ from the distribution $\cD$. Contrasting with  \Cref{alg:EF21-W} where the new iteration $x^{t+1}$ is sent to all clients, in this variant, it is sent exclusively to those in $S^t$.

Any client $i \in S^t$ adheres to procedures akin to ~\algname{EF21-W}: it compresses the quantity $\frac{1}{nw_i} \nabla f_i(x^t) - g_i^t$ and transmits this to the master. Conversely, client $j$ omitted in $S^t$, i.e., $j \notin S^t$, is excluded from the training for that iteration. Concluding the round, the master updates $g^{t+1}$ by integrating the averaged compressed variances received from clients in the set $S^t$.

\begin{algorithm}
	\begin{algorithmic}[1]
		\STATE {\bfseries Input:} initial model parameters $x^0 \in \RR^d$; initial gradient estimates $g_1^0, g_2^0, \dots,g_n^0 \in \R^d$ stored at the clients; weights ${\color{ForestGreen}w_i} = \nicefrac{L_i}{\sum_j L_j}$; stepsize $\gamma>0$; number of iterations $T > 0$; distribution $\cD$ over clients
		\STATE {\bfseries Initialize:} $g^0 = \sum_{i=1}^n {\color{ForestGreen}w_i} g_i^0 $ on the server	
		\FOR{$t = 0, 1, 2, \dots, T - 1 $}
		\STATE Server computes $x^{t+1} = x^t - \gamma g^t$		
		\STATE Server samples a subset $S^t \sim \cD$ of clients
		\STATE Server broadcasts  $x^{t+1}$ to clients in $S^t$
		\FOR{$i = 1, \dots, n$ {\bf on the clients in parallel}} 		
		\IF{$i\in S^t$}		
		\STATE Compute $u_i^t=\cC_i^t (\frac{1}{n {\color{ForestGreen}w_i}}\nabla f_i(x^{t+1}) - g_i^t)$ and update $g_i^{t+1} = g_i^t +u_i^t$ \label{line:weighted_ef21_pp_grad_ind_update_step}	
		\STATE Send the compressed message $u_i^{t}$ to the server	
		\ENDIF
		\IF{$i\notin S^t$}
		\STATE Set $u_i^t = 0$		for the client and the server
		\STATE Do not change local state $g_i^{t+1} = g_i^t$
		\ENDIF
		\ENDFOR
		\STATE Server updates $g_i^{t+1} = g_i^t +u_i^t$ for all $i\in [n]$, and computes $g^{t+1} = \sum_{i=1}^n {\color{ForestGreen}w_i} g_i^{t+1}$		\label{line:averaging_ef21_pp_weighting} 
		\ENDFOR
		\STATE {\bfseries Output:} Point $\hat{x}^T$ chosen from the set $\{x^0, \dots, x^{T-1}\}$ uniformly at random		
	\end{algorithmic}
	\caption{\algname{EF21-W-PP}: Weighted Error Feedback 2021 with Partial Participation}
	\label{alg:weighted_ef21_pp}
\end{algorithm}

Assume $S$ is drawn from the distribution $\cD$. Let us denote 
\begin{equation}\label{eq:p_i_def}
p_i \eqdef \Prob(i \in S^t).
\end{equation}
In other words, $p_i$ represents the probability of client $i$ being selected in any iteration.  For given parameters $p_i$ such that $p_i \in (0,1]$ for  $i\in [n]$, we introduce the notations $p_{\min} \eqdef \min_i p_i$ and $p_{\max} \eqdef \max_i p_i$, respectively.

\subsection{A lemma}

Having established the necessary definitions, we can now proceed to formulate the lemma.
\begin{lemma} \label{lemma:contraction_ef21_pp_weighted}
	Let $\cC_i^t\in \mathbb{C}(\alpha)$ for all $i\in [n]$ and $t\geq 0$. Let  \Cref{as:L_i} hold. Define
	\begin{equation}\label{eq:grad_distortion_def_pp}
	G_i^t \eqdef  \sqnorm{ g_i^t - \frac{\nabla f_i(x^{t})}{n w_i} } , \qquad G^t \eqdef \sumin w_i G_i^t.
	\end{equation}
	For any $s>0$ and $\rho>0$, let us define the following quantities: 
	\begin{eqnarray*}
	\theta(\alpha, s) &\eqdef& 1 - (1 - \alpha)(1 + s)\\
	\beta(\alpha, s) &\eqdef& \beta(\alpha, s) = (1 - \alpha)(1+s^{-1})\\
	\theta_p &\eqdef& p_{\min}\rho + \theta(\alpha,s) p_{\max} - \rho - (p_{\max}-p_{\min}) \\
	\tilde{B} &\eqdef& \left({\beta(\alpha,s) p_{\max}} + (1-p_{\min}){(1+\rho^{-1})} \right) \LAMsq.
\end{eqnarray*}
	Additionally, assume that
 	$$
	\frac{1 + \rho(1 - p_{\min}) + (p_{\max} - p_{\min}) }{p_{\max}}\geq \theta(\alpha, s) > \frac{\rho(1 - p_{\min}) + (p_{\max} - p_{\min}) }{p_{\max}}.
	$$
	Then, we have
	\begin{equation} 
	\ExpBr{G^{t+1}} \le (1-\theta_p) \ExpBr{G^t} +
\tilde{B} \ExpBr{\sqnorm{x^{t+1} - x^{t}}}.
	\end{equation}
\end{lemma}

\begin{proof}
	Let us define $W^t \eqdef \{g_1^t, \dots,    g_n^t, x^t, x^{t+1}\}$. If client $i$ participates in the training at iteration $t$, then
	\begin{eqnarray*}
		\ExpBr{ G_i^{t+1} \;|\; W^t, i \in S^t} & \overset{\eqref{eq:grad_distortion_def_pp}}{=} & \ExpBr{  \sqnorm{g_i^{t+1} 
				- \frac{\nabla f_i(x^{t+1})}{n w_i}}  \;|\; W^t, i \in S^t}	 \\
		&\overset{\text{line}~\ref{line:weighted_ef21_pp_grad_ind_update_step}\text{ of \Cref{alg:weighted_ef21_pp}}}{=}& \ExpBr{  \sqnorm{g_i^t + \cC_i^t \left( \frac{\nabla f_i(x^{t+1})}{n w_i} - g_i^t \right) 
				- \frac{\nabla f_i(x^{t+1})}{n w_i}}  \;|\; W^t, i \in S^t}	 \\
		&\overset{\eqref{eq:compressor_contraction}}{\leq} &  (1-\alpha) 
		\sqnorm{\frac{\nabla f_i(x^{t+1})}{n w_i} - g_i^t} \\
		& = & (1-\alpha) 
		\sqnorm{\frac{\nabla f_i(x^{t+1})}{n w_i} - \frac{\nabla f_i(x^{t})}{n w_i} + \frac{\nabla f_i(x^{t})}{n w_i} - g_i^t}\\
		&\overset{\eqref{eq:young}}{\leq}& (1-\alpha) (1+ s) \sqnorm{ \frac{\nabla f_i(x^{t})}{n w_i} - g_i^t}\\
		&& \qquad  + (1-\alpha)  \left(1+s^{-1}\right) \frac{1}{n^2 w_i^2} 
		\sqnorm{\nabla f_i(x^{t+1}) - \nabla f_i(x^t)} \\
		&\overset{\eqref{eq:L_i}}{\leq} & (1-\alpha) (1+ s) \sqnorm{ \frac{\nabla f_i(x^{t})}{n w_i} - g_i^t}\\
		&& \qquad  + (1-\alpha)  \left(1+s^{-1}\right) \frac{L_i^2}{n^2 w_i^2} 
		\sqnorm{x^{t+1} - x^t}.
	\end{eqnarray*}
	
Utilizing the tower property and taking the expectation with respect to $W^t$, we derive:
\begin{eqnarray} \label{eq:g_i_for_sampled_ef21_weighted}
	\ExpBr{ G_i^{t+1} \;|\; i \in S^t} \leq (1-\theta(\alpha,s)) \ExpBr{G_i^{t}} + \beta(\alpha,s) \frac{L_i^2}{n^2 w_i^2} 
	\ExpBr{\sqnorm{x^{t+1} - x^t}},
\end{eqnarray}
where $\theta(\alpha, s) = 1 - (1 - \alpha)(1 + s)$, and $\beta(\alpha, s) = (1 - \alpha)(1+s^{-1})$. We now aim to bound the quantity $\ExpBr{G_i^{t+1} \;|\; i \notin S^t}$, starting with  an  application of the tower property:\begin{eqnarray*}
	\ExpBr{ G_i^{t+1} \;|\; i \notin S^t} &=&
	\ExpBr{\ExpBr{G_i^{t+1} \;|\; W^t, i \notin S^t}} \\
	&\overset{\eqref{eq:grad_distortion_def_pp}}{=}& \ExpBr{\ExpBr{\sqnorm{ g_i^{t+1} - \frac{\nabla f_i(x^{t+1})}{n w_i} }| \; W^t, i \notin S^t}} \\
	&=&
	\ExpBr{\ExpBr{\sqnorm{ g_i^{t} - \frac{\nabla f_i(x^{t+1})}{n w_i} + \frac{\nabla f_i(x^{t})}{n w_i} - \frac{\nabla f_i(x^{t})}{n w_i} }| \; W^t, i \notin S^t}} \\
	&\overset{\eqref{eq:young}}{\le}& \ExpBr{\ExpBr{(1+\rho)\sqnorm{ g_i^{t} - \frac{\nabla f_i(x^{t})}{n w_i} } + (1+\rho^{-1})\sqnorm{\frac{\nabla f_i(x^{t})}{n w_i} - \frac{\nabla f_i(x^{t+1})}{n w_i}}| \; W^t, i \notin S^t}} \\
	&=&
	(1+\rho) \ExpBr{G_i^t} + \frac{(1+\rho^{-1})}{n^2 w_i^2} \ExpBr{\sqnorm{\nabla f_i(x^{t+1}) - \nabla f_i(x^{t})}\; }.
\end{eqnarray*}

Given that  \Cref{as:L_i} is satisfied, by applying \eqref{eq:L_i} to the second term, we obtain:
\begin{equation}\label{eq:g_i_for_non_sampled_ef21_weighted}
\ExpBr{ G_i^{t+1} \;|\; i \notin S^t} \leq (1+\rho) \ExpBr{G_i^t} + \frac{L_i^2(1+\rho^{-1})}{n^2 w_i^2} \ExpBr{\sqnorm{x^{t+1} - x^{t}}}.
\end{equation}

We combine the two preceding bounds:
\begin{eqnarray*}
	\ExpBr{ G_i^{t+1} } & = &  \Prob(i \in S^t) \ExpBr{G_i^{t+1} \;|\; i \in S^t} + \Prob(i \notin S^t) \ExpBr{G_i^{t+1}\;|\; i \notin S^t}\\
	&\overset{\eqref{eq:p_i_def}}{=}& p_i \ExpBr{ G_i^{t+1}|\; i \in S^t} +  (1-p_i)\ExpBr{ G_i^{t+1}|\; i \notin S^t}\\
	&\overset{\eqref{eq:g_i_for_sampled_ef21_weighted}+\eqref{eq:g_i_for_non_sampled_ef21_weighted}}{\le}& p_i \left[(1-\theta(\alpha,s)) \ExpBr{G_i^{t}} +\beta(\alpha,s) \frac{L_i^2}{n^2 w_i^2} \ExpBr{\sqnorm{x^{t+1} - x^t}}\right] \\
	&&\quad + (1-p_i)\left[(1+\rho) \ExpBr{G_i^t} +  \frac{L_i^2(1+\rho^{-1})}{n^2 w_i^2} \ExpBr{\sqnorm{x^{t+1} - x^{t}}}\right] \\
	&=& \left( (1-\theta(\alpha,s))p_i + (1 - p_i)(1+\rho) \right) \ExpBr{G_i^{t}}\\ 
	& & \quad + \left({\beta(\alpha,s) p_i} + (1-p_i){(1+\rho^{-1})} \right) \frac{L_i^2}{n^2 w_i^2} \ExpBr{\sqnorm{x^{t+1} - x^{t}}}.
\end{eqnarray*}

Consequently, for $\ExpBr{G^{t+1}}$, we derive the subsequent bound:
\begin{eqnarray*}
	\ExpBr{G^{t+1}} &\overset{\eqref{eq:grad_distortion_def_pp}}{=}& \ExpBr{\sumin w_i G_i^{t+1}} \\
	&=& \sumin w_i \ExpBr{G_i^{t+1}} \\
	&\le& \sumin w_i \left( (1-\theta(\alpha,s))p_i + (1 - p_i)(1+\rho) \right) \ExpBr{G_i^{t}} \\
	&& \qquad + \sumin w_i \left({\beta(\alpha,s) p_i} + (1-p_i){(1+\rho^{-1})} \right) \frac{L_i^2}{n^2 w_i^2} \ExpBr{\sqnorm{x^{t+1} - x^{t}}},
\end{eqnarray*}
where we applied the preceding inequality. Remembering the definitions $p_{\min} \eqdef \min_i p_i$ and $p_{\max} \eqdef \max_i p_i$, we subsequently obtain:
\begin{eqnarray*}
	\ExpBr{G^{t+1}} &\le& \sumin w_i \left( (1-\theta(\alpha,s))p_{\max} + (1 - p_{\min})(1+\rho) \right) \ExpBr{G_i^{t}} \\
	&& \qquad +	\sumin \left({\beta(\alpha,s) p_{\max}} + (1-p_{\min}){(1+\rho^{-1})} \right)\frac{L_i^2}{n^2 w_i} \ExpBr{\sqnorm{x^{t+1} - x^{t}}} \\
	&=& \left( (1-\theta(\alpha,s))p_{\max} + (1 - p_{\min})(1+\rho) \right) \sumin w_i \ExpBr{G_i^t} \\
	&& \qquad +	\left({\beta(\alpha,s) p_{\max}} + (1-p_{\min}){(1+\rho^{-1})} \right) \sumin \frac{L_i^2}{n^2 w_i} \ExpBr{\sqnorm{x^{t+1} - x^{t}}}.
\end{eqnarray*}

Applying~\eqref{eq:grad_distortion_def_pp} and~\eqref{eq:weight_definition}, we obtain:
\begin{eqnarray*}
	\ExpBr{G^{t+1}} &=& \left( (1-\theta(\alpha,s))p_{\max} + (1 - p_{\min})(1+\rho) \right) \ExpBr{G^t} \\
	&& \qquad + \left({\beta(\alpha,s) p_{\max}} + (1-p_{\min}){(1+\rho^{-1})} \right) \sumin \frac{L_i^2}{n^2 \frac{L_i}{\sum_{j=1}^{n} L_j}} \ExpBr{\sqnorm{x^{t+1} - x^{t}}} \\
	&=& \left( (1-\theta(\alpha,s))p_{\max} + (1 - p_{\min})(1+\rho) \right) \ExpBr{G^t} \\
	&& \qquad +	\left({\beta(\alpha,s) p_{\max}} + (1-p_{\min}){(1+\rho^{-1})} \right) \sum_{j=1}^{n} \frac{L_j}{n} \sumin \frac{L_i}{n} \ExpBr{\sqnorm{x^{t+1} - x^{t}}} \\
	&=&	\left( (1-\theta(\alpha,s))p_{\max} + (1 - p_{\min})(1+\rho) \right) \ExpBr{G^t} \\
	&& \qquad +	\left({\beta(\alpha,s) p_{\max}} + (1-p_{\min}){(1+\rho^{-1})} \right)  \LAMsq  \ExpBr{\sqnorm{x^{t+1} - x^{t}}}.
\end{eqnarray*}
Subsequently, in order to simplify the last inequality, we introduce the variables $1 - \theta_p$ and $\tilde{B}$:
\begin{eqnarray*}
	1 - \theta_p &\eqdef& (1-\theta(\alpha,s))p_{\max} + (1 - p_{\min})(1+\rho) \\
	&=&	p_{\max} - p_{\max} \theta(\alpha,s) + 1 - p_{\min} + \rho - p_{\min}\rho \\
	&=&	1 - \left(-p_{\max} + p_{\max} \theta(\alpha,s) + p_{\min} - \rho + p_{\min}\rho\right) \\
	&=&	1 - \left(p_{\min}\rho + p_{\max} \theta(\alpha,s) - \rho - (p_{\max}-p_{\min})\right) .
\end{eqnarray*}	

Therefore,
\begin{eqnarray*}	
	\theta_p &=& \left(p_{\max} \theta(\alpha,s) - \rho(1 - p_{\min}) - (p_{\max}-p_{\min})\right) \\
	\tilde{B} &\eqdef& \left({\beta(\alpha,s) p_{\max}} + (1-p_{\min}){(1+\rho^{-1})} \right) \LAMsq.
\end{eqnarray*}

Expressed in terms of these variables, the final inequality can be reformulated as:
\begin{eqnarray*}
	\ExpBr{G^{t+1}} \le (1-\theta_p) \ExpBr{G^t} +
	\tilde{B} \ExpBr{\sqnorm{x^{t+1} - x^{t}}}.
\end{eqnarray*}

Since we need the contraction property over the gradient distortion $\ExpBr{G^{t+1}}$, we require $0 < \theta_p \leq 1$. We rewrite these conditions as follows:
$$
\frac{1 + \rho(1 - p_{\min}) + (p_{\max} - p_{\min}) }{p_{\max}}\geq \theta(\alpha, s) > \frac{\rho(1 - p_{\min}) + (p_{\max} - p_{\min}) }{p_{\max}}.
$$
\end{proof}

\subsection{Main result}

We are ready to prove the main convergence theorem.
\begin{theorem} Consider \Cref{alg:weighted_ef21_pp} (\algname{EF21-W-PP}) applied to the distributed optimization problem~\eqref{eq:main_problem}. Let Assumptions~\ref{as:smooth},~\ref{as:L_i},~\ref{as:lower_bound} hold, assume that $\cC_i^t \in \mathbb{C}(\alpha)$ for all $i \in [n]$ and $t > 0$, set 
\begin{equation*}
G^t \eqdef \sumin w_i \left\|g_i^t - \frac{1}{n w_i} \nabla f_i(x^t)\right\|^2,
\end{equation*}	
where $w_i = \frac{L_i}{\sum_j L_j}$ for all $i \in [n]$, and let the stepsize satisfy
\begin{equation}
\label{eq:ef21-w-pp-gamma}
0 < \gamma \leq \left(L + \sqrt{\frac{\tilde{B}}{\theta_p}} \right)^{-1},
\end{equation}
where $s>0$, $\rho > 0$, and
\begin{eqnarray*}
	\theta(\alpha, s) &\eqdef& 1 - (1 - \alpha)(1 + s)\\
	\beta(\alpha, s) &\eqdef&  (1 - \alpha)(1+s^{-1})\\
	\theta_p &\eqdef& p_{\min}\rho + \theta(\alpha,s) p_{\max} - \rho - (p_{\max}-p_{\min}) \\
	\tilde{B} &\eqdef& \left({\beta(\alpha,s) p_{\max}} + (1-p_{\min}){(1+\rho^{-1})} \right) \LAMsq.
\end{eqnarray*}
Additionally, assume that
$$
\frac{1 + \rho(1 - p_{\min}) + (p_{\max} - p_{\min}) }{p_{\max}}\geq \theta(\alpha, s) > \frac{\rho(1 - p_{\min}) + (p_{\max} - p_{\min}) }{p_{\max}}.
$$
If for $T > 1$ we define $\hat{x}^T$ as an element of the set $\{x^0, x^1, \dots, x^{T-1}\}$ chosen uniformly at random, then
\begin{equation}
\ExpBr{\|\nabla f(\hat{x}^T) \|^2} \leq \frac{2 (f(x^0) - f^\ast)}{\gamma T} + \frac{G^0}{\theta_p T}.
\end{equation}
\end{theorem}

\begin{proof}
Following the same approach employed in the proof for the \algname{SGD} case, we obtain
\begin{eqnarray*}
	\ExpBr{f(x^{t+1})-f^\ast} & \leq &  \ExpBr{f(x^{t})-f^\ast}
	-\frac{\gamma}{2} \ExpBr{\sqnorm{\nabla f(x^{t})}}   \\
&& \qquad 	
	 -\left(\frac{1}{2 \gamma}-\frac{L}{2}\right) \ExpBr{\sqnorm{x^{t+1}-x^{t}}}+ \frac{\gamma}{2}\ExpBr{G^t}.
\end{eqnarray*}

Let $\delta^{t} \eqdef \ExpBr{f(x^{t})-f^\ast}$, $s^{t} \eqdef \ExpBr{G^t }$ and 
$r^{t} \eqdef\ExpBr{\sqnorm{x^{t+1}-x^{t}}}.$ Applying the previous lemma, we obtain:
\begin{eqnarray*}
	\delta^{t+1}+\frac{\gamma}{2 \theta_p} s^{t+1} &\leq& \delta^{t}-\frac{\gamma}{2}\ExpBr{\sqnorm{\nabla f(x^{t})}} -          \left(\frac{1}{2 \gamma}-\frac{L}{2}\right) r^{t}+\frac{\gamma}{2} s^{t}+\frac{\gamma}{2 \theta_p}\left(\tilde{B} r^t + (1 - \theta_p) s^{t}\right) \\
	&=&\delta^{t}+\frac{\gamma}{2\theta} s^{t}-\frac{\gamma}{2}\ExpBr{\sqnorm{\nabla f(x^{t})}}-\underbrace{\left(\frac{1}{2\gamma} -\frac{L}{2} - \frac{\gamma}{2\theta_p} \tilde{B} \right)}_{\geq 0} r^{t} \\
	& \leq& \delta^{t}+\frac{\gamma}{2\theta_p} s^{t} -\frac{\gamma}{2}\ExpBr{\sqnorm{\nabla f(x^{t})}}.
\end{eqnarray*}
By summing up inequalities for $t =0, \ldots, T-1,$ we get
$$
0 \leq \delta^{T}+\frac{\gamma}{2 \theta_p} s^{T} \leq \delta^{0}+\frac{\gamma}{2 \theta_p} s^{0}-\frac{\gamma}       {2} \sum_{t=0}^{T-1} \ExpBr{\sqnorm{\nabla f(x^{t})}}.
$$

Finally, via multiplying both sides by $\frac{2}{\gamma T}$, after rearranging we get:
\begin{eqnarray*}
	\sum_{t=0}^{T-1} \frac{1}{T} \ExpBr{\sqnorm{\nabla f (x^{t})}} \leq \frac{2 \delta^{0}}{\gamma T} + \frac{s^0}{\theta_p T}.
\end{eqnarray*}

It remains to notice that the left-hand side can be interpreted as $\ExpBr{\sqnorm{\nabla f(\hat{x}^{T})} }$, where $\hat{x}^{T}$ is chosen from the set $\{x^{0}, x^{1}, \ldots, x^{T-1}\}$ uniformly at random.
\end{proof}

Our analysis of this extension follows a similar approach as the one used by \citet{EF21BW}, for algorithm they called \algname{EF21-PP}. Presented analysis of \algname{EF21-W-PP} has a better provides a better multiplicative factor in definition of $\widetilde{B} \propto \LAMsq$, where in vanilla \algname{EF21-PP} had $\widetilde{B} \propto \LQMsq$. This fact  improves upper bound on allowable step size in \eqref{eq:ef21-w-pp-gamma}.

\clearpage
\section{Improved Theory for EF21 in the Rare Features Regime} \label{sec:RF}

In this section, we adapt our new results to the {\em rare features} regime proposed and studied by~\citet{EF21-RF}.

\subsection{Algorithm}

In this section, we focus on Algorithm~\ref{alg:EF21_RF}, which is an adaptation of \algname{EF21} (as delineated in \Cref{alg:EF21}) that specifically employs Top$K$ operators. This variant is tailored for the rare features scenario, enhancing the convergence rate by shifting from the average of squared Lipschitz constants to the square of their average. The modifications introduced in Algorithm~\ref{alg:EF21_RF}, compared to the standard \algname{EF21}, are twofold and significant.

Primarily, the algorithm exclusively engages Top$K$ compressors, leveraging the inherent sparsity present in the data. Additionally, the initial gradient estimates $ g_i^0$ are confined to the respective subspaces $\RR^d_i$, as characterized by equation~\eqref{eq:r_d_i_def}. With the exception of these distinct aspects, the algorithm's execution parallels that of the original \algname{EF21}.

\begin{algorithm}
	\begin{algorithmic}[1]
		\STATE {\bfseries Input:} initial model $x^0 \in \RR^d$; initial gradient estimates $g_1^0\in \R^d_1,\dots,g_n^0 \in \R^d_n$ (as defined in equation~\eqref{eq:r_d_i_def}) stored at the server and the clients; stepsize $\gamma>0$; sparsification levels $K_1,\dots,K_n \in [d]$; number of iterations $T > 0$	
		\STATE {\bfseries Initialize:} $g^0 = \avein g_i^0 $ on the server
		\FOR{$t = 0, 1, 2, \dots, T - 1 $}
		\STATE Server computes $x^{t+1} = x^t - \gamma g^t$ and  broadcasts  $x^{t+1}$ to all $n$ clients
		\FOR{$i = 1, \dots, n$ {\bf on the clients in parallel}} 
			\STATE Compute $u_i^t=\text{Top}K_i (\nabla f_i(x^{t+1}) - g_i^t)$ and update $g_i^{t+1} = g_i^t +u_i^t$ \label{line:sparse_g_update_step}		
			\STATE Send the compressed message $u_i^{t}$ to the server
		\ENDFOR 
			\STATE Server updates $g_i^{t+1} = g_i^t +u_i^t$ for all $i\in [n]$, and computes $g^{t+1} = \avein g_i^{t+1}$		 \label{line:sparse_averaging_line}
		\ENDFOR
		\STATE {\bfseries Output:} Point $\hat{x}^T$ chosen from the set $\{x^0, \dots, x^{T-1}\}$ uniformly at random
	\end{algorithmic}
	\caption{\algname{EF21}: Error Feedback 2021 with Top$K$ compressors}
	\label{alg:EF21_RF}
\end{algorithm}

\subsection{New sparsity measure}

To extend our results to the {\em rare features} regime, we need to slightly change the definition of the parameter $c$ in the original paper. The way we do it is unrolled as follows. First, we recall the following definitions from~\citep{EF21-RF}:
\begin{equation}
	\cZ \eqdef \{ (i, j)\in [n] \times [d]\; | \;  [\nabla f_i(x)]_j = 0 \ \forall x \in \RR^d \},
\end{equation}
and
\begin{equation}
	\cI_j \eqdef \{ i \in [n]\; | \;   (i, j) \notin \cZ\}, \qquad \cJ_i \eqdef \{ j \in [d]\; | \;   (i, j) \notin \cZ\}.
\end{equation}

We also need the following definition of $\RR^d_i$:
\begin{equation}\label{eq:r_d_i_def}
\RR^d_i \eqdef \{u = (u_1, \dots, u_d) \in \RR^d : u_j = 0 \text{ whenever } (i,j) \in \cZ \}.
\end{equation}

Now we are ready for a new definition of the sparsity parameter $c$:
\begin{equation}\label{eq:sparse_c_definiton}
	c \eqdef n \cdot \max\limits_{j \in [d]} \sum\limits_{i \in \cI_j} w_i,
\end{equation}
where $w_i$ is defined as in~\eqref{eq:weight_definition}. We note that $c$ recovers the standard definition from \citep{EF21-RF} when $w_i = \frac{1}{n}$ for all $i\in [n]$.

\subsection{Lemmas}

We will proceed through several lemmas.

\begin{lemma}\label{lm:sparse_tighter_bound}
	Let $u_i \in \RR^d_i$ for all $i \in [n]$. Then, the following inequality holds:
	\begin{equation}\label{eq:sparse_tighter_bound}
		\left\|\sumin w_i u_i \right\|^2 \leq \frac{c}{n} \sumin w_i \|u_i\|^2.
	\end{equation}
\end{lemma}
\begin{proof}
	Initially, we observe that
	\begin{align}\label{eq:sparse_aux_1}
		\left\|\sumin w_i u_i \right\|^2 = \sum\limits_{j=1}^d \left(\sumin w_i u_{ij} \right)^2.
	\end{align}
	We note that for any $j \in [d]$ it holds that
	\begin{equation}\label{eq:sparse_aux_2}
		\begin{aligned}
			\left(\sumin w_i u_{ij} \right)^2 &= \left(\sum\limits_{i \in \cI_j} w_i u_{ij} \right)^2 \\
			&= \left(\sum\limits_{i \in \cI_j} w_i \right)^2 \left(\sum\limits_{i \in \cI_j} \frac{w_i}{\sum\limits_{i' \in \cI_j} w_{i'}} u_{ij} \right)^2  \leq \left(\sum\limits_{i \in \cI_j} w_i \right)^2  \sum\limits_{i \in \cI_j} \frac{w_i}{\sum\limits_{i' \in \cI_j} w_{i'}} u_{ij}^2,
		\end{aligned}
	\end{equation}
	where on the last line we used the Jensen's inequality. Subsequent arithmetic manipulations and the incorporation of definition~\eqref{eq:sparse_c_definiton} yield:
	\begin{eqnarray}
			\left(\sumin w_i u_{ij} \right)^2 & \overset{\eqref{eq:sparse_aux_2}}{\leq}  & \left(\sum\limits_{i \in \cI_j} w_i \right) \cdot \sum\limits_{i \in \cI_j} w_i u_{ij}^2  \notag \\
			& = & \left(\sum\limits_{i \in \cI_j} w_i \right) \cdot \sumin w_i u_{ij}^2  \notag \\
			& \leq & \left[\max_{j \in [d]}\left(\sum\limits_{i \in \cI_j} w_i \right)\right] \cdot \sumin w_i u_{ij}^2 \notag  \\
			& \overset{\eqref{eq:sparse_c_definiton}}{=} & \frac{c}{n} \sumin w_i u_{ij}^2.
			\label{eq:sparse_aux_3}
	\end{eqnarray}
	
	Substituting Equation~\eqref{eq:sparse_aux_3} into Equation~\eqref{eq:sparse_aux_1} completes the proof.
\end{proof}

\begin{lemma}
	Assume that $g_i^0 \in \RR^d_i$ for all $i \in [n]$. Then, it holds for all $t>0$ that
	\begin{equation}\label{eq:grad_est_distortion}
		\| g^t - \nabla f(x^t)\|^2 \leq \frac{c}{n} G^t.
	\end{equation}
\end{lemma}

\begin{proof}
	By Lemma 8 in~\cite{EF21-RF}, for \algname{EF21} the update $g_i^t$ stays in $\RR^d_i$ if $g_i^0 \in \RR^d_i$. We then proceed as follows:
	\begin{equation}\label{eq:sparse_aux_4}
		\begin{aligned}
			\|g^t - \nabla f(x^t) \|^2 \overset{\eqref{line:sparse_averaging_line}}{=}\left\|\avein g_i^t - \nabla f_i(x^t) \right\|^2 = \left\| \sumin w_i \left[\frac{1}{nw_i} (g_i^t - \nabla f_i(x^t)) \right] \right\|^2.
		\end{aligned}
	\end{equation}
	Since $g_i^t \in \RR^d_i$, as was noted, and $\nabla f_i(x^t) \in \RR^d_i$, by the definition of $\RR^d_i$, then $\frac{1}{nw_i} (g_i^t - \nabla f_i(x^t))$ also belongs to $\RR^d_i$. By \Cref{lm:sparse_tighter_bound}, we further proceed:
	\begin{eqnarray*}
			\|g^t - \nabla f(x^t) \|^2 &\overset{\eqref{eq:sparse_aux_4}} {=} &  \left\| \sumin w_i \left[\frac{1}{nw_i} (g_i^t - \nabla f_i(x^t)) \right] \right\|^2 \\
			& \overset{\eqref{eq:sparse_tighter_bound}}{\leq} & \frac{c}{n} \sumin w_i \left\| \frac{1}{nw_i} (g_i^t - \nabla f_i(x^t))\right\|^2\\
			& = & \frac{c}{n} \sumin \frac{1}{n^2 w_i} \| g_i^t - \nabla f_i(x^t)\|^2 \quad = \quad \frac{c}{n} G^t,
	\end{eqnarray*}
	which completes the proof.
\end{proof}

For the convenience of the reader, we briefly revisit Lemma 6 from~\citep{EF21-RF}.
\begin{lemma}[Lemma 6 from~\cite{EF21-RF}]
	If \Cref{as:L_i} holds, then for $i \in [n]$, we have
	\begin{equation}\label{eq:sparse_smoothness}
		\sum\limits_{j: (i, j) \notin \cZ} ((\nabla f_i(x))_j - (\nabla f_i(y))_j )^2 \leq L_i^2 \sum\limits_{j: (i, j) \notin \cZ}  (x_j - y_j)^2 \quad \forall x, y \in \RR^d.
	\end{equation}
\end{lemma}

Now, we proceed to the following lemma, which aims to provide a tighter bound for the quantity $\sumin \frac{1}{L_i} \|\nabla f_i(x) - \nabla f_i(y)\|^2$.

\begin{lemma}\label{lm:sparse_L+}
	If \Cref{as:L_i} holds, then
	\begin{equation}\label{eq:sparse_L_+}
		\sumin \frac{1}{L_i} \|\nabla f_i(x) - \nabla f_i(y) \|^2 \leq c \LAM \|x - y\|^2.
	\end{equation}
\end{lemma}
\begin{proof}
	The proof commences as follows:
	\begin{eqnarray}
			\sumin \frac{1}{L_i} \|\nabla f_i(x) - \nabla f_i(y) \|^2 &= & \sumin \frac{1}{L_i} \sum\limits_{j: (i, j) \notin \cZ} ((\nabla f_i(x))_j - (\nabla f_i(y))_j )^2 \notag \\
			& \overset{\eqref{eq:sparse_smoothness}}{\leq}  & \sumin \frac{1}{L_i}  L_i^2 \sum\limits_{j: (i, j) \notin \cZ}  (x_j - y_j)^2 \notag  \\
			& =  & \sumin \sum\limits_{j: (i, j) \notin \cZ}  L_i (x_j - y_j)^2 \notag  \\
			& = & \sum\limits_{j=1}^d  \sum\limits_{i: (i, j) \notin \cZ} L_i (x_j - y_j)^2  \quad = \quad \sum\limits_{j=1}^d \left[(x_j - y_j)^2 \sum\limits_{i: (i, j) \notin \cZ} L_i \right]. \label{eq:sparse_aux_5}
	\end{eqnarray}
	To advance our derivations, we consider the maximum value over $\sum\limits_{i: (i, j) \notin \cZ} L_i$:
	\begin{eqnarray*}
			\sumin \frac{1}{L_i} \|\nabla f_i(x) - \nabla f_i(y) \|^2 &\overset{\eqref{eq:sparse_aux_5}}{\leq} & \sum\limits_{j=1}^d \left[(x_j - y_j)^2 \sum\limits_{i: (i, j) \notin \cZ} L_i \right] \notag \\
			& \leq  & \sum\limits_{j=1}^d \left[(x_j - y_j)^2 \max\limits_{j \in [d]} \sum\limits_{i: (i, j) \notin \cZ} L_i \right] \notag \\
			& = & \left[\max\limits_{j \in [d]} \sum\limits_{i: (i, j) \notin \cZ} L_i \right] \sum\limits_{j=1}^d (x_j - y_j)^2 \notag \\
			& =   & \left[\max\limits_{j \in [d]} \sum\limits_{i: (i, j) \notin \cZ} L_i \right] \|x - y\|^2 \notag \\
			& =  & \left[\max\limits_{j \in [d]} \sum\limits_{i \in \cI_j} L_i \right] \|x - y\|^2 \quad \overset{\eqref{eq:sparse_c_definiton}}{=}  \quad c \LAM \|x - y\|^2,
	\end{eqnarray*}
	what completes the proof.
\end{proof}

For clarity and easy reference, we recapitulate Lemma 10 from~\cite{EF21-RF}.
\begin{lemma}[Lemma 10 from~\cite{EF21-RF}]\label{lm:sparse_grad_est_evolution}
	The iterates of  \Cref{alg:EF21_RF}  method satisfy 
	\begin{equation}\label{eq:sparse_grad_est_evolution}
		\left\|g_i^{t+1} - \nabla f_i(x^{t+1})\right\|^2 \leq (1 - \theta(\alpha)) \left\|g_i^t - \nabla f_i(x^t)\right\|^2 + \beta(\alpha) \|\nabla f_i(x^{t+1}) - \nabla f_i(x^t)\|^2,
	\end{equation}
	where  $\alpha = \min\left\{\min\limits_{i \in [n]} \frac{K_i}{|\cJ_i|}, 1 \right\}$.
\end{lemma}

\begin{lemma}
	Under \Cref{as:L_i}, iterates of \Cref{alg:EF21_RF} satisfies 
	\begin{equation}
		G^{t+1} \leq (1 - \theta(\alpha)) G^t + \beta(\alpha) \frac{c}{n} \LAMsq \|x - y\|^2.
	\end{equation}
\end{lemma}
\begin{proof}
	The proof is a combination of Lemmas~\ref{lm:sparse_L+} and~\ref{lm:sparse_grad_est_evolution}:
	\begin{eqnarray*}
			G^{t+1} &\overset{\eqref{eq:new_def_distortion}}{=}  & \frac{1}{n^2} \sumin \frac{1}{w_i} \left\|g_i^{t+1} - \nabla f_i(x^{t+1})\right\|^2 \\
			& \overset{\eqref{eq:sparse_grad_est_evolution}}{\leq}  & \frac{1}{n^2} \sumin \frac{1}{w_i} \left[ (1 - \theta(\alpha)) \|g_i^t - \nabla f_i(x^t)\|^2 + \beta(\alpha) \|\nabla f_i(x^{t+1}) - \nabla f_i(x^t) \|^2 \right] \\
			& =  & (1 - \theta(\alpha)) G^t + \frac{\beta(\alpha)}{n^2} \sumin \frac{1}{w_i} \|\nabla f_i(x^{t+1}) - \nabla f_i(x^t) \|^2 \\
			& \overset{\eqref{eq:weight_definition}}{=}  & (1 - \theta(\alpha)) G^t + \frac{\beta(\alpha)}{n^2}  \left(\sum_{j=1}^n L_j\right)\sumin \frac{1}{L_i} \|\nabla f_i(x^{t+1}) - \nabla f_i(x^t) \|^2 \\
			& \overset{\eqref{eq:sparse_L_+}}{\leq}  & (1 - \theta(\alpha)) G^t + \frac{\beta(\alpha)}{n^2} \frac{c}{n} \left( \sum_{j=1}^n L_j \right)^2\|x - y\|^2\\
			& = & (1 - \theta (\alpha) ) G^t + \beta(\alpha) \frac{c}{n} \LAMsq \|x - y\|^2.
	\end{eqnarray*}
\end{proof}

\subsection{Main result}

And now we are ready to formulate the main result.
\begin{theorem}
	Let Assumptions~\ref{as:smooth},~\ref{as:L_i} and~\ref{as:lower_bound} hold. Let $g_i^0 \in \RR^d_i$ for all $i \in [n]$, $$\alpha = \min\left\{\min\limits_{i \in [n]} \frac{K_i}{|\cJ_i|}, 1 \right\}, \qquad 0<\gamma \leq \frac{1}{L + \frac{c}{n} \LAM \xi(\alpha) }.$$ Under these conditions, the iterates of \Cref{alg:EF21_RF} satisfy
	\begin{equation}
		\frac{1}{T} \sum\limits_{t=0}^{T-1} \|\nabla f(x^t)\|^2 \leq \frac{2 (f(x^0) - f^\ast) }{\gamma T} + \frac{c}{n} \frac{G^0}{\theta(\alpha) T}.
	\end{equation}
\end{theorem}
\begin{proof}
	Let us define the Lyapunov function: 
	\begin{equation}\label{eq:sparse_Lyapunov}
		\Psi^t \eqdef f(x^t) - f^\ast + \frac{\gamma c}{2\theta n} G^t.
	\end{equation}
	We start the proof as follows:
	\begin{eqnarray*}
			\Psi^{t+1} &\overset{\eqref{eq:sparse_Lyapunov}}{=}  & f(x^{t+1}) - f^\ast + \frac{\gamma c}{2\theta n} G^{t+1} \\
			& \overset{\eqref{eq:descent_lemma}}{\leq}  & f(x^t) - f^\ast - \frac{\gamma}{2} \|\nabla f(x^t)\|^2 - \left(\frac{1}{2\gamma} - \frac{L}{2}\right) \|x^{t+1} - x^t\|^2  + \frac{\gamma}{2} \|g^t - \nabla f(x^t)\|^2 +  \frac{\gamma c}{2\theta n} G^{t+1} \\
			& \overset{\eqref{eq:grad_est_distortion}}{\leq} &  f(x^t) - f^\ast - \frac{\gamma}{2} \|\nabla f(x^t)\|^2 - \left(\frac{1}{2\gamma} - \frac{L}{2}\right) \|x^{t+1} - x^t\|^2  + \frac{\gamma}{2}\frac{c}{n} G^t +  \frac{\gamma c}{2\theta n} G^{t+1} \\
			& \overset{\eqref{eq:sparse_grad_est_evolution}}{\leq}  & f(x^t) - f^\ast - \frac{\gamma}{2} \|\nabla f(x^t)\|^2 - \left(\frac{1}{2\gamma} - \frac{L}{2}\right) \|x^{t+1} - x^t\|^2 \\
			& &  \quad + \frac{\gamma}{2}\frac{c}{n} G^t +  \frac{\gamma c}{2\theta n} \left((1 - \theta) G^t  + \beta\frac{c}{n} \LAMsq \|x^{t+1} - x^t\|^2 \right) \\
			& =  & f(x^t) - f^\ast + \frac{\gamma c}{2\theta n} G^t - \frac{\gamma}{2} \|\nabla f(x^t)\|^2 - \left(\frac{1}{2\gamma} - \frac{L}{2} - \frac{\gamma}{2} \frac{\beta}{\theta} \frac{c^2}{n^2} \cdot \LAMsq \right) \|x^{t+1} - x^t\|^2\\
			& =  & \Psi^t - \frac{\gamma}{2} \|\nabla f(x^t)\|^2 - \underbrace{\left(\frac{1}{2\gamma} - \frac{L}{2} - \frac{\gamma}{2} \frac{\beta}{\theta} \frac{c^2}{n^2} \cdot \LAMsq \right)}_{\geq 0} \|x^{t+1} - x^t\|^2 \\
			& \leq & \Psi^t - \frac{\gamma}{2} \|\nabla f(x^t)\|^2.
	\end{eqnarray*}

Unrolling the inequality above, we get
\begin{align*}
	0 \leq \Psi^T \leq \Psi^{T-1} - \frac{\gamma}{2} \|\nabla f(x^{T-1})\|^2 \leq \Psi^0 - \frac{\gamma}{2}\sum\limits_{t=0}^{T-1} \|\nabla f(x^t) \|^2,
\end{align*}
from what the main result follows.
\end{proof}

\clearpage
\section{Experiments: Further Details}
\label{app:exp-additional-details-main-part}

\begin{center}	
	\begin{figure*}[h]
		\centering
		\captionsetup[sub]{font=normalsize,labelfont={}}	
		\captionsetup[subfigure]{labelformat=empty}		
		
		\begin{subfigure}[ht]{0.5\textwidth}
			\includegraphics[width=\textwidth]{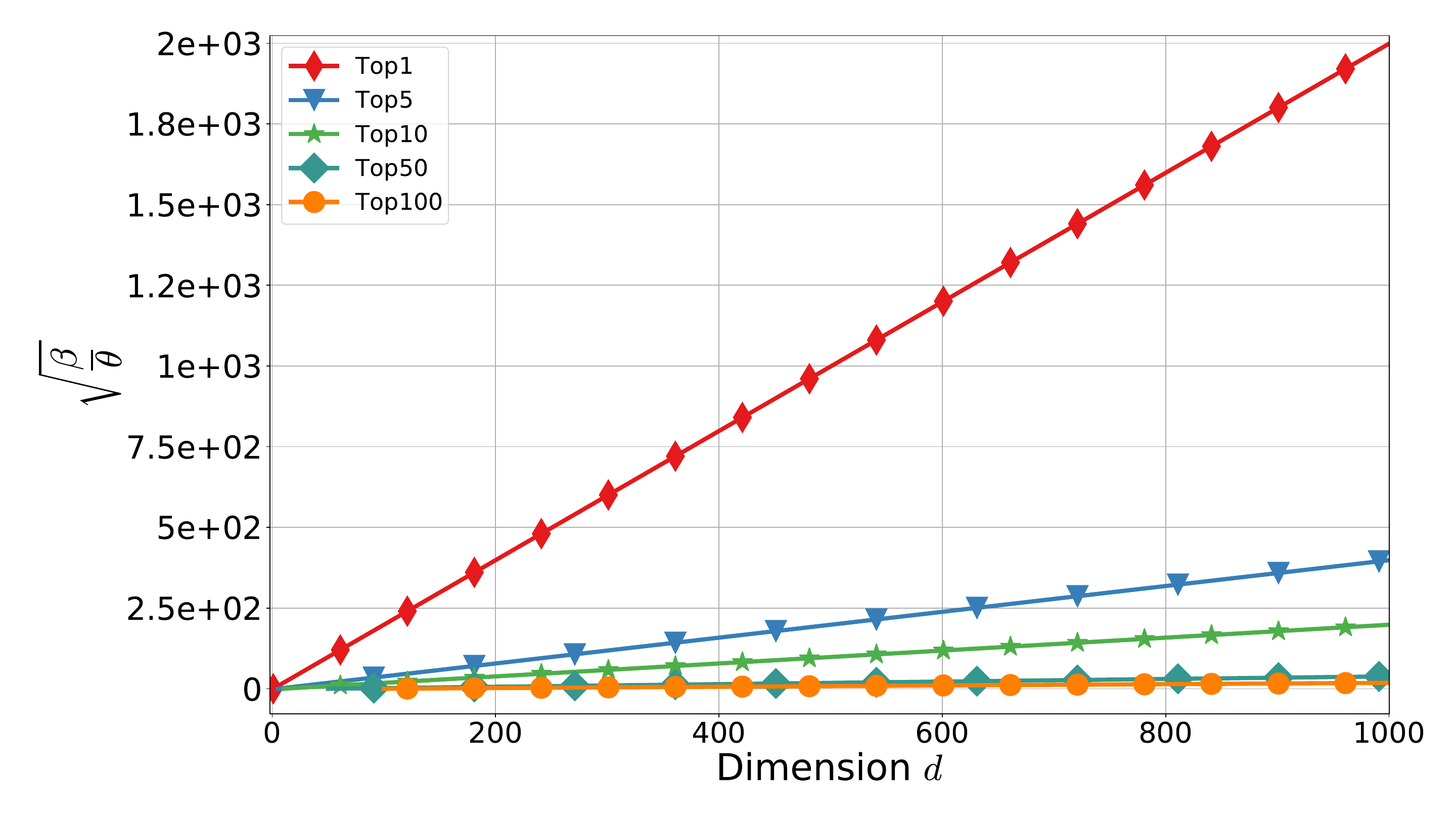} \caption{}
		\end{subfigure}
		
		\caption{\small{The factor $\xi=\sqrt{{\beta}/{\theta}}$ as a function of optimization variable dimension $d$ for several \algnamesmall{TopK} compressors. The behavior is independent of properties of $\{f_1(x),\dots,f_n(x)\}$ and $f(x)$.}}
		\label{fig:beta-over-theta-theoretical}
	\end{figure*}
\end{center}

\subsection{Computing and software environment}
\label{app:compute-env}

We used the Python software suite \texttt{FL\_PyTorch} \citep{burlachenko2021fl_pytorch} to simulate the distributed environment for training. We carried out experiments on a compute node with \texttt{Ubuntu 18.04 LTS}, $256$ GBytes of DRAM DDR4 memory at $2.9$GHz, and $48$ cores ($2$ sockets with $24$ cores per socket) of {Intel(R) Xeon(R) Gold 6246 CPU at $3.3$GHz}. We used double-precision arithmetic during computing gradient oracles. All our computations were carried on \texttt{CPU}.

\subsection{Comments on the improvement}
\label{app:practical-role}

The standard \algname{EF21} analysis \citep{EF21} allows to utilize \algname{EF21} with maximum allowable step size $\gamma$ equal to:
\begin{eqnarray*}
	\label{frm:ef21}
	\gamma = \left( L + \LQM \sqrt{\frac{\beta(\alpha)}{\theta(\alpha)}} \right)^{-1}, \qquad  \theta(\alpha) = 1 - \sqrt{1-\alpha}, \qquad \beta(\alpha)=\frac{1-\alpha}{1 - \sqrt{1-\alpha}}.
\end{eqnarray*}

Our analysis allows us to replace the quantity $\LQM$ with $\LAM$. This improvement has an important consequence. The replaced quantity affects the step size by a factor of $\xi(\alpha)=\sqrt{\frac{\beta(\alpha)}{\theta(\alpha)}}$. This factor can be arbitrarily large as $d$ increases, as shown in Figure \ref{fig:beta-over-theta-theoretical}. If $d$ is increasing and the parameter $k$ of \algname{TopK} compressor is fixed, then even a small improvement in the constant term can have a significant impact in an absolute sense to the computed step size if $\xi(\alpha) \gg L$.

\subsection{When improved analysis leads to more aggressive steps}
\label{app:when-improved-analysis-better}

The quantity $\LQM \eqdef \sqrt{\frac{1}{n} \sum_{i=1}^n L_i ^2}$ plays essential role in \algname{EF21} analysis. As we saw with special consideration this quantity for \algname{EF21} and its extensions is improvable. The improved analysis allows us to replace it with $\LAM \eqdef \frac{1}{n} \sum_{i=1}^n L_i.$ Clearly, by the arithmetic-quadratic mean inequality,  $$\Lvar \eqdef \LQMsq - \LAMsq  \ge 0.$$ The difference $\LQM - \LAM$ can be expressed as follows:

\begin{eqnarray*}
	\LQM - \LAM &=&  (\LQM - \LAM)\left(\frac{\LQM + \LAM}{\LQM + \LAM}\right) \\
	&=& \frac{\LQMsq - \LAMsq}{\LQM + \LAM} \quad = \quad \frac{1}{\LQM + \LAM} \cdot {\frac{1}{n} \sum_{i=1}^{n} \left(L_i - \frac{1}{n} \sum_{i=1}^{n} L_i\right) ^2}.
\end{eqnarray*}

The coefficient $\frac{1}{\LQM+\LAM}$ in the last equation can be bounded from below and above as follows:
$$\frac{1}{2 \LQM} = \frac{1}{2 \cdot \max(\LQM, \LAM)} \le \frac{1}{\LQM+\LAM} \le \frac{1}{2 \cdot \min(\LQM, \LAM)} \le \frac{1}{2\LAM}.$$

As a consequence, difference $\LQM - \LAM$ is bound above by the estimated variance of $L_i$ divided by the mean of $L_i$, also known as \textit{Index of Dispersion} in statistics. From this consideration, we can more easily observe that \algname{EF21-W} can have an arbitrarily better stepsize than vanilla \algname{EF21} if the variance of $L_i$ is increasing faster than the mean of $L_i$.

\subsection{Dataset generation for synthetic experiment}
\label{app:dataset-gen-synthetic}

First, we assume that the user provides two parameters: $\mu \in \mathbb{R}_+$ and $L \in \mathbb{R}_+$. These parameters define the construction of strongly convex function $f_i(x)$, which are modified by meta-parameters $q \in [-1,1]$ and $z > 0$, described next.
\begin{enumerate}
\item  Each client initially has
$$f_i(x) \eqdef \frac{1}{n_i} \norm{\bA_i x - {b_i}}^2,$$ where $\bA_i$ is initialized in such way that $f_i$ is $L_{i}$ smooth and $\mu_{f_i}$ strongly convex. Parameters are defined in the following way: 
\[L_{i}=\frac{i}{n} \cdot (L-\mu) + \mu, \qquad \mu_{f_i} = \mu.\]

\item The scalar value $q \in [-1,+1]$ informally plays the role of meta parameter to change the distribution of $L_{i}$ and make values of $L_{i}$ close to one of the following: (i) $\mu$; (ii) $L$; (iii) $(L+\mu)/2$. The exact modification of $L_{i}$ depends on the sign of meta parameter $q$.

\begin{itemize}
\item Case $q \in [0,1]$. In this case for first $n/2$ (i.e., $i \in [0, n/2]$) compute the value $L_{i,q} = \mathrm{lerp}(L_{i}, \mu, q)$, where $\mathrm{lerp}(a,b,t):\RD \times \RD \times [0,1] \to \RD$ is standard linear interpolation	
\[
\mathrm{lerp}(a, b, t) = a (1-t) + bt.
\]

The last $n/2$ ($i \in [n/2+1, n]$) compute the value $L_{i,q} = \mathrm{lerp}(L_{i}, L, q)$. For example, if $q=0$ then $L_{i,q}=L_{i}, \forall i \in [n]$, and if $q=1$ then $L_{i,q} = \mu$ for first ${n}/{2}$ clients and $L_{i,q} = L$ for last ${n}/{2}$ clients.

\item  Case $q \in [-1,0]$. In this for all $n$ clients the new value $L_{i,q} = \mathrm{lerp}(L_{i}, ({L+\mu})/{2}, -q)$. In this case for example if $q=0$ then $L_{i,q}=L_{i}$ and if $q=-1$ then $L_{i,q} = ({L+\mu})/{2}$.

\end{itemize}

The process firstly fills the $\bA_i$ in such form that $L_{i}$ forms a uniform spectrum in $[\mu, L]$ with the center of this spectrum equal to $a=\frac{L+\mu}{2}$. And then as $q \to 1$, the variance of $L_{i,q}$ is increasing. 

\item  We use these new values  $L_{i,q}$ for all $i \in [n]$ clients as a final target $L_{i}^{\mathrm{new}}$ values. Due to numerical issues, we found that it's worthwhile for the first and last client to add extra scaling. First client scales $L_{1,q}$ by 
factor $1/z$, and last $n$-th client scales $L_{n,q}$ by factor $z$. Here $z \ge 0$ is an additional meta-parameter.

\item Next obtained values are used to generate $\bA_i$ in such way that $\nabla^2 f_i(x)$ has uniform spectrum in $[\mu_{f_i,q,z}, L_{i,q,z}]$.

\item  As a last step the objective function $f(x)$ is scaled in such a way that it is $L$ smooth with constant value $L$. The $b_i$ for each client is initialized as $b_i \eqdef \bA_i \cdot x_{\mathrm{solution}}$, where $x_{\mathrm{solution}}$ is fixed solution.
\end{enumerate}

\subsection{Dataset shuffling strategy for LIBSVM dataset}
\label{app:dataset-shuffling-for-libsvm}

Our dataset shuffling strategy heuristically splits data points so that $\Lvar$ is maximized. It consists of the following steps:
\begin{enumerate}
\item  {\bf Sort data points from the whole dataset according to $L$ constants.} Sort all data points according to the smoothness constants of the loss function for each single data point.
\item  {\bf Assign a single data point to each client} Assume that there are total $m$ data points in the datasets, and the total number of clients is $n$. At the beginning each client $i$ holds a single data point $\lfloor (i-1 + 1/2) \cdot \frac{m}{n}\rfloor$.
\item  {\bf Pass through all points.} Initialize set $F=\{\}$. Next, we pass through all points except those assigned from the previous step. For each point we find the best client $i'\in [n] \backslash F$ to assign the point in a way that assignment of point to client $i'$ maximize $\Lvar \eqdef \LQMsq - \LAMsq$. Once the client already has $\lceil \frac{m}{n} \rceil$ data points assigned to it, the client is added to the set $F$. 
\end{enumerate}

The set $F$ in the last step guarantees each client will have $\frac{m}{n}$ data points. In general, this is a heuristic greedy strategy that approximately maximizes $\Lvar$ under the constraint that each client has the same amount of data points equal to $\lfloor \frac{m}{n} \rfloor$. Due to its heuristic nature, the Algorithm does not provide deep guarantees, but it was good enough for our experiments.

\clearpage
\section{Additional Experiments}
\label{app:exp-additional-experiments}

In this section, we present additional experiments for comparison \algname{EF21-W}, \algname{EF21-W-PP}, \algname{EF21-W-SGD} with their vanilla versions. We applied these algorithms in a series of synthetically generated convex and non-convex optimization problems and for training logistic regression with non-convex regularized with using several \texttt{LIBSVM} datasets \citep{chang2011libsvm}. While carrying out additional experiments we will use three quantities. These quantities have already been mentioned in the main part, but we will repeat them here:
$$\LQM \eqdef {\color{red} \sqrt{\avein L_i^2 }}, \quad  \LAM \eqdef {\color{blue} \avein L_i } , \quad \Lvar \eqdef \LQMsq - \LAMsq ={\frac{1}{n} \sum_{i=1}^{n} \left(L_i - \frac{1}{n} \sum_{i=1}^{n} L_i\right) ^2}.$$

The relationship between these quantities was discussed in Appendix~\ref{app:when-improved-analysis-better}. In our experiments we used \algname{TopK} compressor. The \algname{TopK} compressor returns sparse vectors filled with zeros, except $K$ positions, which correspond to $K$ maximum values in absolute value and which are unchanged by the compressor. Even if this compressor breaks ties arbitrarily, it is possible to show that $\alpha = \frac{K}{d}$.
The compressor parameter $\alpha$ is defined without considering properties of $f_i$. The quantities $\beta$, $\theta$, $\frac{\beta}{\theta}$ are derived from $\alpha$, and they do not depend on $L_{i}$.

\subsection{Additional experiments for \algname{EF21}}

\begin{center}	
	\begin{figure*}[t]
		\centering
		\captionsetup[sub]{font=normalsize,labelfont={}}	
		\captionsetup[subfigure]{labelformat=empty}
		
		\begin{subfigure}[ht]{0.3\textwidth}
			\includegraphics[width=\textwidth]{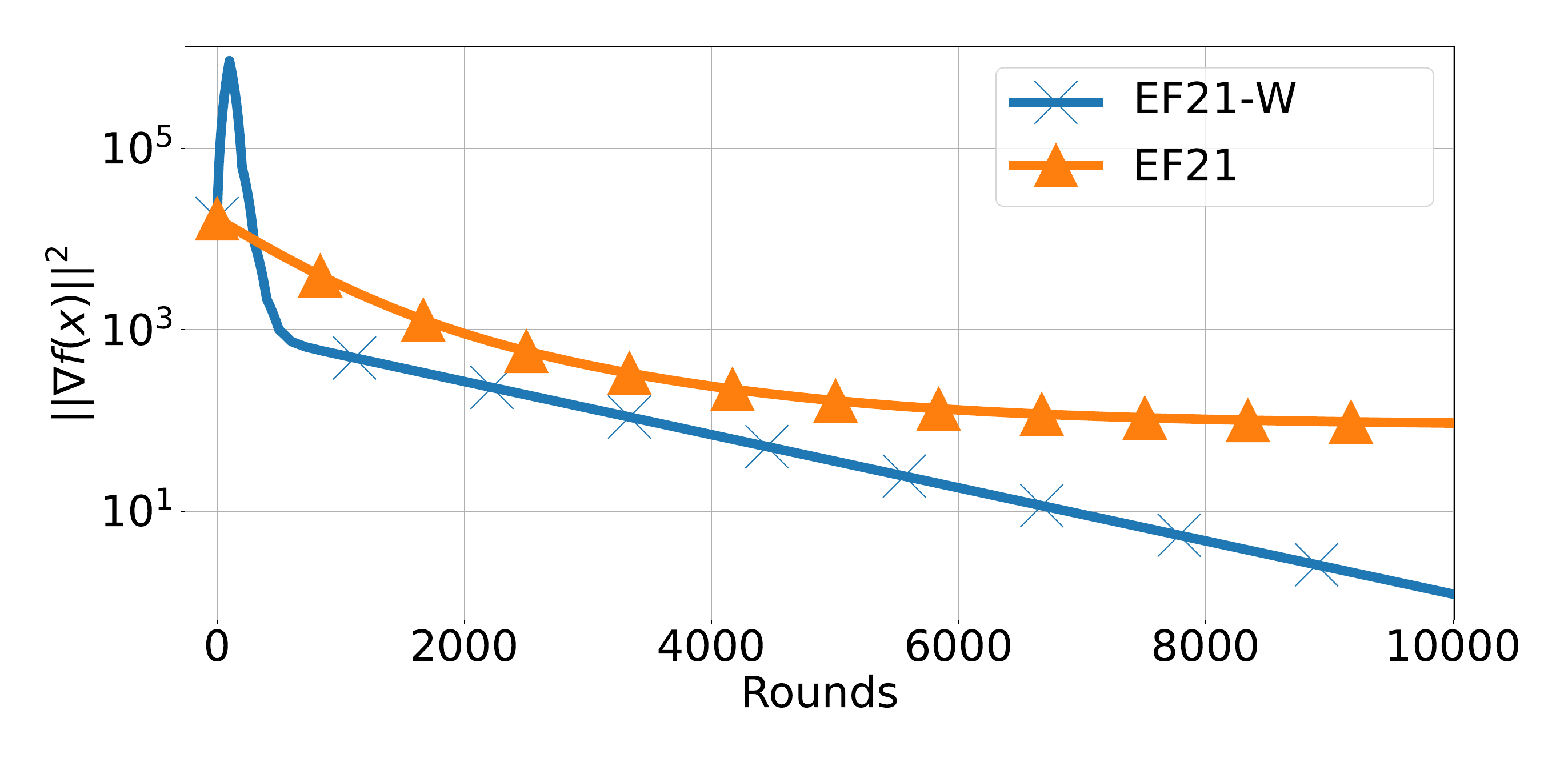} 
			\caption{(a) $\Lvar\approx 4.45 \times 10^6$}
		\end{subfigure}		
		\begin{subfigure}[ht]{0.3\textwidth}
			\includegraphics[width=\textwidth]{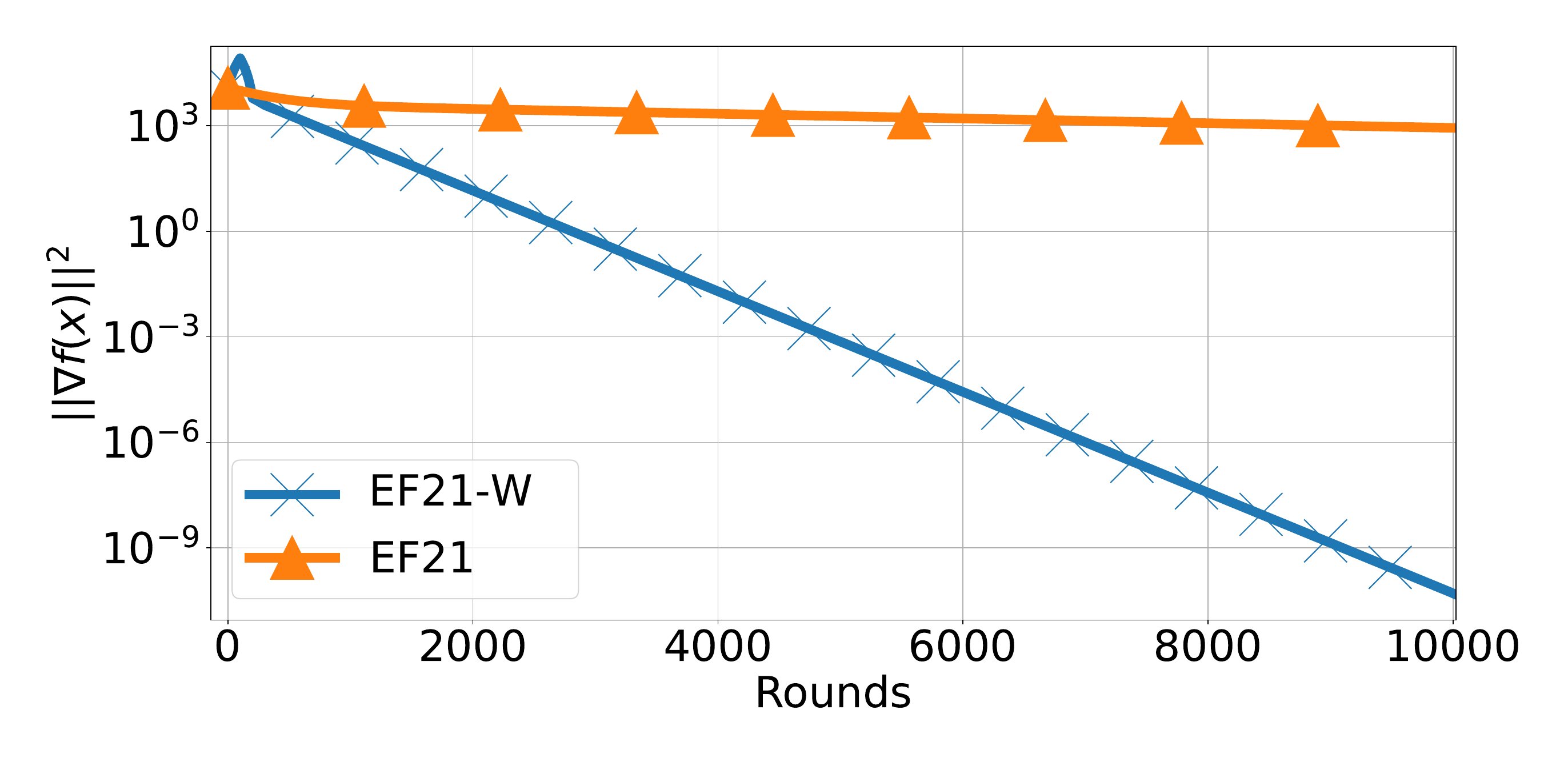}
			\caption{(b) $\Lvar\approx 1.97 \times 10^6$}
		\end{subfigure}
		
		\begin{subfigure}[ht]{0.3\textwidth}
			\includegraphics[width=\textwidth]{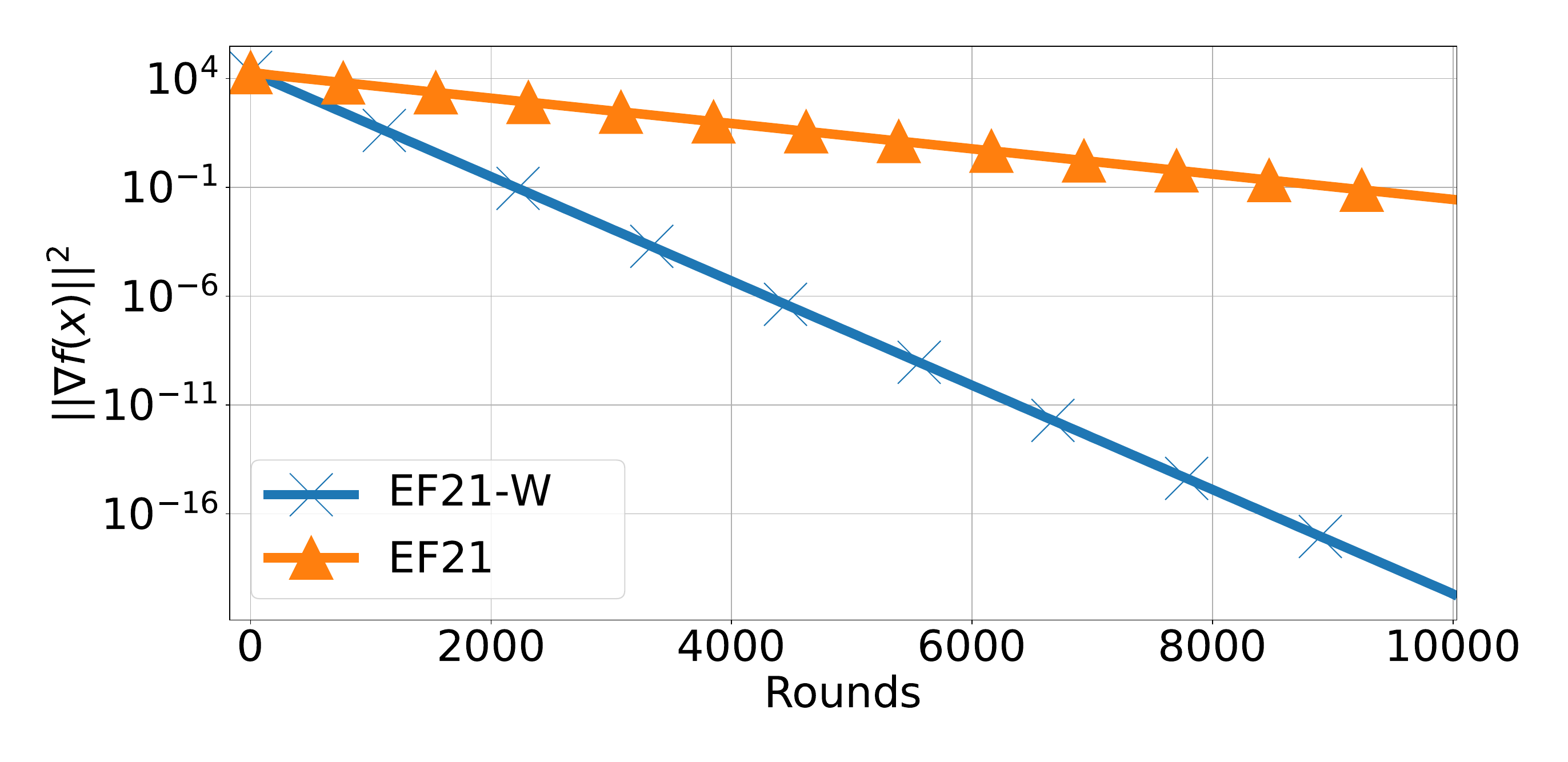}
			\caption{(c) $\Lvar\approx 1.08 \times 10^5$}
		\end{subfigure}		
		\begin{subfigure}[ht]{0.3\textwidth}
			\includegraphics[width=\textwidth]{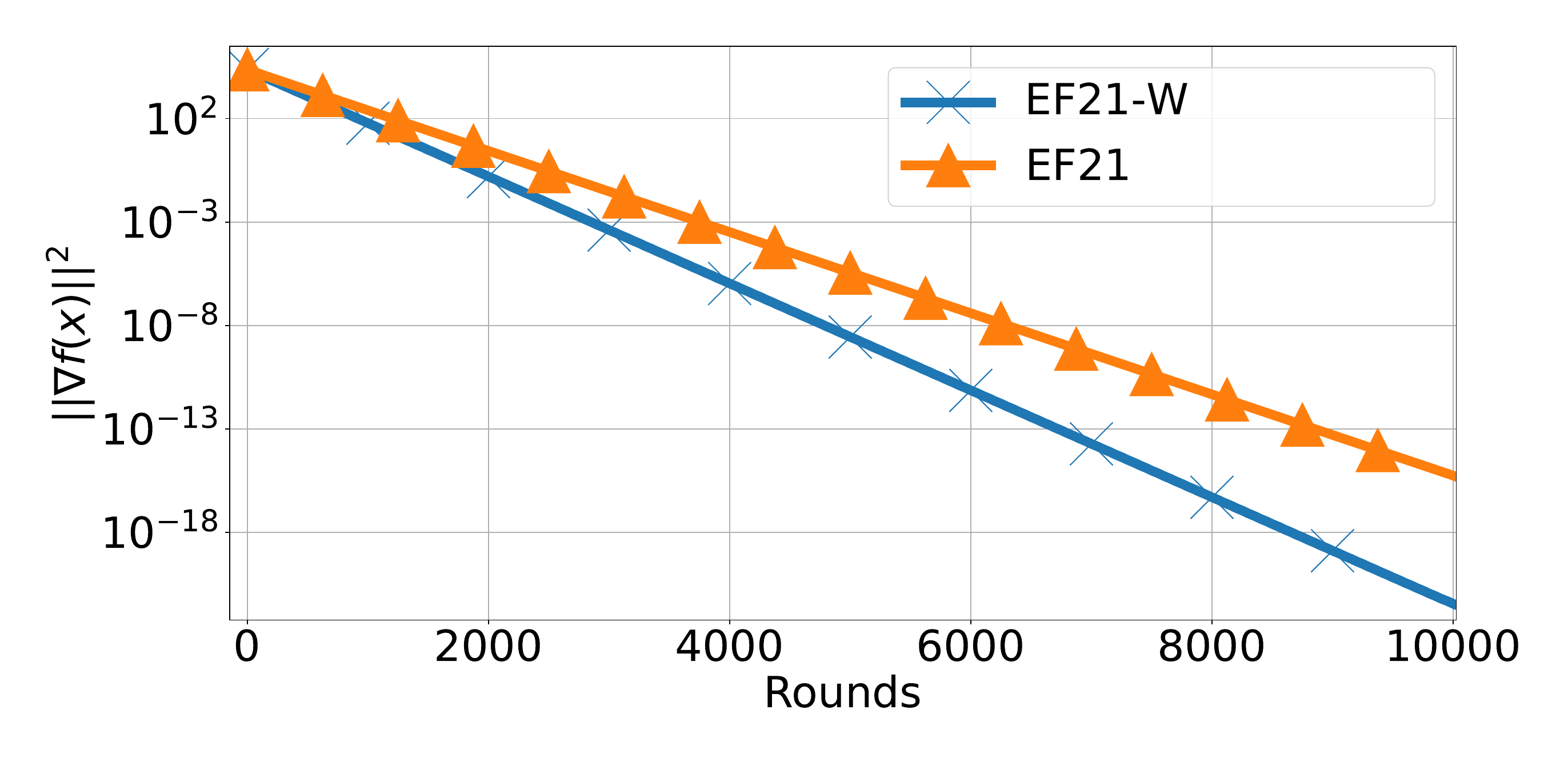}
			\caption{(d) $\Lvar\approx 5.42 \times 10^3$}			
		\end{subfigure}
		
		\caption{\small{Convex smooth optimization. \algnamesmall{EF21} and \algnamesmall{EF21-W} with \algnamesmall{Top1} client compressor, $n=2\,000$, $d=10$. The objective function is constitute of $f_i(x)$ defined in Eq.\eqref{eq:linreg-cvx}. Regularization term $\lambda \frac{\|x\|^2}{2}$, where $\lambda=0.01$. Theoretical step size. Full participation. Extra details are in Table \ref{tbl:app-syn-ef21-cvx}.}}
		\label{fig:app-syn-ef21-cvx}
	\end{figure*}
\end{center}

\paragraph{Convex case with synthetic datasets.} 
We aim to solve optimization problem  \eqref{eq:main_problem} in the case when  the functions $f_1,\dots,f_n$  are strongly convex. In particular, we work with
\begin{eqnarray}
	\label{eq:linreg-cvx}
	f_i(x) \eqdef \frac{1}{n_i} \norm{\bA_i x - {b_i}}^2 + \frac{\lambda}{2} \|x\|^2,
\end{eqnarray}
where $\lambda =0.01$. It can be shown that $L_{i} = \frac{2}{n_i} \lambda_{\max}({\bA_i}^\top \bA_i) + \lambda$. The result of experiments for training linear regression model with a convex regularized is presented in Figure \ref{fig:app-syn-ef21-cvx}. The total number of rounds for simulation is $r=10,000$.  Instances of optimization problems were generated for values $L=50$, $\mu = 1$ with several values of $q,z$ with using the dataset generation schema described in Appendix~\ref{app:dataset-gen-synthetic}. The summary of derived quantities is presented in Table \ref{tbl:app-syn-ef21-cvx}. We present several optimization problems to demonstrate the possible different relationships between $\LQM$ and $\LAM$. As we see from experiments, the \algname{EF21-W} is superior as the variance of $L_{i}$ tends to increase. The plots in Figure \ref{fig:app-syn-ef21-cvx} (a)--(d) correspond to decreasing variance of $L_{i}$. As we see, as the variance of $L_i$ decreases, the difference between \algname{EF21-W} and \algname{EF21} also tends to decrease. Finally, \algname{EF21-W} is always at least as best as \algname{EF21}.

\begin{table}[h]
	\begin{center}
		\caption{\small{Convex Optimization experiment in Figure \ref{fig:app-syn-ef21-cvx}. Quantities which define theoretical step size.}}
		\label{tbl:app-syn-ef21-cvx}
		\begin{tabular}{|c||c|c|c|c|c|c|c|c|c|c|}
			\hline
			Tag & $L$ & $q$ & $z$ & $\Lvar$ & $\xi=\sqrt{\frac{\beta}{\theta}}$ & $\LQM$ & $\LAM$ & $\gamma_{\algnametiny{EF21}}$ & $\gamma_{\algnametiny{EF21-W}}$ \\
			\hline
			(a) & $50$ & $1$ & $10^4$ & $4.45  \times  10^6$ & $18.486$ & $2111.90$ & $52.04$ & $2.55 \times 10^{-5}$ & $9.87 \times 10^{-4}$ \\
			\hline
			(b) & $50$ & $1$ & $10^3$ & $1.97  \times  10^6$ & $18.486$ & $1408.49$ & $63.56$ & $3.83 \times 10^{-5}$ & $8.16  \times 10^{-4}$ \\
			\hline
			(c) & $50$ & $1$ & $10^2$ & $1.08  \times 10^5$ & $18.486$ & $339.34$ & $80.97$ & $1.58 \times 10^{-4}$ & $6.46  \times 10^{-4}$ \\
			\hline
			(d) & $50$ & $0.8$ & $1$ & $5.42  \times  10^3$ & $18.486$ & $112.51$ & $85.03$ & $4.69 \times 10^{-4}$ & $6.16  \times  10^{-4}$\\
			\hline
		\end{tabular}
	\end{center}
\end{table}

\paragraph{Non-convex case with synthetic datasets.} 
 We aim to solve optimization problem  \eqref{eq:main_problem} in the case when  the functions $f_1,\dots,f_n$  are non-convex. In particular, we work with \begin{eqnarray}
 	\label{eq:ncvx-linear-reg}
	f_i(x) \eqdef \frac{1}{n_i} \norm{\bA_i x - {b_i}}^2 + \lambda \cdot \sum_{j=1}^{d} \frac{x_j^2}{x_j^2 + 1}.
\end{eqnarray}

\begin{center}	
	\begin{figure*}[t]
		\centering
		\captionsetup[sub]{font=normalsize,labelfont={}}	
		\captionsetup[subfigure]{labelformat=empty}
		
		\begin{subfigure}[ht]{0.3\textwidth}
			\includegraphics[width=\textwidth]{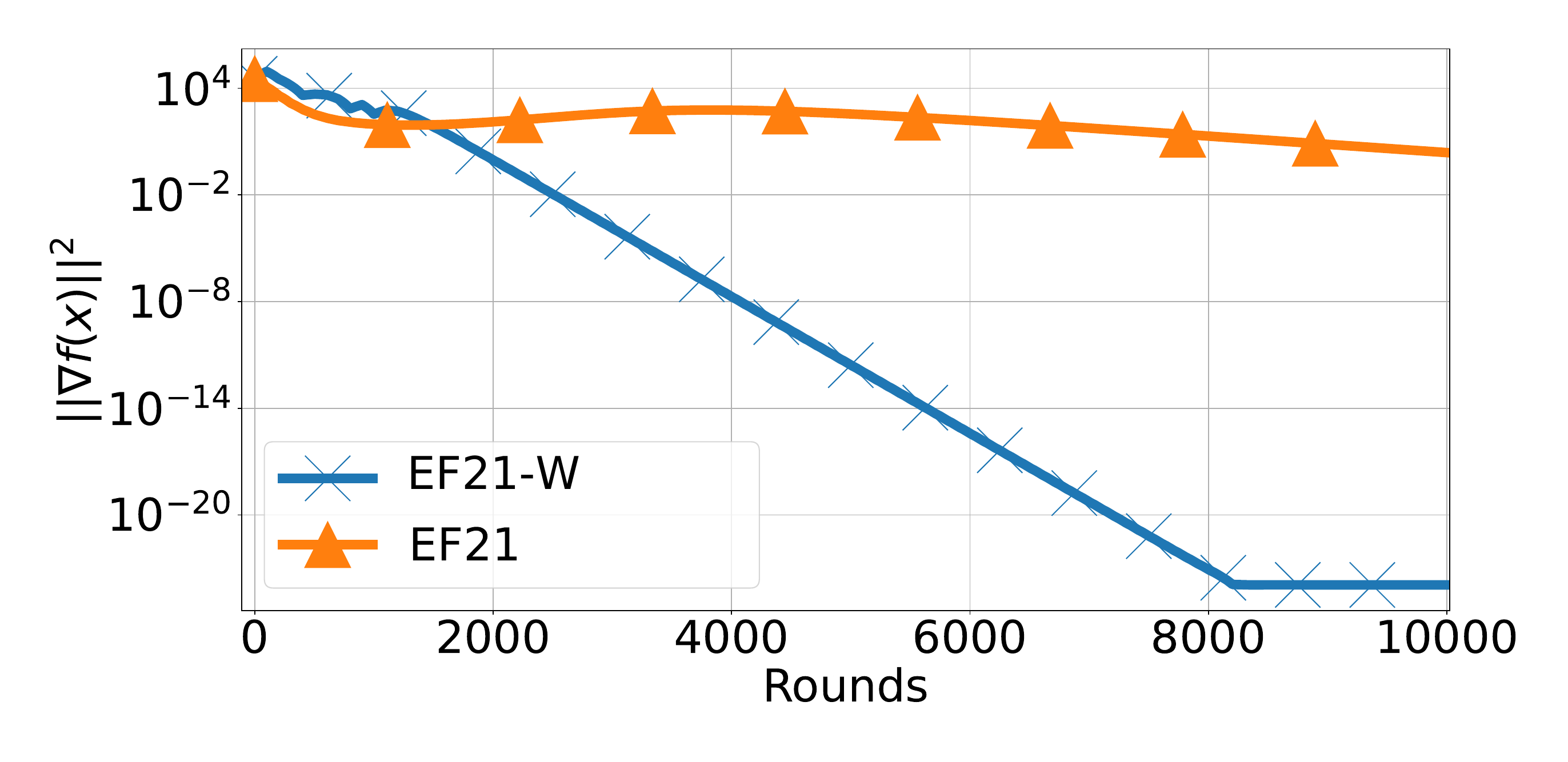} 
			\caption{(a) $\Lvar \approx 4.45 \times 10^6$}
		\end{subfigure}		
		\begin{subfigure}[ht]{0.3\textwidth}
			\includegraphics[width=\textwidth]{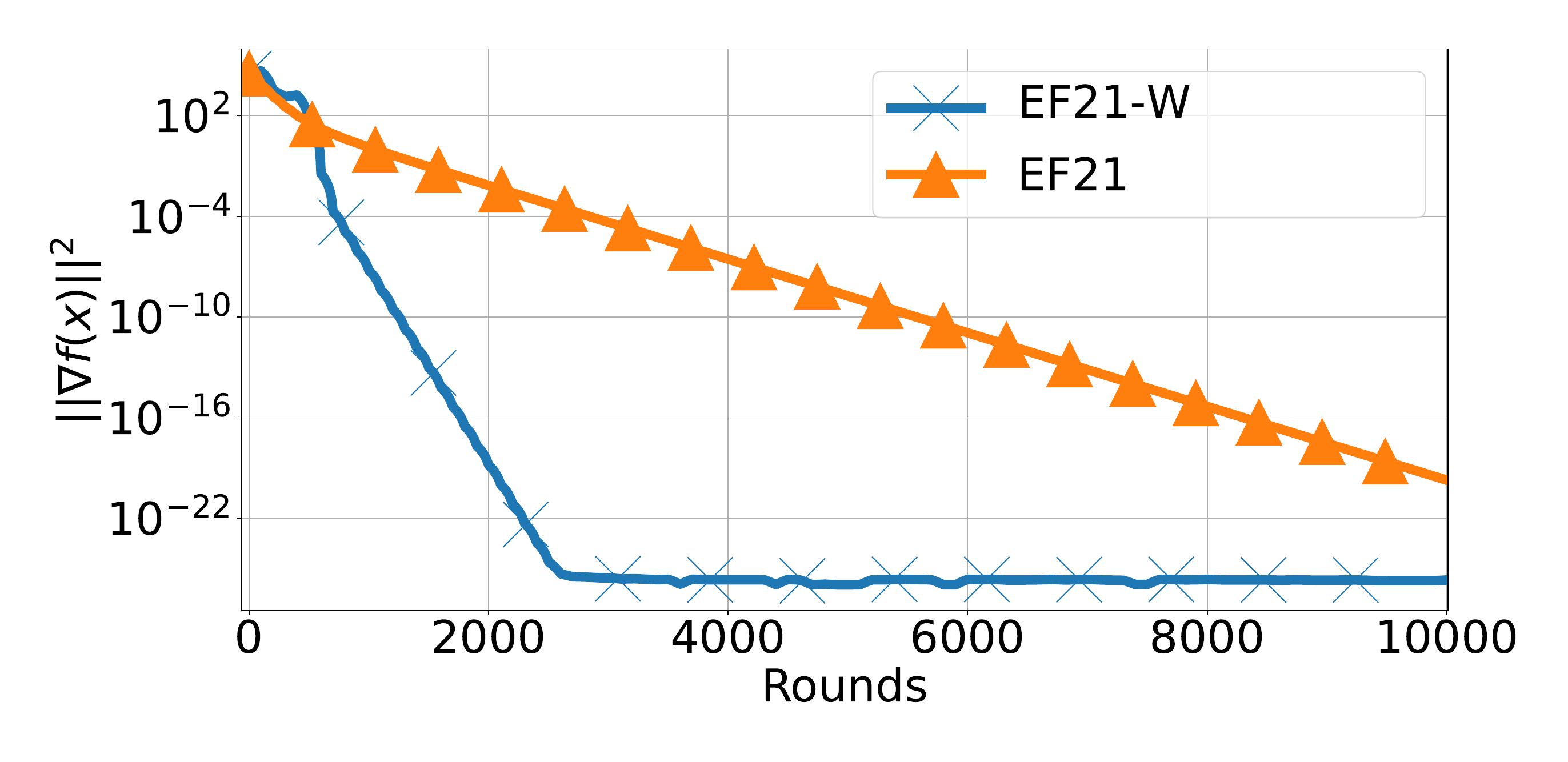} 
			\caption{(b) $\Lvar \approx 1.97  \times 10^6$}
		\end{subfigure}
		
		\begin{subfigure}[ht]{0.3\textwidth}
			\includegraphics[width=\textwidth]{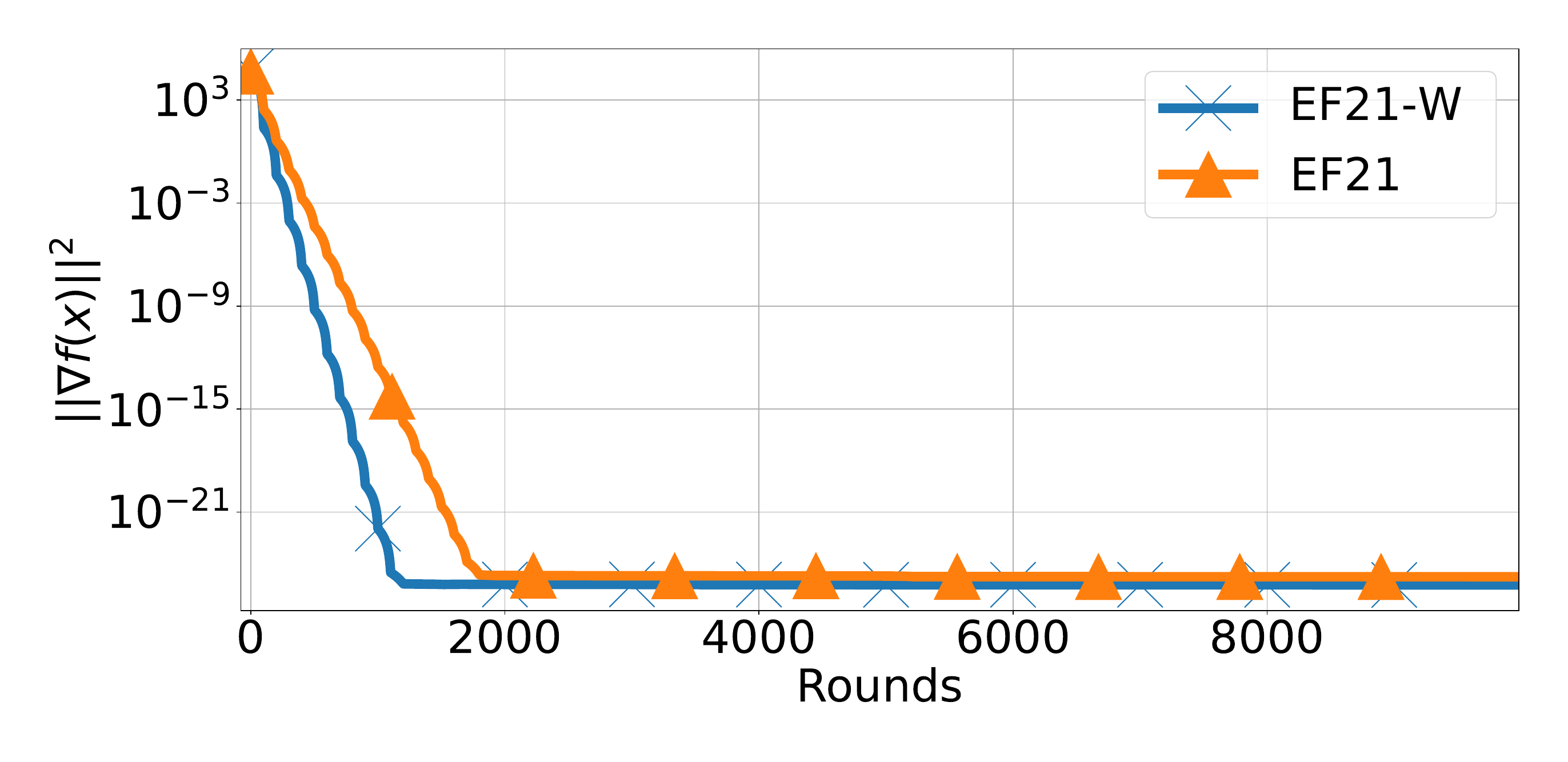} 
			\caption{(c) $\Lvar \approx 1.08  \times  10^5$}
		\end{subfigure}		
		\begin{subfigure}[ht]{0.3\textwidth}
			\includegraphics[width=\textwidth]{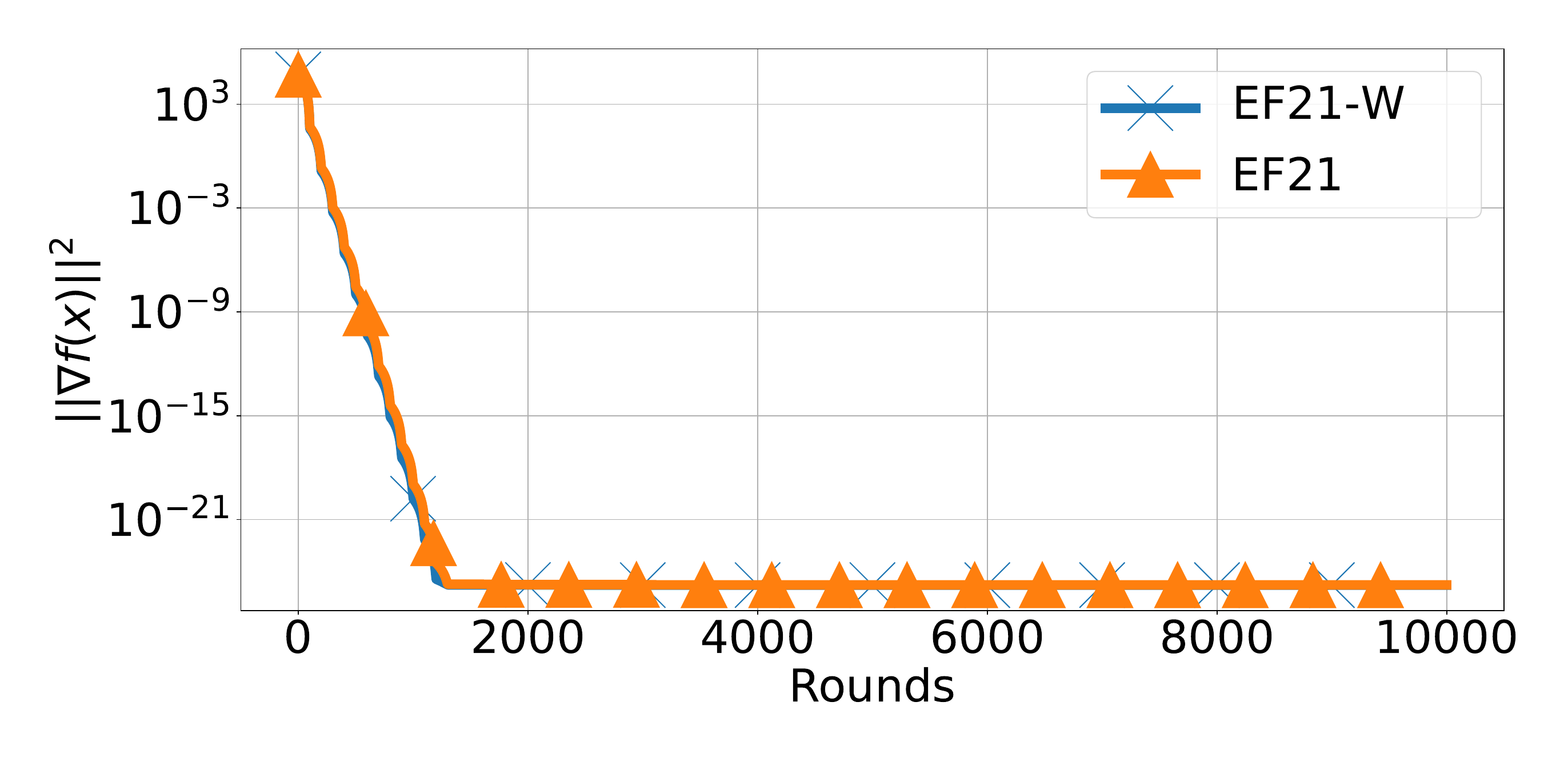} 
			\caption{(d) $\Lvar \approx 5.42  \times  10^3$}
		\end{subfigure}
		
		\caption{\small{Non-Convex smooth optimization. \algnamesmall{EF21} and \algnamesmall{EF21-W} with \algnamesmall{Top1} client compressor, $n=2,000$, $d=10$. The objective function is constitute of $f_i(x)$ defined in Eq. \eqref{eq:ncvx-linear-reg}.
Regularization term $\lambda \sum_{j=1}^{d} \frac{x_j^2}{x_j^2 + 1}$, with $\lambda = 100$. Theoretical step size. Full client participation. Extra details are in Table~\ref{tbl:app-syn-ef21-ncvx}.}}
		\label{fig:app-syn-ef21-ncvx}
	\end{figure*}
\end{center}

The result of experiments for training linear regression model with a non-convex regularization is presented in Figure \ref{fig:app-syn-ef21-ncvx}. The regularization coefficient $\lambda=100$. Instances of optimization problems were generated for values $L=50, \mu = 1$ and several values of $q,z$ for employed 
dataset generation schema from Appendix~\ref{app:dataset-gen-synthetic}.

 The summary of derived quantities is presented in Table \ref{tbl:app-syn-ef21-ncvx}.  We present various instances of optimization problems to demonstrate the different relationships between $\LQM$ and $\LAM$. As we see in the case of small variance of $L_{i}$ algorithm \algname{EF21-W} is at least as best as standard \algname{EF21}.

\begin{table}[h]
	\begin{center}
		\caption{\small{Non-convex optimization experiment in Figure \ref{fig:app-syn-ef21-ncvx}. Quantities which define theoretical step size.}}				
		\label{tbl:app-syn-ef21-ncvx}
		\begin{tabular}{|c||c|c|c|c|c|c|c|c|c|c|}
			\hline
			Tag & $L$ & $q$ & $z$ & $\Lvar$ & $\xi=\sqrt{\frac{\beta}{\theta}}$ & $\LQM$ & $\LAM$ & $\gamma_{\algnametiny{EF21}}$ & $\gamma_{\algnametiny{EF21-W}}$ \\
			\hline
			(a) & $50$ & $1$ & $10^4$ & $4.45  \times 10^6$ & $18.486$ & $2126.25$ & $252.035$ & $2.52 \times 10^{-5}$ & $2.03  \times  10^{-4}$ \\
			\hline
			(b) & $50$ & $1$ & $10^3$ & $1.97 \times 10^6$ & $18.486$ & $1431.53$ & $263.55$ & $3.74 \times 10^{-5}$ & $1.95 \times 10^{-4}$ \\
			\hline
			(c) & $50$ & $1$ & $10^2$ & $1.08 \times 10^5$ & $18.486$ & $433.05$ & $280.958$ & $1.21 \times 10^{-4}$ & $1.83  \times  10^{-4}$ \\
			\hline
			(d) & $50$ & $0.8$ & $1$ & $5.42  \times 10^3$ & $18.486$ & $294.39$ & $285.022$ & $1.17 \times 10^{-4}$ & $1.81  \times 10^{-4}$\\
			\hline
		\end{tabular}
	\end{center}
\end{table}

\paragraph{Non-convex logistic regression on benchmark datasets.} We aim to solve optimization problem  \eqref{eq:main_problem} in the case when the functions $f_1,\dots,f_n$ are non-convex. In particular, we work with logistic regression with a non-convex robustifying regularization term:
\begin{eqnarray}
	\label{eq:ncvx-log-reg}
	f_i(x) \eqdef \frac{1}{n_i} \sum_{j=1}^{n_i} \log \left(1+\exp({-y_{ij} \cdot a_{ij}^{\top} x})\right) + \lambda \cdot \sum_{j=1}^{d} \frac{x_j^2}{x_j^2 + 1},
\end{eqnarray}
where $ (a_{ij},  y_{ij}) \in \mathbb{R}^{d} \times \{-1,1\}$.

We used several \texttt{LIBSVM} datasets \citep{chang2011libsvm} for our benchmarking purposes. The results are presented in Figure \ref{fig:app-real-ef21-ncvx} and Figure \ref{fig:app-real-ef21-ncvx-aus}. The important quantities for these instances of optimization problems are summarized in Table \ref{tbl:app-real-ef21-ncvx}. From Figures \ref{fig:app-real-ef21-ncvx} (a), (b), (c), we can observe that for these datasets, the \algname{EF21-W} is better, and this effect is observable in practice. For example, from these examples, we can observe that $12.5$K rounds of \algname{EF21-W} corresponds to only $10$K rounds of \algname{EF21}. This improvement is essential for Federated Learning, in which both communication rounds and communicate information during a round represent the main bottlenecks and are the subject of optimization. Figures \ref{fig:app-real-ef21-ncvx} (d), (e), (f) demonstrate that sometimes the \algname{EF21-W} can have practical behavior close to \algname{EF21}, even if there is an improvement in step-size (For exact values of step size see Table \ref{tbl:app-real-ef21-ncvx}). The experiment on \texttt{AUSTRALIAN} datasets are presented in Figure \ref{fig:app-real-ef21-ncvx-aus}. This example demonstrates that in this \texttt{LIBSVM} benchmark datasets, the relative improvement in the number of rounds for \algname{EF21-W} compared to \algname{EF21} is considerable. For example $40$K rounds of \algname{EF21} corresponds to $5$K rounds of \algname{EF21-W} in terms of attainable $\| \nabla f(x^t)\|^2$.

\begin{table}[h]
	\begin{center}
	\tiny
		\caption{\small{Non-convex optimization experiments in Figures \ref{fig:app-real-ef21-ncvx}, \ref{fig:app-real-ef21-ncvx-aus}. Derived quantities which define theoretical step size.}}
		\label{tbl:app-real-ef21-ncvx}
		\begin{tabular}{|c||c|c|c|c|c|c|c|c|c|c|}
			\hline
			Tag & $L$ & $\Lvar$ & $\xi=\sqrt{\frac{\beta}{\theta}}$ & $\LQM$ & $\LAM$ & $\gamma_{\algnametiny{EF21}}$ & $\gamma_{\algnametiny{EF21-W}}$ \\
			\hline
			\tiny{(a) W1A} & $0.781$ & $3.283$ & $602.49$ & $2.921$ & $2.291$ & $5.678  \times 10^{-4}$ & $7.237   \times  10^{-4}$ \\
			\hline
			\tiny{(b) W2A} & $0.784$ & $2.041$ & $602.49$ & $2.402$ & $1.931$ & $6.905  \times 10^{-4}$ & $8.589  \times 10^{-4}$ \\
			\hline
			\tiny{(c) W3A} & $0.801$ & $1.579$ & $602.49$ & $2.147$ & $1.741$ & $7.772  \times 10^{-4}$ & $9.523   \times  10^{-4}$ \\
			\hline
			\tiny{(d) MUSHROOMS} & $2.913$ & $5.05 \times10^{-1}$ & $226.498$ & $3.771$ & $3.704$ & $1.166  \times 10^{-3}$ & $1.187  \times 10^{-3}$\\
			\hline
			\tiny{(e) SPLICE} & $96.082$ & $2.23   \times  10^{2}$ & $122.497$ & $114.43$ & $113.45$ & $7.084  \times 10^{-5}$ & $7.14  \times 10^{-5}$\\
			\hline
			\tiny{(f) PHISHING} & $0.412$ & $9.2   \times  10^{-4}$ & $138.498$ & $0.429$ & $0.428$ & $1.670  \times 10^{-2}$ & $1.674   \times  10^{-2}$\\			
			\hline				
			\tiny{(g) AUSTRALIAN} & $3.96   \times  10^6 $ & $1.1   \times  10^{16}$ & $18.486$ & $3.35   \times  10^{7}$ & $3.96   \times  10^{6}$ & $9.733  \times 10^{-10}$ & $8.007   \times  10^{-9}$\\
			\hline				
		\end{tabular}
	\end{center}
\end{table}

\begin{center}	
	\begin{figure*}[t]
		\centering
		\captionsetup[sub]{font=normalsize,labelfont={}}	
		\captionsetup[subfigure]{labelformat=empty}
		
		\begin{subfigure}[ht]{0.3\textwidth}
			\includegraphics[width=\textwidth]{./ef21-vc-figs/expreal/fig1-09JULY23-w1a.pdf}
			\caption{\small{(a) \texttt{W1A}}}
		\end{subfigure}		
		\begin{subfigure}[ht]{0.3\textwidth}
			\includegraphics[width=\textwidth]{./ef21-vc-figs/expreal/fig2-09JULY23-w2a.pdf}
			\caption{\small{(b) \texttt{W2A}}}
		\end{subfigure}
		\begin{subfigure}[ht]{0.3\textwidth}
			\includegraphics[width=\textwidth]{./ef21-vc-figs/expreal/fig3-09JULY23-w3a.pdf} 
			\caption{\small{(c) \texttt{W3A}}}
		\end{subfigure}	
		\\
		\begin{subfigure}[ht]{0.3\textwidth}
			\includegraphics[width=\textwidth]{./ef21-vc-figs/expreal/fig4-09JULY23-mushrooms}
			\caption{\small{(d) \texttt{MUSHROOMS}}} 
		\end{subfigure}		
		\begin{subfigure}[ht]{0.3\textwidth}
			\includegraphics[width=\textwidth]{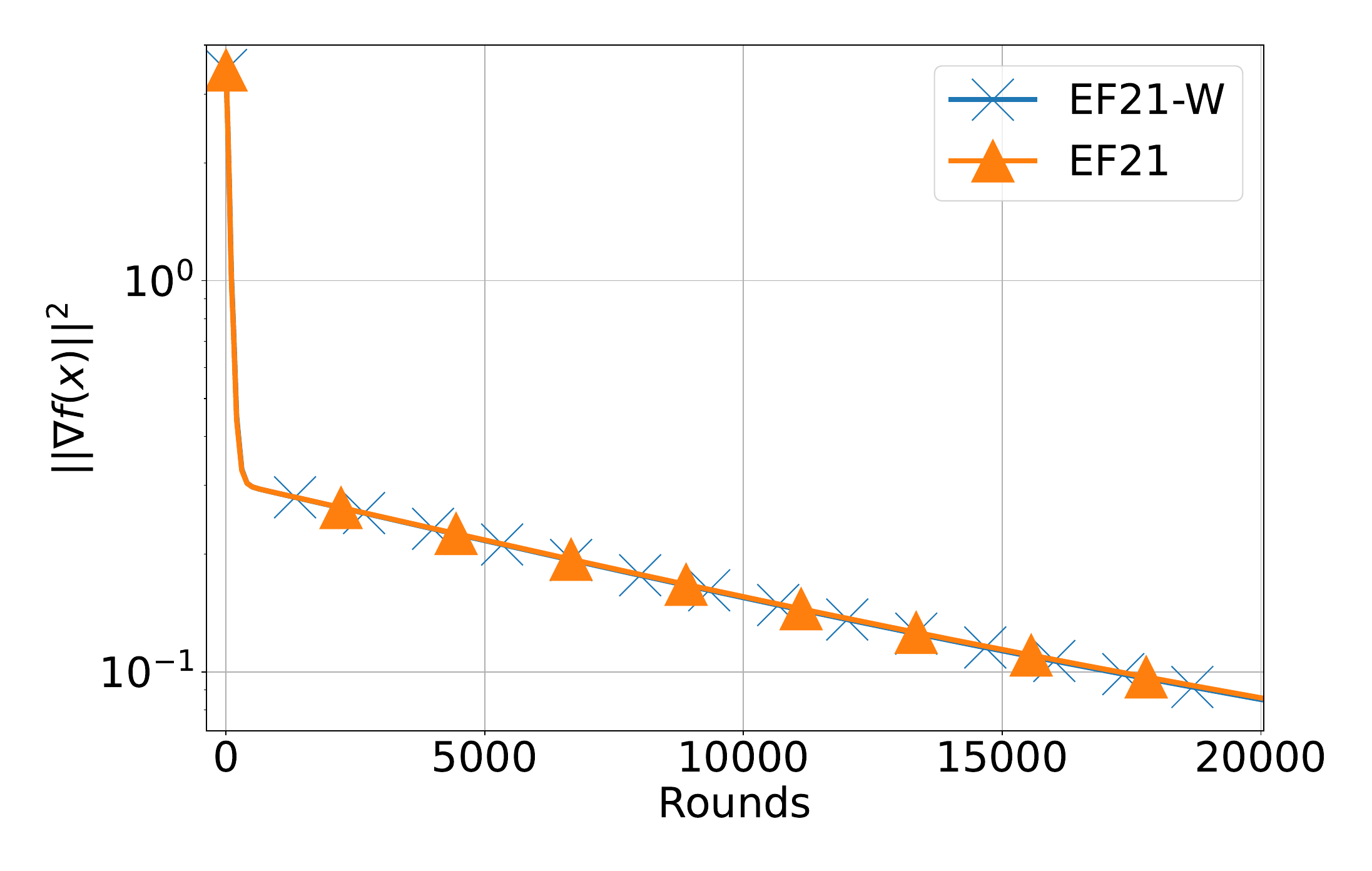} 
			\caption{\small{(e) \texttt{SPLICE}}}
		\end{subfigure}
		\begin{subfigure}[ht]{0.3\textwidth}
			\includegraphics[width=\textwidth]{./ef21-vc-figs/expreal/fig6-09JULY23-phishing.pdf} 
			\caption{\small{(f) \texttt{PHISHING}}} 
		\end{subfigure}
		
		\caption{\small{Non-Convex Logistic Regression: comparison of \algnamesmall{EF21} and \algnamesmall{EF21-W}. The used compressor is \algnamesmall{Top1}. The number of clients $n=1,000$. Regularization term $\lambda \sum_{j=1}^{d} \frac{x_j^2}{x_j^2 + 1}$, with $\lambda=0.001$. Theoretical step size. Full client participation. The objective function is constitute of $f_i(x)$ defined in Eq. \eqref{eq:ncvx-log-reg}. Extra details are in Table \ref{tbl:app-real-ef21-ncvx}.}}
		\label{fig:app-real-ef21-ncvx}
	\end{figure*}
\end{center}

\begin{center}	
	\begin{figure*}[t]
		\centering
		\captionsetup[sub]{font=normalsize,labelfont={}}	
		\captionsetup[subfigure]{labelformat=empty}
		
		\begin{subfigure}[ht]{0.3\textwidth}
			\includegraphics[width=\textwidth]{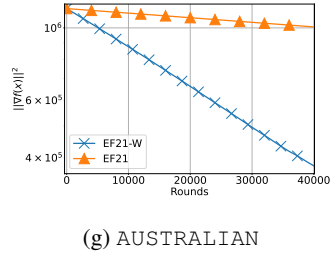} 
			\caption{\small{(g) \texttt{AUSTRALIAN}}}
		\end{subfigure}
		
		\caption{\small{Non-Convex Logistic Regression: comparison of the performance of standard \algnamesmall{EF21} and \algnamesmall{EF21-W}. The used compressor is \algnamesmall{Top1}. The number of clients $n=200$. Regularization term $\lambda \sum_{j=1}^{d} \frac{x_j^2}{x_j^2 + 1}$, with $\lambda=1,000$. Theoretical step size. The objective function is constitute of $f_i(x)$ defined in Eq. \eqref{eq:ncvx-log-reg}. Extra details are in Table  \ref{tbl:app-real-ef21-ncvx}.}}
		\label{fig:app-real-ef21-ncvx-aus}
	\end{figure*}
\end{center}

\paragraph{Non-convex logistic regression with non-homogeneous compressor.}

In this supplementary experiment, we leveraged the \texttt{AUSTRALIAN} \texttt{LIBSVM} datasets \citep{chang2011libsvm} to train logistic regression, incorporating a non-convex sparsity-enhanced regularization term defined in Eq. \eqref{eq:ncvx-log-reg}. The experiment featured the utilization of a non-homogeneous compressor known as \algname{Natural} by \citet{horvoth2022natural}, belonging to the family of unbiased compressors and adhering to Definition \ref{eq:unbiased} with $w=1/8$. This compressor, in a randomized manner, rounds the exponential part and zeros out the transferred mantissa part when employing the standard IEEE 754 Standard for Floating-Point Arithmetic \cite{IEEE754-2008} representation for floating-point numbers. Consequently, when using a single float-point format (FP32) during communication, only $9$ bits of payload per scalar need to be sent to the master, and the remaining $23$ bits of mantissa can be entirely dropped.

The experiment results are depicted in Figure \ref{fig:app-natual-ef21-ncvx}. In this experiment, we fine-tuned the theoretical step size by multiplying it with a specific constant. As we can see the \algname{EF21-W} consistently outperforms \algname{EF21} across all corresponding step-size multipliers. As we see \algname{EF21-W} operates effectively by utilizing unbiased non-homogeneous compressors, and the advantages over \algname{EF21} extend beyond the scope of applying \algname{EF21-W} solely to homogeneous compressors. Finally, it is worth noting that the increased theoretical step size in \algname{EF21-W} does not entirely capture the practical scenario of potentially enhancing the step size by a significantly large multiplicative factor (e.g., $\times 40$), which remains a subject for future research.

\begin{center}	
	\begin{figure*}[t]
		\centering
		\captionsetup[sub]{font=normalsize,labelfont={}}	
		\captionsetup[subfigure]{labelformat=empty}
		
		\begin{subfigure}[ht]{0.35\textwidth}
			\includegraphics[width=\textwidth]{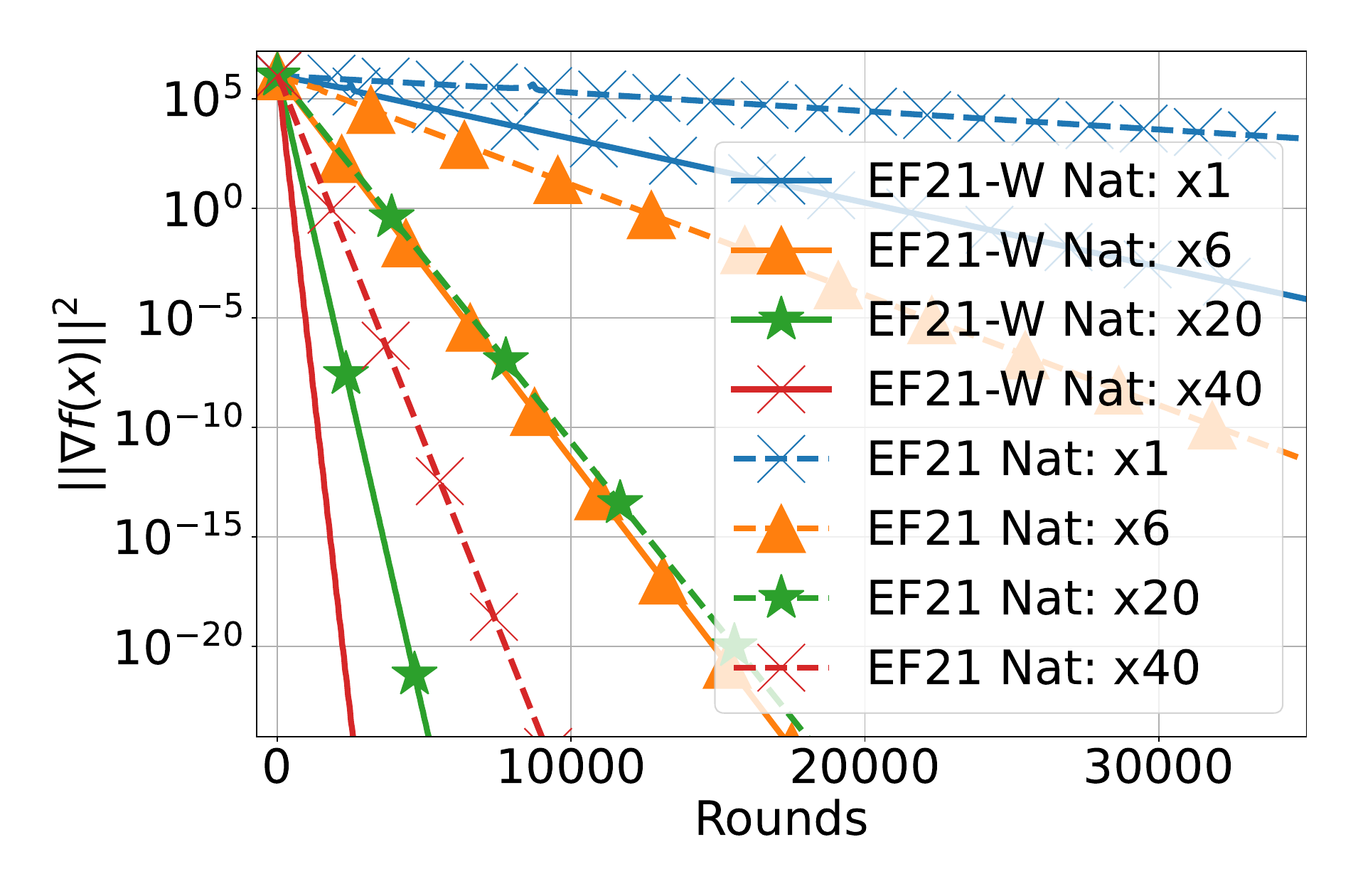} \caption{\small{(a)}}
		\end{subfigure}
		\begin{subfigure}[ht]{0.35\textwidth}
			\includegraphics[width=\textwidth]{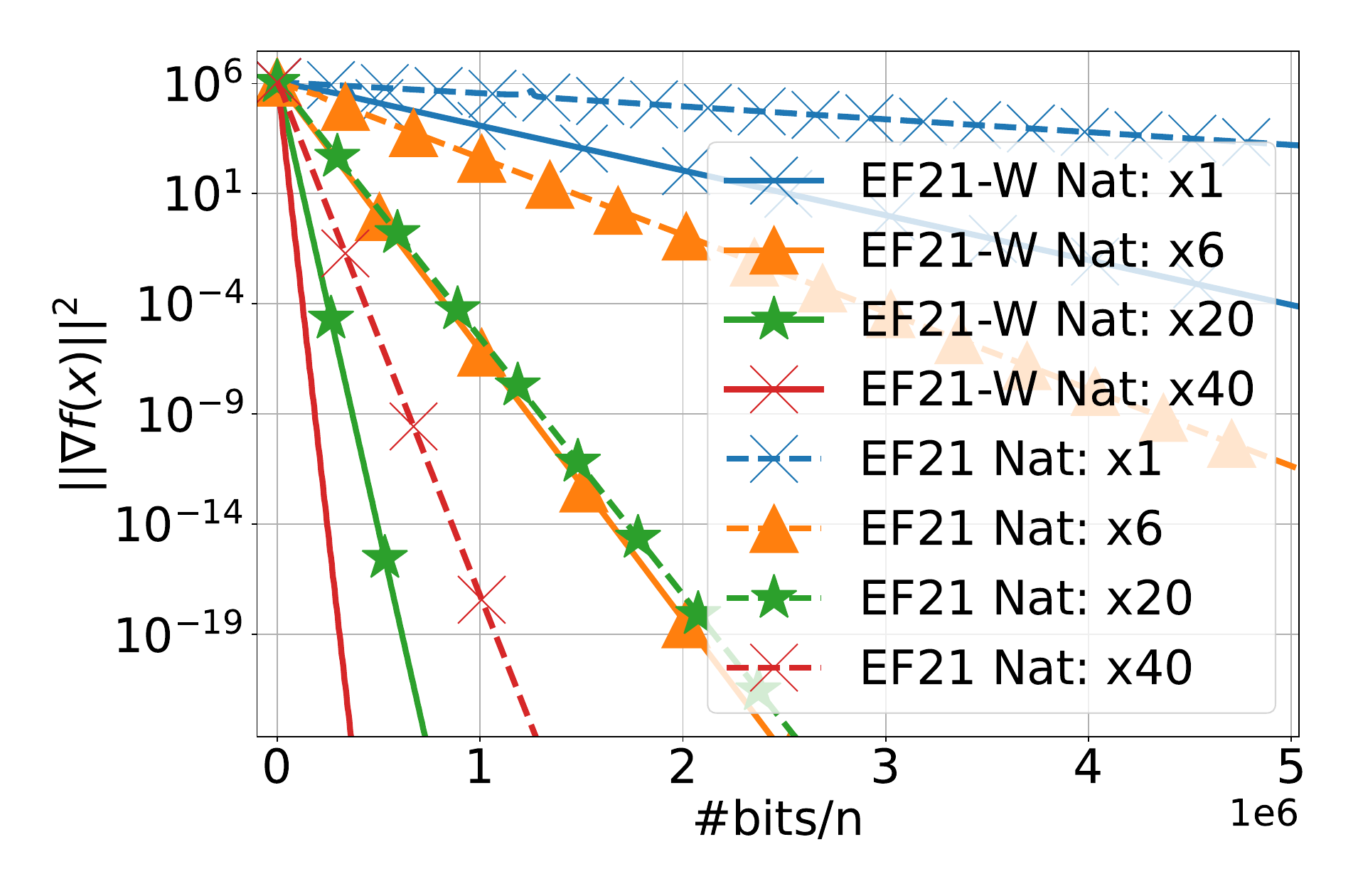} \caption{\small{(b)}}
		\end{subfigure}
				
				\caption{\small{Non-Convex Logistic Regression: comparison of the performance of standard \algnamesmall{EF21} and \algnamesmall{EF21-W}. The used compressor for \algnamesmall{EF21} and \algnamesmall{EF21-W} is \algnamesmall{Natural} compressor \citet{horvoth2022natural}. The number of clients $n=200$. The objective function is constitute of $f_i(x)$ defined in Eq. \eqref{eq:ncvx-log-reg}. Regularization term $\lambda \sum_{j=1}^{d} \frac{x_j^2}{x_j^2 + 1}$, with $\lambda=1,000$. Multipliers of theoretical step size. Full participation. Computation format single precision (FP32). Dataset: \texttt{AUSTRALIAN}.}}
				\label{fig:app-natual-ef21-ncvx}
		\end{figure*}
\end{center}

\clearpage

\subsection{Additional experiments for \algname{EF21-W-PP}}

\paragraph{Convex case with synthetic datasets.} 

\begin{center}	
	\begin{figure*}[t]
		\centering
		\captionsetup[sub]{font=normalsize,labelfont={}}	
		\captionsetup[subfigure]{labelformat=empty}
		
		\begin{subfigure}[ht]{0.3\textwidth}
			\includegraphics[width=\textwidth]{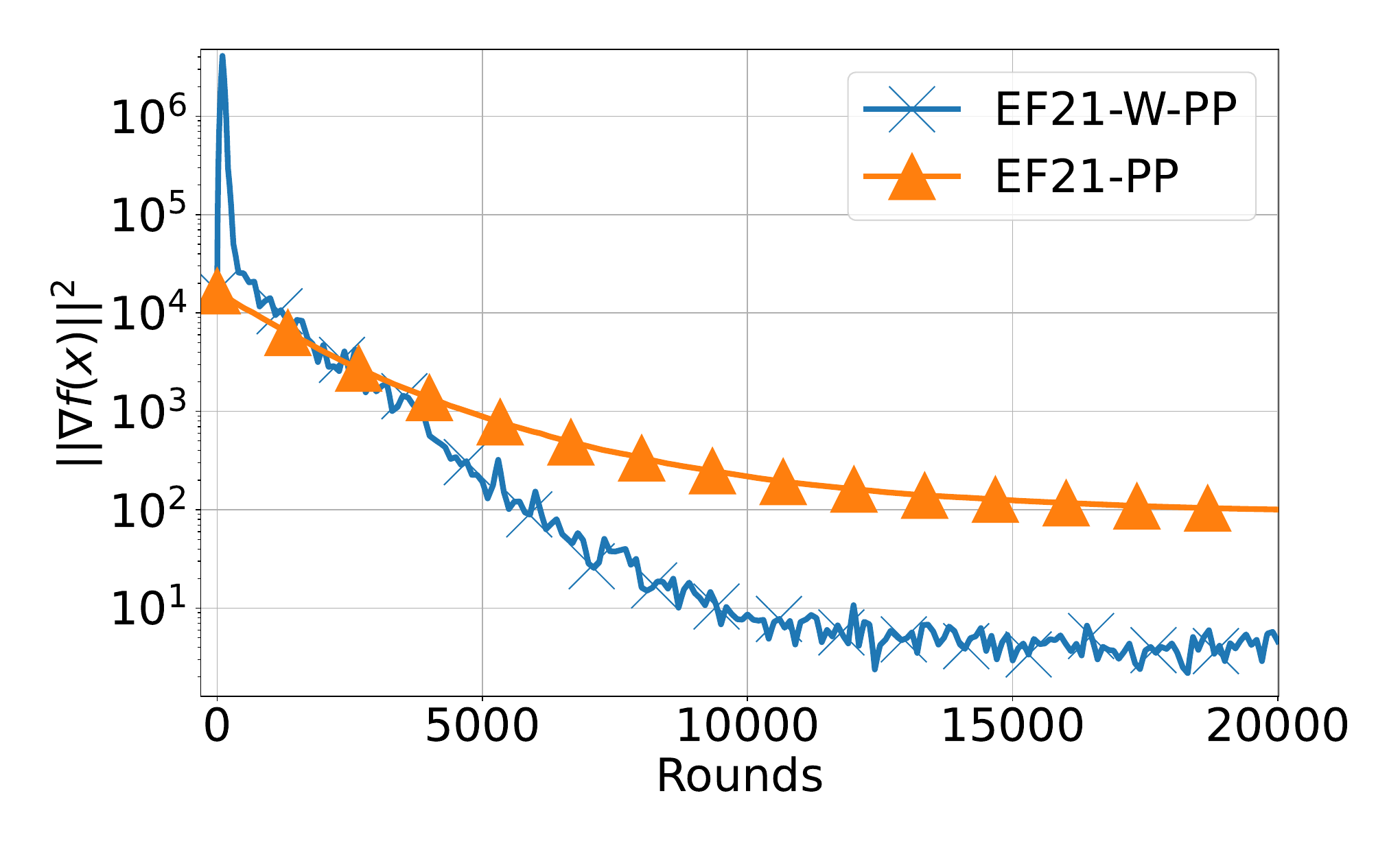} \caption{ (a) $\Lvar = 4.45   \times  10^6$ }
		\end{subfigure}		
		\begin{subfigure}[ht]{0.3\textwidth}
			\includegraphics[width=\textwidth]{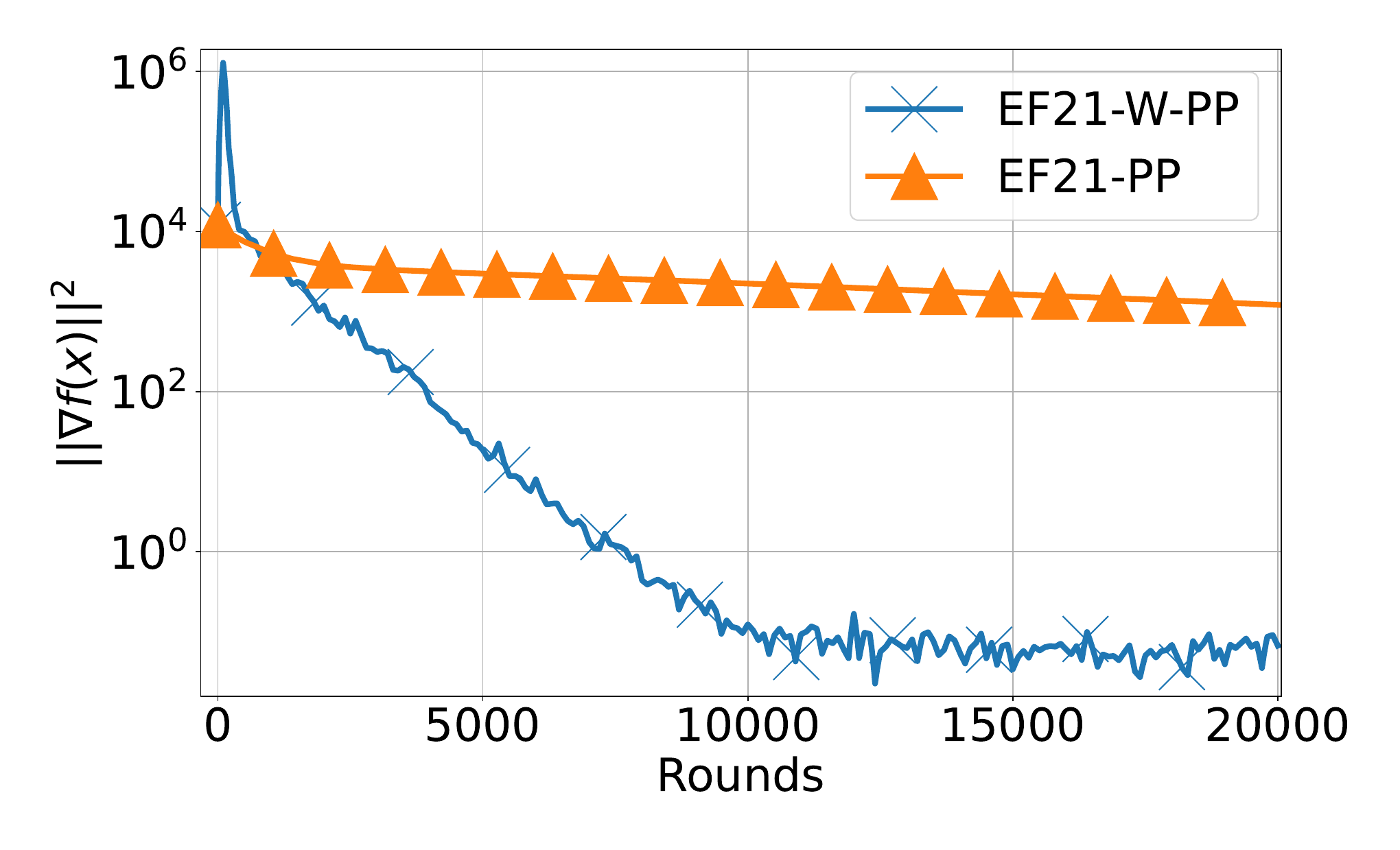} \caption{ (b) $\Lvar = 1.97   \times  10^6$ }
		\end{subfigure}
		\begin{subfigure}[ht]{0.3\textwidth}
			\includegraphics[width=\textwidth]{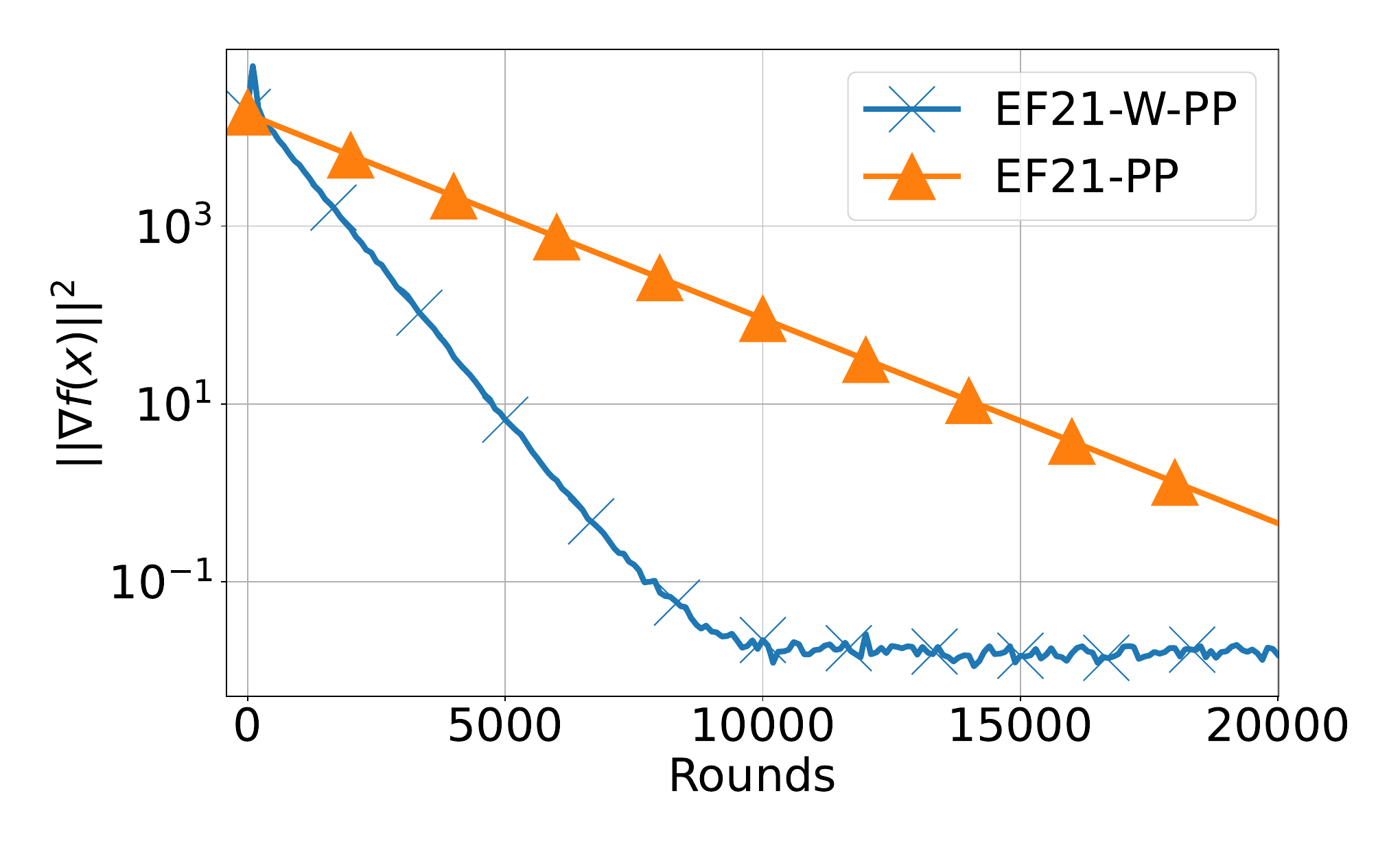} \caption{ (c) $\Lvar = 1.08   \times  10^5$}
		\end{subfigure}
		
		\caption{\small{Convex smooth optimization. \algnamesmall{EF21-PP} and \algnamesmall{EF21-W-PP} with \algnamesmall{Top1} client compressor, $n=2\,000$, $d=10$. The objective function is constitute of $f_i(x)$ defined in Eq. \eqref{eq:cvx-linear-reg}. Regularization term $\lambda \frac{\|x\|^2}{2}$, $\lambda=0.01$. Theoretical step size. The objective function is constitute of $f_i(x)$ defined in Eq.\eqref{eq:cvx-linear-reg}.	Each client participates in each round with probability $p_i=0.5$. Extra details are in Table \ref{tbl:app-syn-ef21-pp-cvx}.}}
		\label{fig:app-syn-ef21-pp-cvx}
	\end{figure*}
\end{center}

\begin{table}[h]
	\begin{center}
		\caption{\small{Convex optimization experiment in Figure \ref{fig:app-syn-ef21-pp-cvx}. Derived quantities which define theoretical step size.}}
		\label{tbl:app-syn-ef21-pp-cvx}
		\begin{tabular}{|c||c|c|c|c|c|c|c|c|c|c|}
			\hline
			Tag & $L$ & $q$ & $z$ & $\Lvar$ & $\sqrt{\frac{\beta}{\theta}}$ & $\LQM$ & $\LAM$ & $\gamma_{\algnametiny{EF21-PP}}$ & $\gamma_{\algnametiny{EF21-W-PP}}$ \\
			\hline
			(a) & $50$ & $1$ & $10^4$ & $4.45   \times  10^6$ & $18.486$ & $2111.90$ & $52.04$ & $2.55  \times 10^{-5}$ & $9.87   \times  10^{-4}$ \\
			\hline
			(b) & $50$ & $1$ & $10^3$ & $1.97   \times  10^6$ & $18.486$ & $1408.49$ & $63.56$ & $3.83  \times 10^{-5}$ & $8.16   \times  10^{-4}$ \\
			\hline
			(c) & $50$ & $1$ & $10^2$ & $1.08   \times  10^5$ & $18.486$ & $339.34$ & $80.97$ & $1.58  \times 10^{-4}$ & $6.46   \times  10^{-4}$ \\
			\hline
		\end{tabular}
	\end{center}
\end{table}

We aim to solve optimization problem  \eqref{eq:main_problem} in the case when 
the functions $f_1,\dots,f_n$ are strongly convex. In particular, we choose: \begin{eqnarray}
	\label{eq:cvx-linear-reg}
	f_i(x) \eqdef \frac{1}{n_i} \norm{\bA_i x - {b_i}}^2 + \frac{\lambda}{2} \|x\|^2.
\end{eqnarray}

In this synthetic experiment, we have used the maximum allowable step size suggested by the theory of \algname{EF21-PP} and for the proposed \algname{EF21-W-PP} algorithm. The initial  gradient estimators have been initialized as $g_i^0 = \nabla f_i (x^0)$ for all $i$. The number of clients in simulation $n=2000$, dimension of optimization problem $d=10$, number of samples per client $n_i=10$, and number of communication rounds is $r=10,000$. For both \algname{EF21-PP} and \algname{EF21-W-PP} clients we used \algname{Top1} biased contractile compressor. In our experiment, each client's participation in each communication round is governed by an independent Bernoulli trial which takes $p_i=0.5$. The result of experiments for training linear regression model with a convex regularizer is presented in Figure \ref{fig:app-syn-ef21-pp-cvx}. The regularization constant was chosen to be $\lambda=0.01$. Instances of optimization problems were generated for values $L=50$,$ \mu = 1$ with several values of $q$ and $z$. The summary of derived quantities is presented in Table \ref{tbl:app-syn-ef21-pp-cvx}. We present several optimization problems to demonstrate the possible different relationships between $\LQM$ and $\LAM$. As we see from experiments, the \algname{EF21-W-PP} is superior as the variance of $L_{i}$ tends to increase. As we can observe \algname{EF21-W-PP} is always at least as best as \algname{EF21-PP}.

\paragraph{Non-convex logistic regression on benchmark datasets.}  We provide additional numerical experiments in which we compare \algname{EF21-PP} and \algname{EF21-W-PP} for solving \eqref{eq:main_problem}. We address the problem of training a binary classifier via a logistic model on several \texttt{LIBSVM} datasets \citep{chang2011libsvm} with non-convex regularization. We consider the case when the functions $f_1,\dots,f_n$ are non-convex; in particular, we set
$f_i(x)$ as follows:
\begin{eqnarray}
	\label{eq:ncvx-log-reg-2}
	f_i(x) \eqdef \frac{1}{n_i} \sum_{j=1}^{n_i} \log \left(1+\exp({-y_{ij} \cdot a_{ij}^{\top} x})\right) + \lambda \cdot \sum_{j=1}^{d} \frac{x_j^2}{x_j^2 + 1}, 
\end{eqnarray}
where $(a_{ij},  y_{ij}) \in \mathbb{R}^{d} \times \{-1,1\}$.

We conducted distributed training of a logistic regression model on \texttt{W1A}, \texttt{W2A}, \texttt{W3A}, \texttt{PHISHING}, and \texttt{AUSTRALIAN} datasets with non-convex regularization. The initial gradient estimators are set $g_i^0 = \nabla f_i(x^0)$ for all $ i \in [n]$. For comparison of \algname{EF21-PP} and \algname{EF21-W-PP}, we used the largest step size allowed by theory. We used the dataset shuffling strategy described in Appendix~\ref{app:dataset-shuffling-for-libsvm}. The results are presented in Figure \ref{fig:app-real-ef21-pp-ncvx} and Figure \ref{fig:app-real-ef21-pp-ncvx-aus}. The important quantities for these instances of optimization problems are summarized in Table~\ref{tbl:app-real-ef21-pp-ncvx}.

\begin{table}[h]
	\begin{center}
		\caption{\small{Non-Convex optimization experiments in Figures \ref{fig:app-real-ef21-pp-ncvx}, \ref{fig:app-real-ef21-pp-ncvx-aus}. Quantities which define theoretical step size.}}
		\label{tbl:app-real-ef21-pp-ncvx}
		\begin{tabular}{|c||c|c|c|c|c|c|c|c|c|}
			\hline
			Tag & $L$ & $\Lvar$ & $\LQM$ & $\LAM$ & $\gamma_{\algnametiny{EF21-PP}}$ & $\gamma_{\algnametiny{EF21-W-PP}}$ \\
			\hline
			\tiny{(a) W1A} & $0.781$ & $3.283$ & $2.921$ & $2.291$ & $2.315  \times 10^{-4}$ & $2.95   \times  10^{-4}$ \\
			\hline
			\tiny{(b) W2A} & $0.784$ & $2.041$ & $2.402$ & $1.931$ & $2.816  \times 10^{-4}$ & $3.503  \times 10^{-4}$ \\
			\hline
			\tiny{(c) W3A} & $0.801$ & $1.579$ & $2.147$ & $1.741$ & $3.149   \times  10^{-4}$ & $3.884   \times  10^{-4}$ \\
			\hline
			\tiny{(d) PHISHING} & $0.412$ & $9.2   \times  10^{-4}$ & $0.429$ & $0.428$ & $6.806  \times 10^{-3}$ & $6.823   \times  10^{-3}$\\
			\hline				
			\tiny{(e) AUSTRALIAN} & $3.96   \times  10^6 $ & $1.1   \times  10^{16}$ & $3.35   \times  10^{7}$ & $3.96   \times  10^{6}$ & $3.876  \times 10^{-10}$ & $3.243   \times  10^{-9}$\\
			\hline				
		\end{tabular}
	\end{center}
\end{table}

\begin{center}	
	\begin{figure*}[t]
		\centering
		\captionsetup[sub]{font=normalsize,labelfont={}}	
		\captionsetup[subfigure]{labelformat=empty}
		
		\begin{subfigure}[ht]{0.3\textwidth}
			\includegraphics[width=\textwidth]{./ef21-vc-figs-pp/expreal/w1a.pdf} \caption{ \small{(a) \texttt{W1A} }}
		\end{subfigure}		
		\begin{subfigure}[ht]{0.3\textwidth}
			\includegraphics[width=\textwidth]{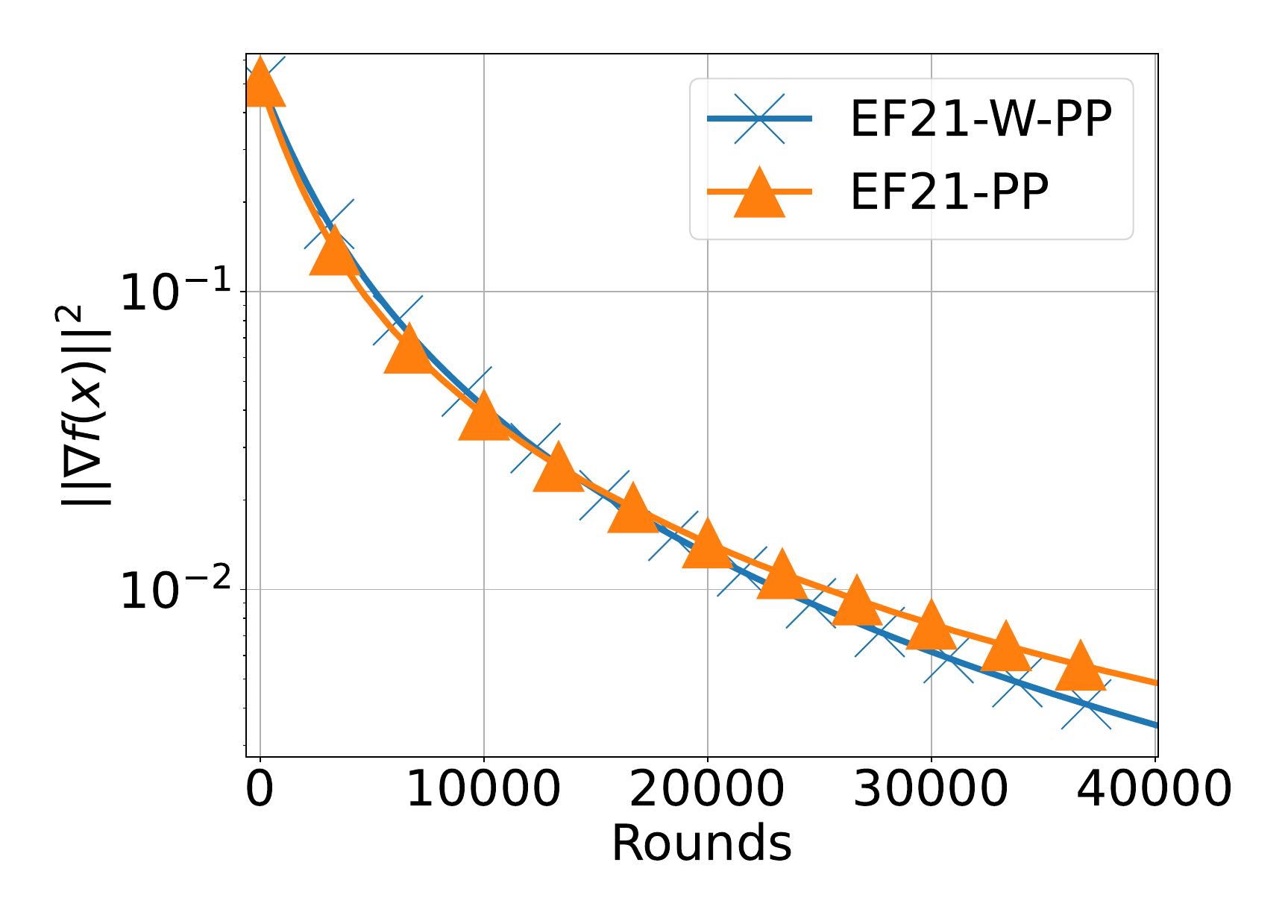} \caption{ \small{ (b) \texttt{W2A} }}
		\end{subfigure} \\
		\begin{subfigure}[ht]{0.3\textwidth}
			\includegraphics[width=\textwidth]{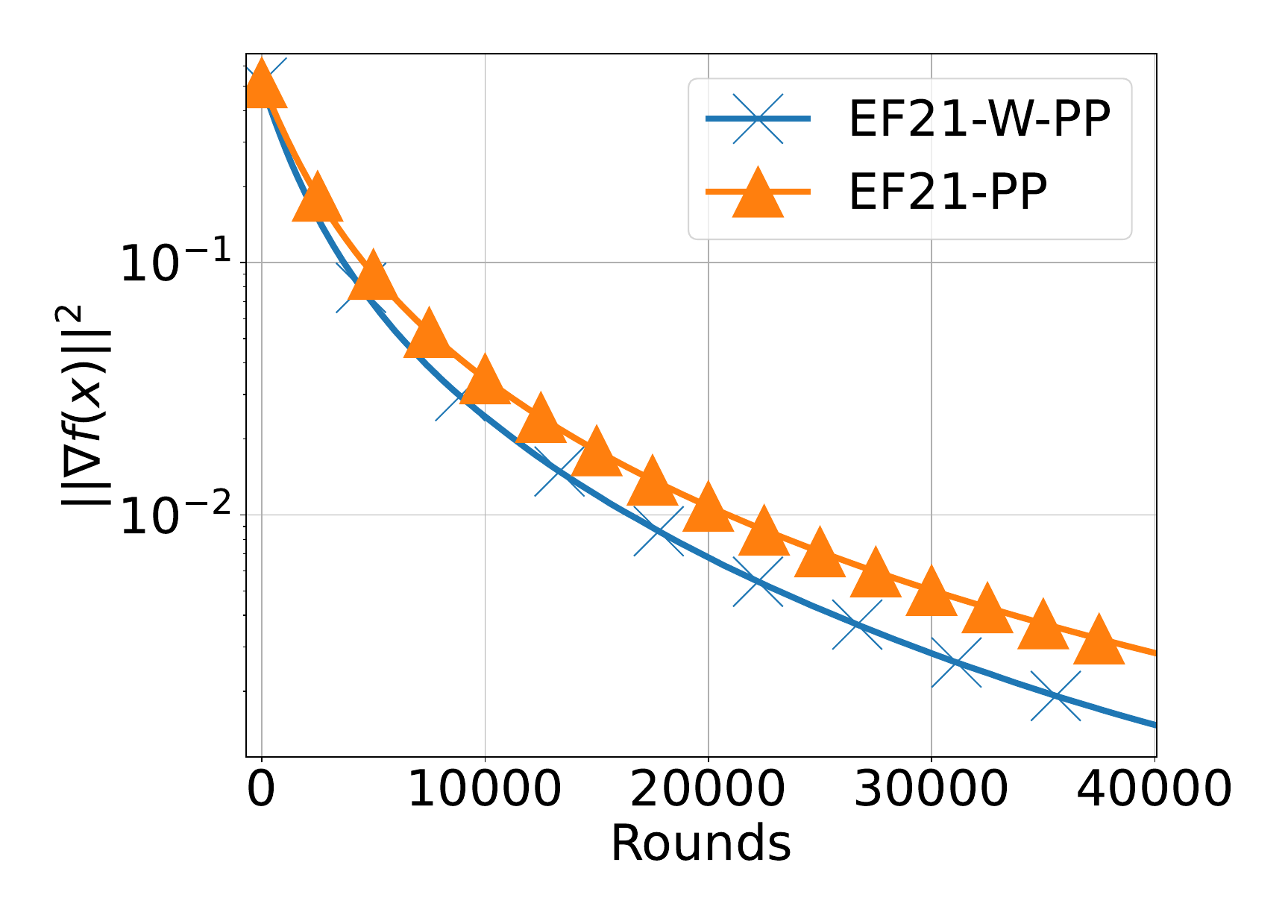} \caption{ \small{(c) \texttt{W3A} }}
		\end{subfigure}		
		\begin{subfigure}[ht]{0.3\textwidth}
			\includegraphics[width=\textwidth]{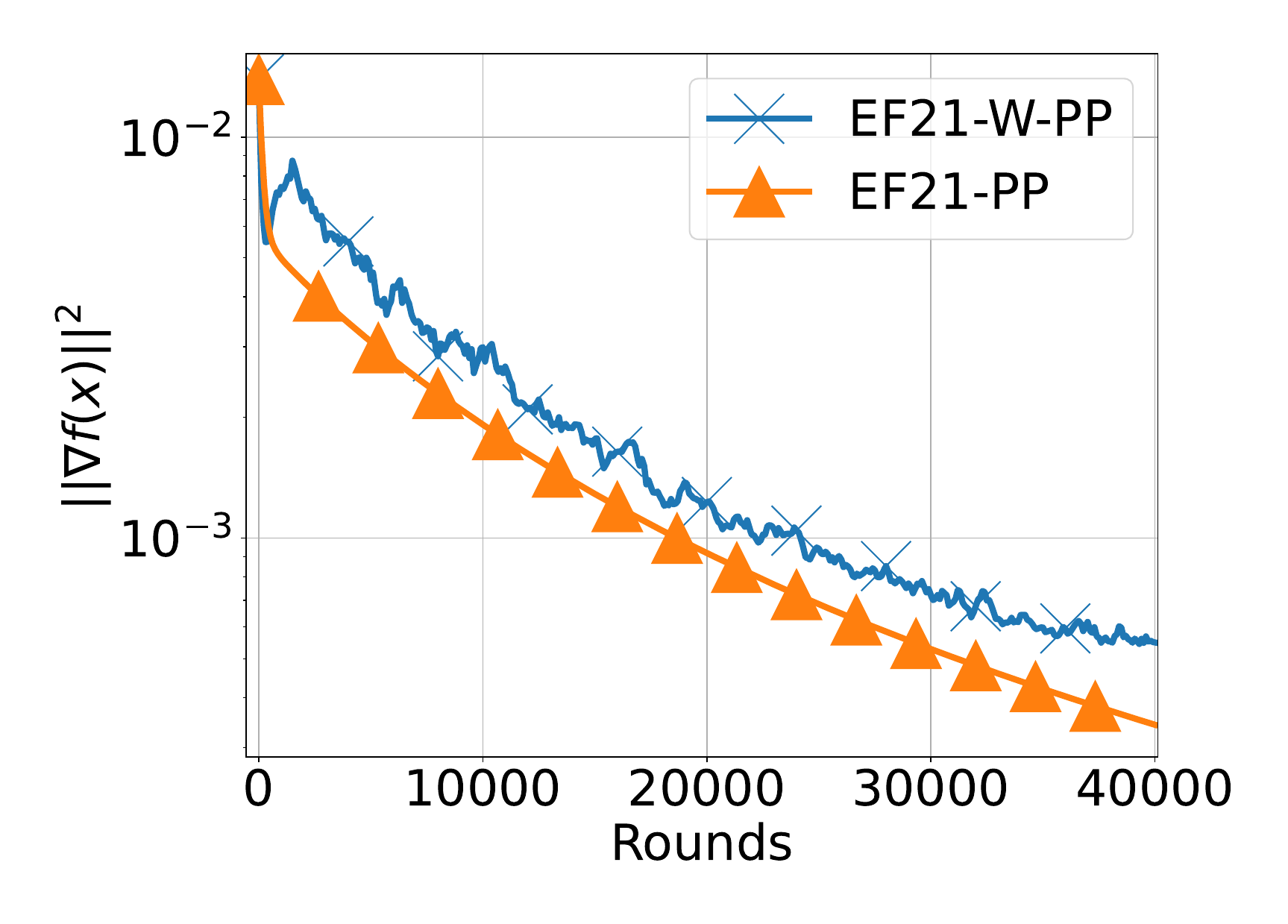} \caption{ \small{(d) \texttt{PHISHING}}}
		\end{subfigure}
		
		\caption{\small{Non-Convex Logistic Regression: comparison of \algnamesmall{EF21-PP} and \algnamesmall{EF21-W-PP}. The used compressor is \algname{Top1}. The number of clients $n=1,000$. Regularization term $\lambda \sum_{j=1}^{d} \frac{x_j^2}{x_j^2 + 1}$, $\lambda=0.001$.  Theoretical step size. Each client participates in each round with probability $p_i=0.5$. The objective function is constitute of $f_i(x)$ defined in Eq.\eqref{eq:ncvx-log-reg-2}. Extra details are in Table \ref{tbl:app-real-ef21-pp-ncvx}.}}
		\label{fig:app-real-ef21-pp-ncvx}
	\end{figure*}
\end{center}

\begin{center}	
	\begin{figure*}[t]
		\centering
		\captionsetup[sub]{font=normalsize,labelfont={}}	
		\captionsetup[subfigure]{labelformat=empty}
		
		\begin{subfigure}[ht]{0.4\textwidth}
			\includegraphics[width=\textwidth]{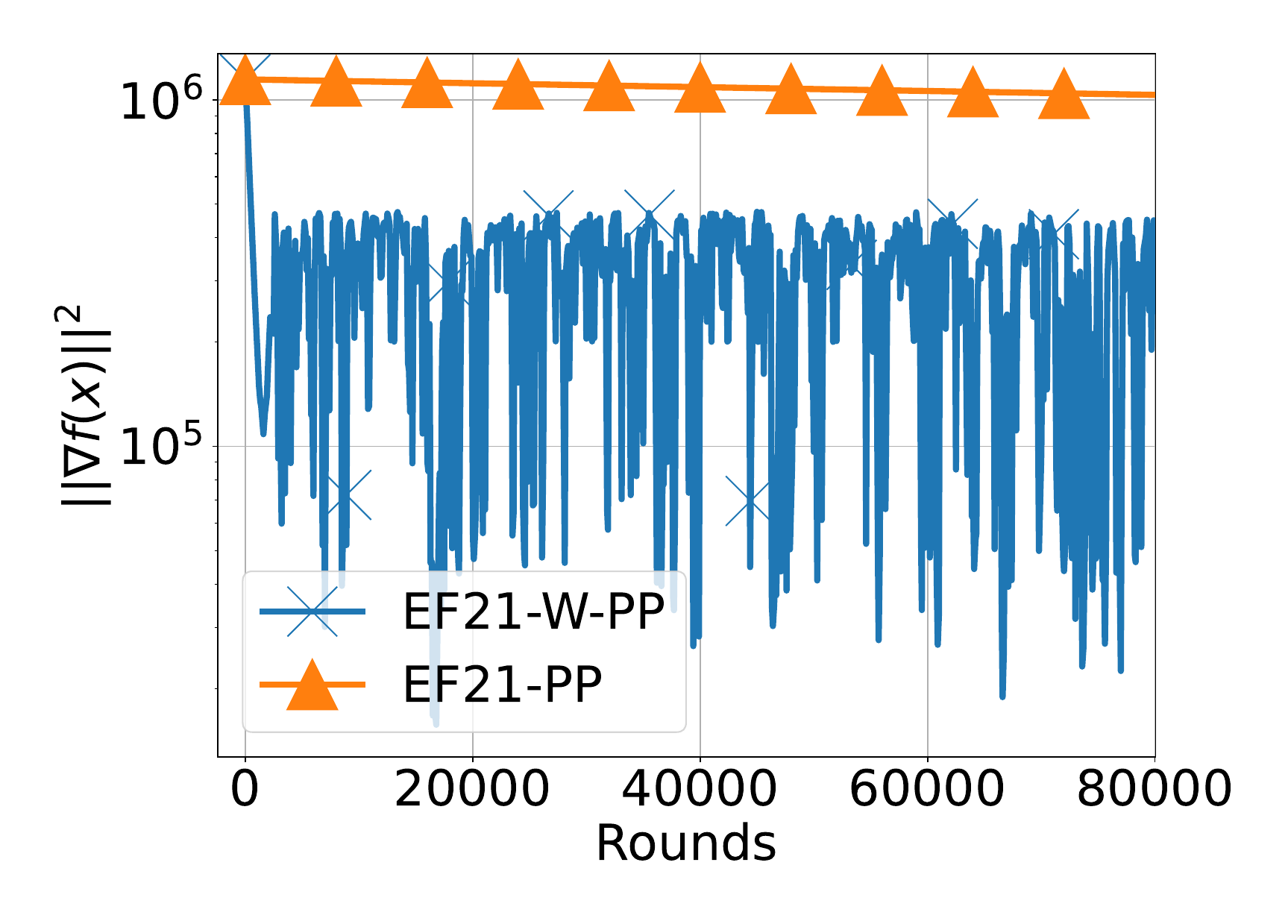} \caption{\small{(e)  \texttt{AUSTRALIAN}}}
		\end{subfigure}
		
		\caption{\small{Non-Convex Logistic Regression: comparison of \algnamesmall{EF21-PP} and \algnamesmall{EF21-W-PP}. The used compressor is \algnamesmall{Top1}. The number of clients $n=200$. Regularization term $\lambda \sum_{j=1}^{d} \frac{x_j^2}{x_j^2 + 1}$, with $\lambda=1,000$.  Theoretical step size. Each client participates in each round with probability $p_i=0.5$. The objective function is constitute of $f_i(x)$ defined in Eq.\eqref{eq:ncvx-log-reg-2}. Extra details are in Table \ref{tbl:app-real-ef21-pp-ncvx}.}}
		\label{fig:app-real-ef21-pp-ncvx-aus}
	\end{figure*}
\end{center}

From Figure \ref{fig:app-real-ef21-pp-ncvx} (a), (b), (c), we can observe that for these datasets, the \algname{EF21-W-PP} is better, and this effect is observable in practice and is not negligible. Figures \ref{fig:app-real-ef21-pp-ncvx} (d), demonstrate that sometimes \algname{EF21-W-PP} in terms of the full gradient at last iterate can have slightly worse behavior compared to \algname{EF21-PP}, even though theory allow more aggressive step-size (For exact values of step size see Table \ref{tbl:app-real-ef21-pp-ncvx}. The experiment on \texttt{AUSTRALIAN} dataset is presented in Figure \ref{fig:app-real-ef21-pp-ncvx-aus}. This example demonstrates that in this \texttt{LIBSVM} benchmark datasets, the relative improvement in the number of rounds for \algname{EF21-W-PP} compared to \algname{EF21-PP} is considerable. The \algname{EF21-W-PP} exhibits more oscillation behavior in terms of $\|\nabla f(x^t)\|^2$ for \texttt{AUSTRALIAN} dataset, however as we can see observe in expectation $\|\nabla f(x^t)\|^2$ tends to decrease faster compare to \algname{EF21-PP}.

\subsection{Additional experiments for \algname{EF21-W-SGD}}

The standard \algname{EF21-SGD} with the analysis described in Corollary 4 \citep{EF21BW} allows performing the optimization procedure with maximum allowable step size up to the factor of $2$ equal to:  
$$\gamma_{\text{\algnametiny{EF21-SGD}}} = \left(L + \sqrt{\frac{\hat{\beta_1}}{\hat{\theta}}} {\color{red}{\LQM}} \right)^{-1}.$$

In last expression quantities $\hat{\theta} = 1 - (1-\alpha)(1+s) (1+\nu)$, and $\hat{\beta_1} = 2(1- \alpha ) \left(1+ s \right)\left(s+\nu^{-1}\right)$. Improved analysis for \algname{EF21-W-SGD} allows to apply step size:
$$\gamma_{\text{\algnametiny{EF21-W-SGD}}} = \left(L + \sqrt{\frac{\hat{\beta_1}}{\hat{\theta}}}{\color{blue}{\LAM}} \right)^{-1}.$$

Therefore in terms of step size $$\gamma_{\text{\algnametiny{EF21-W-SGD}}} \ge \gamma_{\text{\algnametiny{EF21-SGD}}}$$ and  \algname{EF21-W-SGD} exhibits a more aggressive step-size.

We conducted distributed training of a logistic regression model on \texttt{W1A}, \texttt{W2A}, \texttt{W3A}, \texttt{PHISHING}, \texttt{AUSTRALIAN} datasets with non-convex regularization. For all datasets, we consider the  optimization problem \eqref{eq:main_problem}, 
where
\begin{eqnarray}
	\label{eq:ncvx-log-reg-3}
	f_i(x) \eqdef \frac{1}{n_i} \sum_{j=1}^{n_i} \log \left(1+\exp({-y_{ij} \cdot a_{ij}^{\top} x})\right) + \lambda  \sum_{j=1}^{d} \frac{x_j^2}{x_j^2 + 1},
\end{eqnarray}
and $(a_{ij},  y_{ij}) \in \mathbb{R}^{d} \times \{-1,1\}$.

The initial gradient estimators are set to $g_i^0 = \nabla f_i(x^0)$ for all  $i \in [n]$. For comparison of \algname{EF21-SGD} and \algname{EF21-W-SGD}, we used the largest step size allowed by theory. The dataset shuffling strategy repeats the strategy that we have used for \algname{EF21-W-PP} and \algname{EF21-W} and it is described in Appendix~\ref{app:dataset-shuffling-for-libsvm}. The algorithms \algname{EF21-SGD} and \algname{/EF21-W-SGD} employed an unbiased gradient estimator, which was estimated by sampling a single training point uniformly at random and independently at each client.

\begin{table}[h]		
	\begin{center}
		\caption{\small{Non-Convex optimization experiments in Figures \ref{fig:app-real-ef21-sgd-ncvx}, \ref{fig:app-real-ef21-sgd-ncvx-aus}. Quantities which define theoretical step size.}}
		\label{tbl:app-real-ef21-sgd-ncvx}
		\begin{tabular}{|c||c|c|c|c|c|c|c|c|c|}
			\hline
			Tag & $L$ & $\Lvar$ & $\LQM$ & $\LAM$ & $\gamma_{\text{\algnametiny{EF21-SGD}}}$ & $\gamma_{\text{\algnametiny{EF21-W-SGD}}}$ \\
			\hline
			\tiny{(a) W1A} & $0.781$ & $3.283$ & $2.921$ & $2.291$ & $4.014  \times 10^{-4}$ & $5.118   \times  10^{-4}$ \\
			\hline
			\tiny{(b) W2A} & $0.784$ & $2.041$ & $2.402$ & $1.931$ & $4.882  \times 10^{-4}$ & $6.072  \times 10^{-4}$ \\
			\hline
			\tiny{(c) W3A} & $0.801$ & $1.579$ & $2.147$ & $1.741$ & $5.460   \times  10^{-4}$ & $6.733   \times  10^{-4}$ \\
			\hline
			\tiny{(f) PHISHING} & $0.412$ & $9.2   \times  10^{-4}$ & $0.429$ & $0.428$ & $1.183  \times 10^{-2}$ & $1.186   \times  10^{-2}$\\
			\hline				
			\tiny{(g) AUSTRALIAN} & $3.96   \times  10^6 $ & $1.1   \times  10^{16}$ & $3.35   \times  10^{7}$ & $3.96   \times  10^{6}$ & $3.876  \times 10^{-10}$ & $3.243   \times  10^{-9}$\\
			\hline				
		\end{tabular}
	\end{center}
\end{table}

\begin{center}	
	\begin{figure*}[t]
		\centering
		\captionsetup[sub]{font=normalsize,labelfont={}}	
		\captionsetup[subfigure]{labelformat=empty}
		
		\begin{subfigure}[ht]{0.3\textwidth}
			\includegraphics[width=\textwidth]{./ef21-vc-figs-sgd/expreal/w1a-sgd.pdf} \caption{\small{ (a) \texttt{W1A}}}
		\end{subfigure}		
		\begin{subfigure}[ht]{0.3\textwidth}
			\includegraphics[width=\textwidth]{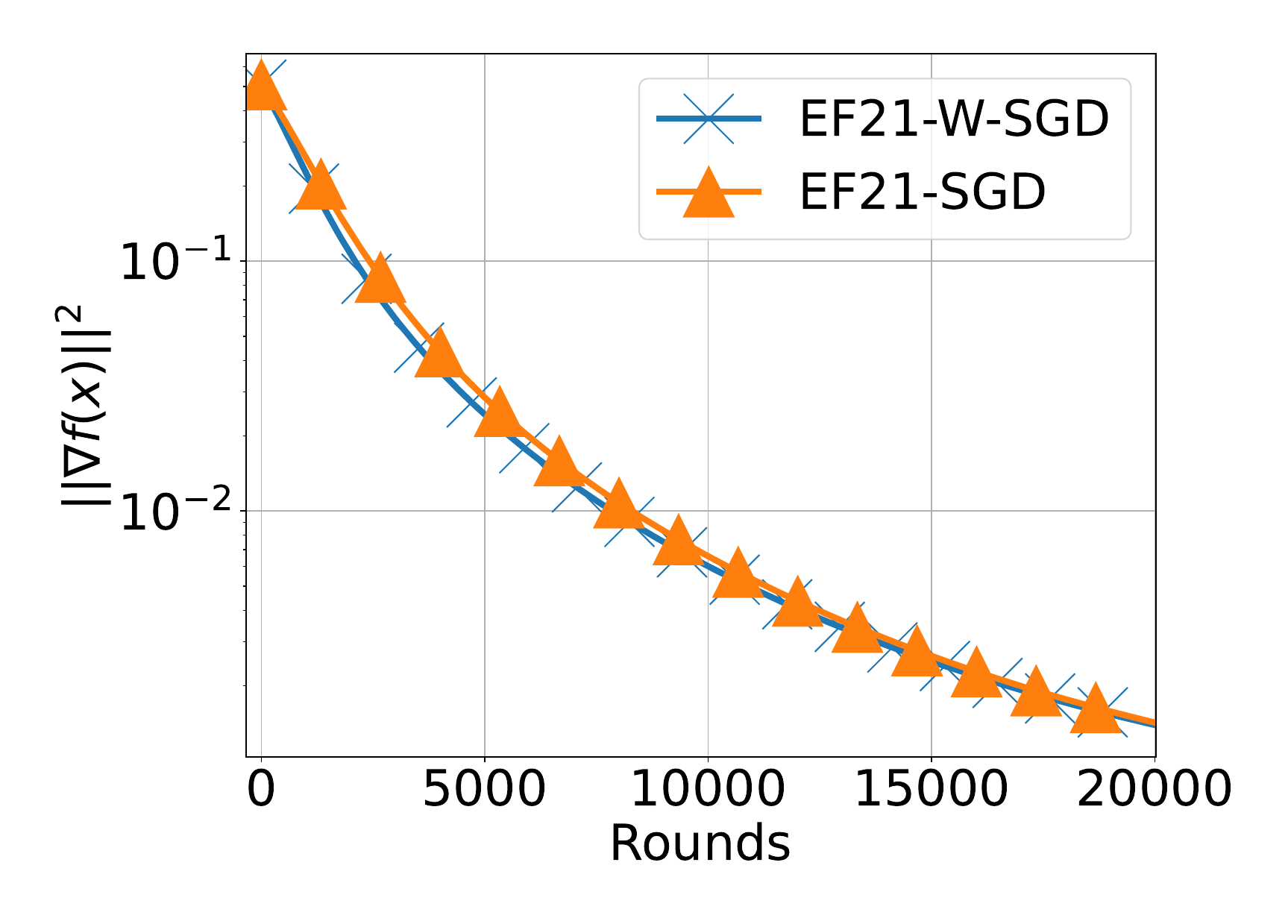} \caption{\small{ (b) \texttt{W2A}}}
		\end{subfigure} \\
		\begin{subfigure}[ht]{0.3\textwidth}
			\includegraphics[width=\textwidth]{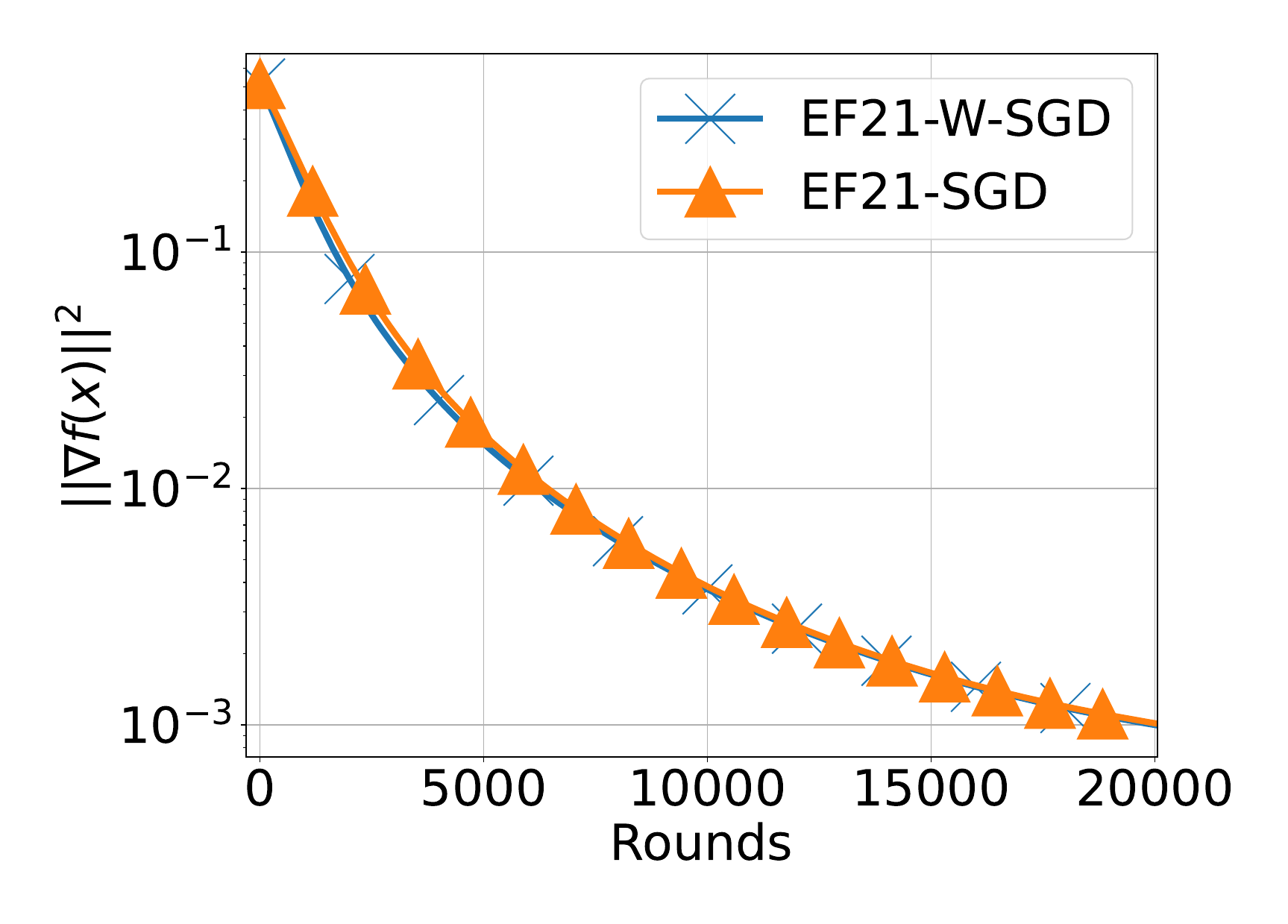} \caption{\small{ (c) \texttt{W3A}}}
		\end{subfigure}		
		\begin{subfigure}[ht]{0.3\textwidth}
			\includegraphics[width=\textwidth]{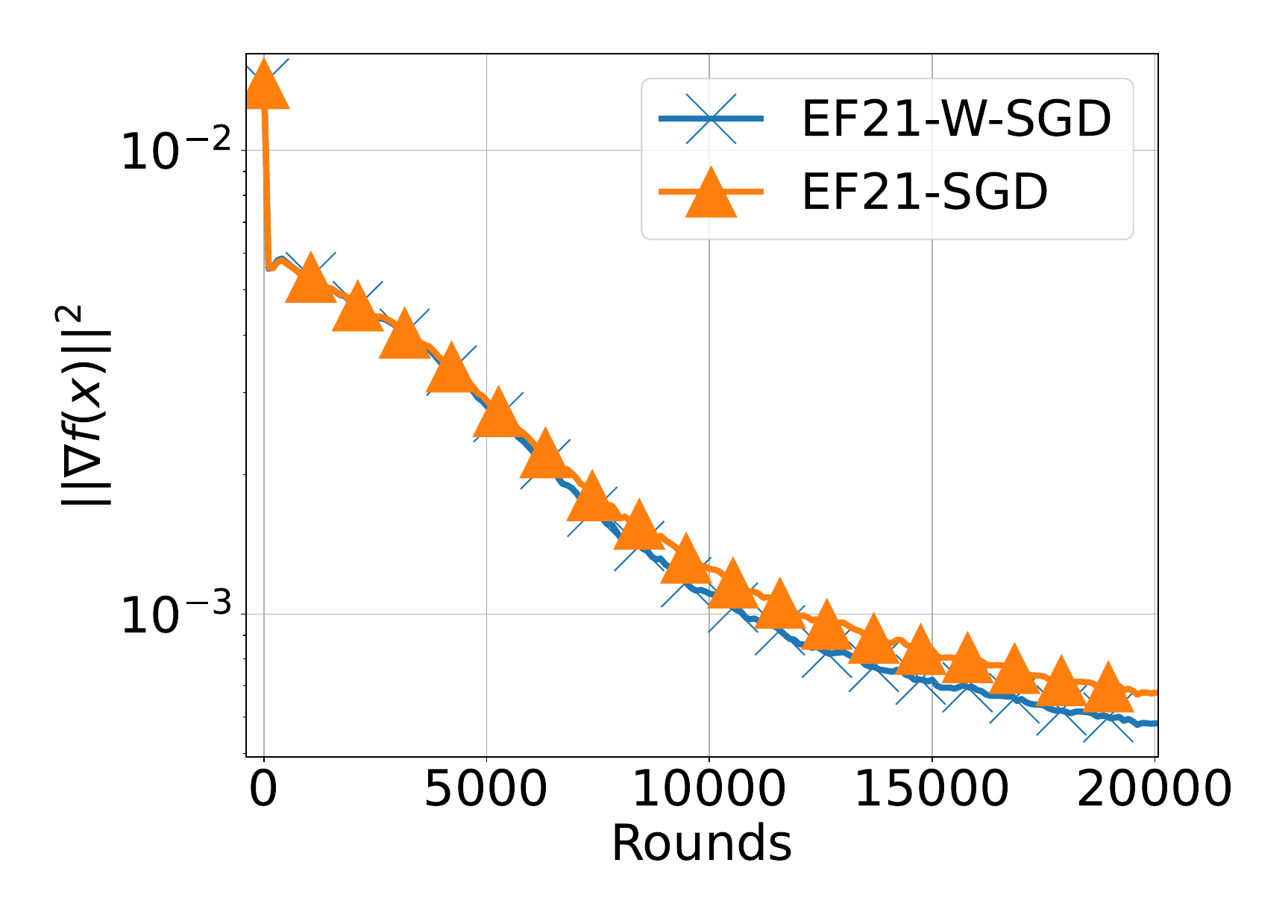} \caption{\small{ (f) \texttt{PHISHING}}}
		\end{subfigure}
		
		\caption{\small{Non-Convex logistic regression: comparison of \algnamesmall{EF21-SGD} and \algnamesmall{EF21-W-SGD}. The used compressor is \algnamesmall{Top1}. The \algnamesmall{SGD} gradient estimator is \algnamesmall{SGD-US}, $\tau=1$. The number of clients $n=1,000$. The objective function is constitute of $f_i(x)$ defined in Eq.\eqref{eq:ncvx-log-reg-3}.	Regularization term $\lambda \sum_{j=1}^{d} \frac{x_j^2}{x_j^2 + 1}$, $\lambda=0.001$. Theoretical step size. See also  Table \ref{tbl:app-real-ef21-sgd-ncvx}.}
		\label{fig:app-real-ef21-sgd-ncvx}}
	\end{figure*}
\end{center}

\begin{center}	
	\begin{figure*}[t]
		\centering
		\captionsetup[sub]{font=normalsize,labelfont={}}	
		\captionsetup[subfigure]{labelformat=empty}
		
		\begin{subfigure}[ht]{0.4\textwidth}
			\includegraphics[width=\textwidth]{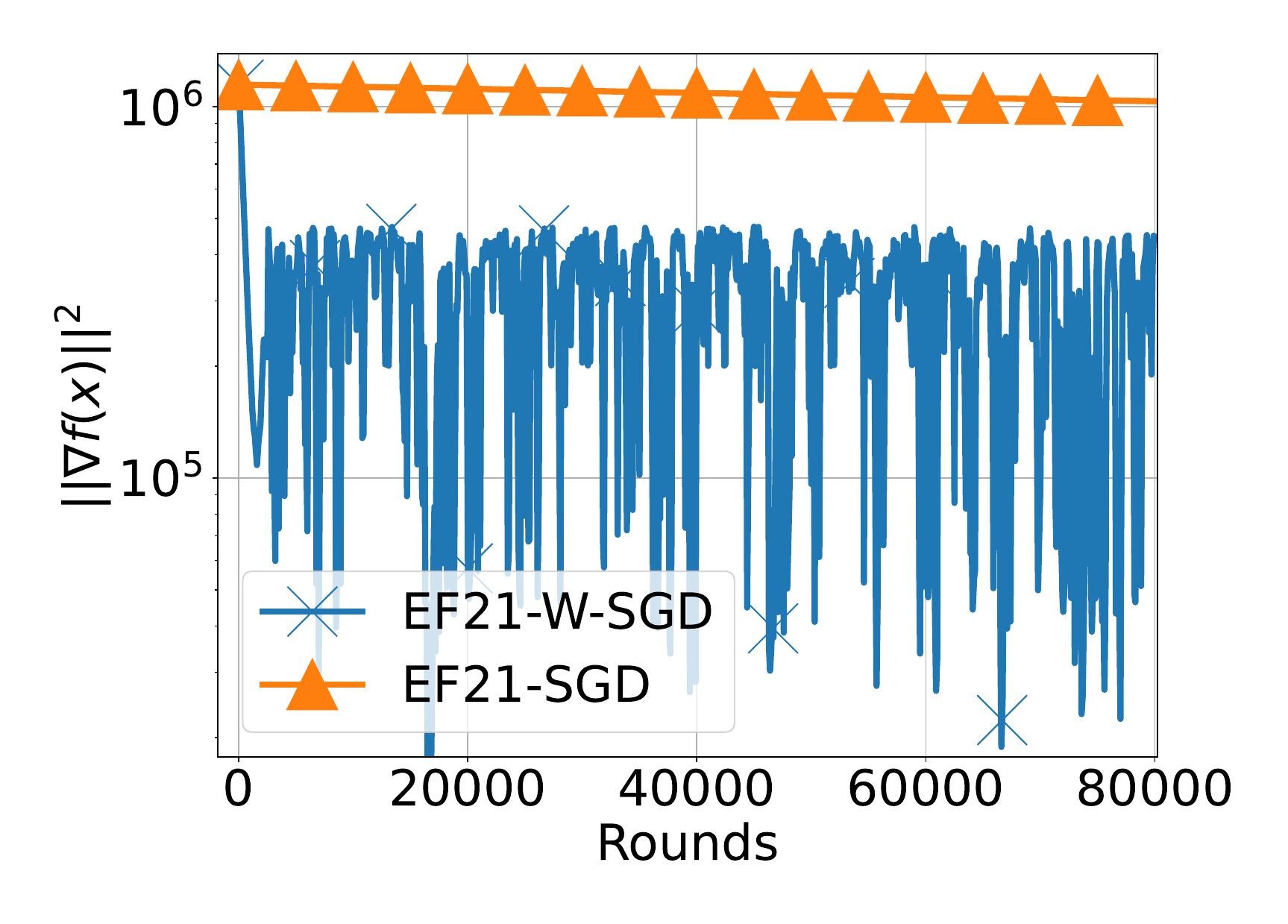} \caption{\small{ \texttt{AUSTRALIAN}}}
		\end{subfigure}
		
		\caption{\small{Non-Convex logistic regression: comparison of \algnamesmall{EF21-SGD} and \algnamesmall{EF21-W-SGD}. The used compressor is \algnamesmall{Top1}. The \algnamesmall{SGD} gradient estimator is \algnamesmall{SGD-US}, $\tau=1$. The number of clients $n=200$. The objective function is constitute of $f_i(x)$ defined in Eq.\eqref{eq:ncvx-log-reg-3}. Regularization term $\lambda \sum_{j=1}^{d} \frac{x_j^2}{x_j^2 + 1}$, with $\lambda=1,000$. Theoretical step size. Full participation. Extra details are in Table \ref{tbl:app-real-ef21-sgd-ncvx}.}}
		\label{fig:app-real-ef21-sgd-ncvx-aus}
	\end{figure*}
\end{center}

The results are presented in Figure \ref{fig:app-real-ef21-sgd-ncvx} and Figure \ref{fig:app-real-ef21-sgd-ncvx-aus}. The important quantities for these instances of optimization problems are summarized in Table \ref{tbl:app-real-ef21-sgd-ncvx}. In all Figures \ref{fig:app-real-ef21-sgd-ncvx} (a), (b), (c), (d) we can observe that for these datasets, the \algname{EF21-W-SGD} is better, and this effect is observable in practice. The experiment on \texttt{AUSTRALIAN} datasets are presented in Figure \ref{fig:app-real-ef21-sgd-ncvx-aus}. This example demonstrates that in this \texttt{LIBSVM} benchmark datasets, the relative improvement in the number of rounds for \algname{EF21-W-SGD} compared to \algname{EF21-SGD} is considerable. Finally, we address oscillation behavior to the fact that employed step size for \algname{EF21-SGD} is too pessimistic, and its employed step size removes oscillation of $\|\nabla f(x^t)\|^2$.


\section{Reproducibility Statement}

To ensure reproducibility, we use the following \texttt{FL\_PyTorch} simulator features: (i) random seeds were fixed for data synthesis; (ii) random seeds were fixed for the runtime pseudo-random generators involved in \algname{EF21-PP} and \algname{EF21-SGD} across clients and the server; (iii)  the thread pool size was turned off to avoid the non-deterministic order of client updates in the server.

If you are interested in the source code for all experiments, please contact the authors.

\end{document}